%% file: main.tex
\documentclass[10pt]{article} %

\usepackage[accepted]{tmlr}

\input{math_commands.tex}

\usepackage{stackengine}
\usepackage{hyperref}
\usepackage{url}
\usepackage{graphicx}
\usepackage{algorithm}
\usepackage{chngpage}
\usepackage{algpseudocode}
\usepackage{amssymb}
\usepackage{caption}
\usepackage{subcaption}
\usepackage{amsthm}
\usepackage{multirow}
\usepackage{stmaryrd}
\usepackage{xcolor}
\usepackage{comment}
\usepackage{array}
\usepackage{float}
\usepackage{xspace}
\usepackage{hhline}
\usepackage[section]{placeins}
\usepackage{caption, makecell, booktabs}
\usepackage{adjustbox}
\newcommand\remi{}
\newcommand\updateremi{}
\newcommand\revisremi{}
\newcolumntype{P}[1]{>{\centering\arraybackslash}m{#1}}

\newcommand{\CBMshap}{CLIP-SHAP\xspace}
\newcommand{\CBMLIME}{CLIP-LIME\xspace}
\newcommand{\method}{CLIP-QDA\xspace}
\newcommand{\methodsample}{CLIP-QDA$^{local}$\xspace}
\newcommand{\methoddata}{CLIP-QDA$^{global}$\xspace}

\newtheorem{proposition}{Proposition}

\title{CLIP-QDA: An Explainable Concept Bottleneck Model}

\author{\name Rémi Kazmierczak  \email remi.kazmierczak@ensta-paris.fr \\
      \addr Unité d'Informatique et d'Ingénierie des Systèmes \\
      ENSTA Paris, Institut Polytechnique de Paris
      \AND
      \name Eloïse Berthier  \email eloise.berthier@ensta-paris.fr \\
      \addr Unité d'Informatique et d'Ingénierie des Systèmes \\
      ENSTA Paris, Institut Polytechnique de Paris\\
      \AND
      \name Goran Frehse  \email goran.frehse@ensta-paris.fr \\
      \addr Unité d'Informatique et d'Ingénierie des Systèmes \\
      ENSTA Paris, Institut Polytechnique de Paris\\
      \AND
      \name Gianni Franchi  \email gianni.franchi@ensta-paris.fr \\
      \addr Unité d'Informatique et d'Ingénierie des Systèmes \\
      ENSTA Paris, Institut Polytechnique de Paris }

\def\month{05}  %
\def\year{2024} %
\def\openreview{\url{https://openreview.net/forum?id=jjmdiMiag7}} %

\begin{document}

\maketitle

\begin{abstract}

\remi{In this paper, we introduce an explainable %
algorithm designed from a multi-modal foundation model, that performs fast and explainable image classification. Drawing inspiration from CLIP-based Concept Bottleneck Models (CBMs), our method creates a latent space where each neuron is linked to a specific word. %
Observing that this latent space can be modeled with simple distributions, we use a Mixture of Gaussians (MoG) formalism to enhance the interpretability of this latent space. Then, we introduce \method, a classifier that only uses statistical values to infer labels from the concepts. In addition, this formalism allows for both \updateremi{sample-wise} and \updateremi{dataset-wise} explanations.  %
These explanations come from the inner design of our architecture, our work is part of a new family of greybox models, combining performances of opaque foundation models and the interpretability of transparent models. Our empirical findings show that in instances where the MoG assumption holds, \method achieves similar accuracy with state-of-the-art CBMs. Our explanations compete with existing XAI methods while being faster to compute.}

\end{abstract}

\section{Introduction}

The field of artificial intelligence is advancing rapidly, driven by sophisticated models like Deep Neural Networks \citep{lecun2015deep} (DNNs). These models find extensive applications in various real-world scenarios, including conversational chatbots \citep{ouyang2022training}, neural machine translation \citep{liu2020very}, and image generation \citep{rombach2021high}. Although these systems demonstrate remarkable accuracy, the process behind their decision-making often remains obscure. Consequently, deep learning encounters certain limitations and drawbacks. The most notable one is the lack of transparency regarding their behavior, which leaves users with limited insight into how specific decisions are reached. This lack of transparency becomes particularly problematic in high-stakes situations, such as medical diagnoses or autonomous vehicles.

The imperative to scrutinize the behavior of DNNs has become increasingly compelling as the field gravitates towards methods of larger scale in terms of both data utilization and number of parameters involved, culminating in what is commonly referred to as ``foundation models'' \citep{brown2020language,radford2021learning,ramesh2021zero}. These models have demonstrated remarkable performance, particularly in the domain of generalization, while concurrently growing more intricate and opaque. Additionally, there is a burgeoning trend in the adoption of multimodality \citep{reed2022generalist}, wherein various modalities such as sound, image, and text are employed to depict a single concept. This strategic use of diverse data representations empowers neural networks to transcend their reliance on a solitary data format. Nonetheless, the underlying phenomena that govern the amalgamation of these disparate inputs into coherent representations remain shrouded in ambiguity and require further investigation.

\remi{The exploration of latent representations is crucial for understanding the internal dynamics of a DNN.}
DNNs possess the capability to transform input data into a space, \remi{called latent space,} where inputs representing the same semantic concept are nearby. 
\remi{For example, in the latent space of a DNN trained to classify images, two different images of cats would be mapped to points that are close to each other} \citep{johnson2016perceptual}. This capability is further reinforced through the utilization of multimodality \remi{\citep{akkus2023multimodal}}, granting access to neurons that represent abstract concepts inherent to multiple types of data signals.

A class of networks that effectively exploits this notion is  Concept Bottleneck Models (CBMs) \citep{koh2020concept}. CBMs are characterized by their deliberate construction of representations for high-level human-understandable concepts, frequently denoted as words. Remarkably, there is a growing trend in employing CLIP \citep{radford2021learning}, a foundation model that establishes a shared embedding space for both text and images, to generate concept bottleneck models in an unsupervised manner.

Unfortunately, while CLIP embeddings represent tangible concepts, the derived values, often termed ``CLIP scores'' pose challenges in terms of interpretation. Furthermore, to the best of our knowledge, there is a notable absence of studies that seek to formally characterize  CLIP's latent space. %
The underlying objective here is to gain insights into how the pre-trained CLIP model organizes a given input distribution. Consequently, there is an opportunity to develop mathematically rigorous methodologies for elucidating the behavior of CLIP.

Then, our contributions are summarized as follows:
\begin{itemize}
\item 
\remi{We propose to represent }\remi{the distribution of} 
\remi{CLIP scores by a mixture of Gaussians.}
This representation enables a mathematically interpretable %
\remi{classification} of images \remi{using} human-understandable concepts. 

\item \remi{Utilizing the modeling approach presented in this study, we use Quadratic Discriminant Analysis (QDA) to classify the labels from the concepts, we name this method \method.  \method demonstrates competitive performance when compared to existing CBMs based on CLIP. Notably,  \method achieves this level of performance while utilizing a reduced set of parameters, limited solely to statistical values, including means, covariance matrices, and label probabilities. }

\item We \remi{propose two} efficient and mathematically grounded \remi{XAI} \updateremi{methods} \remi{for model explanation}, \updateremi{named \methodsample and \methoddata}. These \updateremi{methods} encompass both \remi{global} and \remi{local} assessments \remi{of \revisremi{how} the model behaves.} The \remi{global} metric directly emanates from our Gaussian modeling approach, providing a comprehensive evaluation of  \method's performance. Additionally, our \remi{local} metric draws inspiration from counterfactual analysis, furnishing insights into individual data points. 

\item \updateremi{We extend two established post-hoc XAI methods, LIME and SHAP, to formulate a novel XAI approach specifically tailored for 
CBMs. Departing from the conventional application of these methods, which typically %
produce explanations on the image level, \CBMLIME and \CBMshap generate explanations on the concept level allowed by CBMs. 
}

\item \updateremi{We propose a new evaluation protocol, %
 specifically designed for the unique characteristics of CBMs to assess the effectiveness of explanations. This protocol includes a deletion metric, which examines faithfulness to the model, and a detection metric, which evaluates faithfulness to the data.}

\end{itemize}

\begin{figure}
    \centering
    \includegraphics[width=0.8\textwidth]{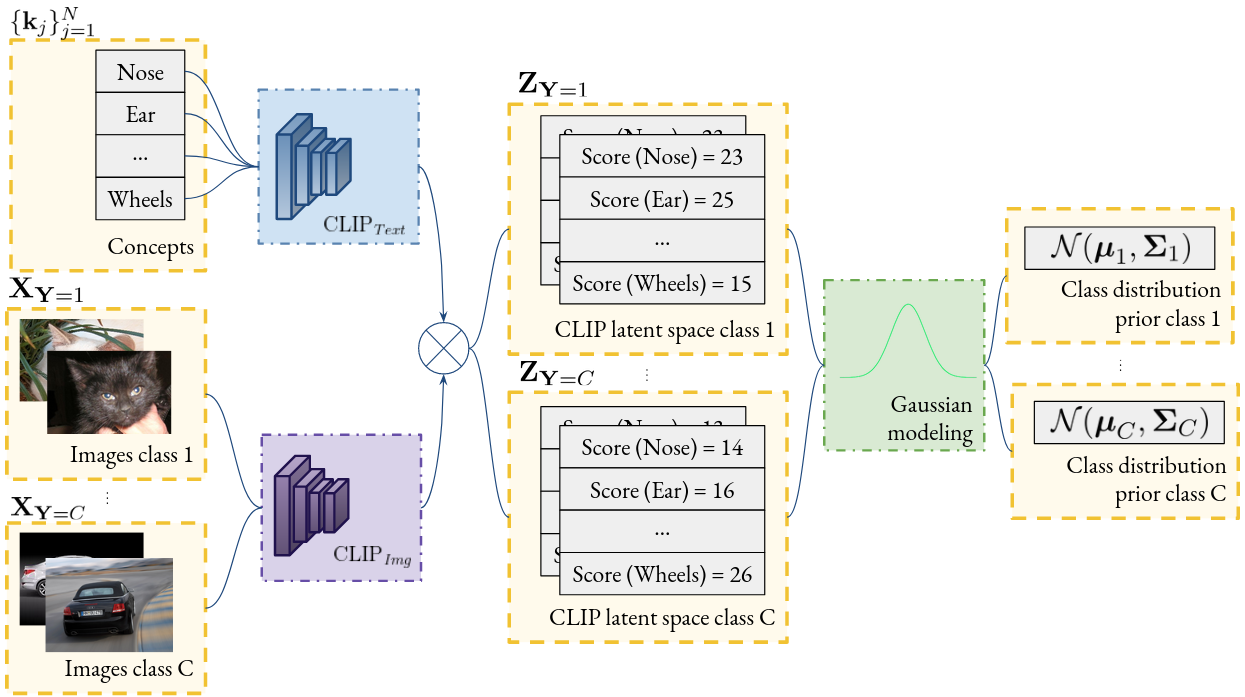}
    \caption{\textbf{Overview of our \remi{modeling} method.} \remi{By considering the whole dataset CLIP scores $\vz$ as class conditioned distributions $\vZ=\begin{bmatrix} Z^{\remi{1}} & \ldots & Z^{\remi{N}} \end{bmatrix}$, we model the CLIP latent space as a mixture of Gaussians, allowing for mathematically grounded explanations.}}
    \label{fig:overview}
\end{figure}

\section{Background and Related Work}

\subsection{Contrastive Image Language Pre-training (CLIP)}

CLIP \citep{radford2021learning} is a %
\revisremi{multi-modal} model that can \textcolor{black}{jointly} 
\remi{process image and text inputs.}
The model was pre-trained on a large dataset of images and texts to learn to associate visual and textual concepts. Then, the capacity of CLIP to create a semantically rich encoding induced the creation of many emergent models in detection, few-shot learning, or image captioning.

The widespread adoption of CLIP stems from the remarkable robustness exhibited by its pre-trained model. Through training on an extensive multimodal dataset, \remi{such as \citet{schuhmann2022laion}}, the model achieves impressive performance. Thus, on few- and zero-shot learning, for which it was designed, it \remi{obtains} impressive results across a wide range of datasets. Notably, CLIP provides a straightforward and efficient means of obtaining semantically rich representations of images in low-dimensional spaces. This capability enables researchers and practitioners to divert the original use of CLIP to various other applications \citep{luo2022clip4clip, menon2022visual,gabeff2023wildclip}.

\subsection{\updateremi{CLIP-based Concept Bottleneck Models (CLIP-based CBMs)}}

\updateremi{The term Concept Bottleneck Model (CBM), as outlined in \citet{koh2020concept}, %
\revisremi{refers to} to the implementation of a bottleneck reliant on human-specified concepts to execute a task, predominantly image classification. Consequently, the resultant algorithm inherently facilitates enhanced interpretability. While the term itself is relatively recent, it characterizes a lineage of methods that were employed in earlier research \citep{kumar2009attribute,lampert2009learning,koh2020concept,losch2019interpretability}. However, despite the advantages offered by CBMs in terms of better understanding, early implementations encountered challenges stemming from the requirement for dedicated datasets. These datasets needed to encompass not only inputs and labels but also incorporate human-specified concepts for each sample.}

\updateremi{In this context}, the emergence of multimodal foundational models has opened up novel opportunities. Recent research \citep{yang2023language,oikarinen2023label} has leveraged large language models to directly construct concepts from CLIP text embeddings, opening the door to a family of CLIP-based CBMs. Additionally, efforts have been made to create sparse CLIP-based CBMs \citep{panousis2023sparse,feng2023leveraging}. \citet{yan2023learning} explore methods to achieve superior representations with minimal labels. \citet{yuksekgonul2022post} capitalize on the CLIP embedding spaces, considering concepts as activation vectors. Finally, \citet{kim2023grounding} build upon the idea of activation vectors to discover counterfactuals. 
\subsection{Explainable AI}

According to \citet{arrieta2020explainable}, we can define an explainable model as a computational model, that is designed to provide specific details or reasons to ensure clarity and ease of understanding regarding its functioning. In broader terms, an explanation denotes the information or output that an explainable model delivers to elucidate its operational processes.

The literature shows a clear distinction between non-transparent (or blackbox) and transparent (or whitebox) models. Transparent models are characterized by their inherent explainability. These models can be readily explained due to their simplicity and easily interpretable features. Examples of such models include linear regression \citep{galton1886regression}, logistic regression \citep{mccullagh2019generalized}, and decision trees \citep{quinlan1986induction}. 
In contrast, non-transparent models are inherently non-explainable.  This category encompasses models that could have been explainable if they possessed simpler and more interpretable features \citep{galton1886regression, quinlan1986induction}, as well as models that inherently lack explainability, including deep neural networks. The distinction between these two types of models highlights the trade-off between model complexity and interpretability \citep{arrieta2020explainable}, with transparent models offering inherent explainability while non-transparent models allow for better performance but require the use of additional techniques for explanations,  named post-hoc methods. Commonly used post-hoc methods include visualization techniques, such as saliency maps, which highlight the influential features in an image that contribute to the model's decision-making. \updateremi{Within this category of methods, notable examples include approaches such as Grad-CAM \citep{selvaraju2017grad}, which generates activation maps by computing the gradients of the output labels.} Sensitivity analysis \citep{cortez2011opening} represents another avenue, involving the analysis of model predictions by varying input data. \updateremi{Sample-wise} explanation techniques are also used to explain the model from a local simplification of the model around a point of interest \citep{lime,plumb2018model}. Finally, feature relevance techniques aim at estimating the impact of each feature on the decision \citep{lundberg2017unified}.

In an endeavor to integrate the strengths of both black and whitebox models, the concept of greybox XAI has been introduced by \citet{bennetot2022greybox}. These models divide the overall architecture into two distinct components. Initially, a blackbox model is employed to process high-entropy input signals, such as images, and transform them into a lower-entropy latent space that is semantically meaningful and understandable by humans. By leveraging the blackbox model's ability to simplify complex problems, a whitebox model is then used to deduce the final result based on the output of the blackbox model. This approach yields a partially explainable model that outperforms traditional whitebox models while retaining partial transparency, in a unified framework. %

\updateremi{Feature Attribution Methods, a category of techniques employed to address the complexity of DNNs' output for explainability, strive to identify crucial features in the input. These methods leverage a mapping function to reduce input complexity. Notable examples in this family include DeepSHAP and KernelSHAP, as proposed by \citet{lundberg2017unified}. This approach resonates with the concept of greybox models, wherein the input is initially simplified to be explainable by a transparent classifier. However, greyboxes differ from Feature Attribution Methods in that the mapping function is independent from the input under consideration.}

\section{\updateremi{A Greybox Concept Bottleneck Model: CLIP-QDA}}

\subsection{\remi{General Framework}}

For our experimental investigations, we consider a 
\remi{general framework based on} prior work on CLIP-based CBMs \citep{yang2023language,oikarinen2023label}. This framework consists of two core components. The first component centers on the extraction of multi-modal features, enabling the creation of connections between images and text. %
The second component encompasses a classifier head. A visual depiction of this process is presented in Figure \ref{fig:baseline}.

\begin{figure}
    \centering
    \includegraphics[width=0.8\textwidth]{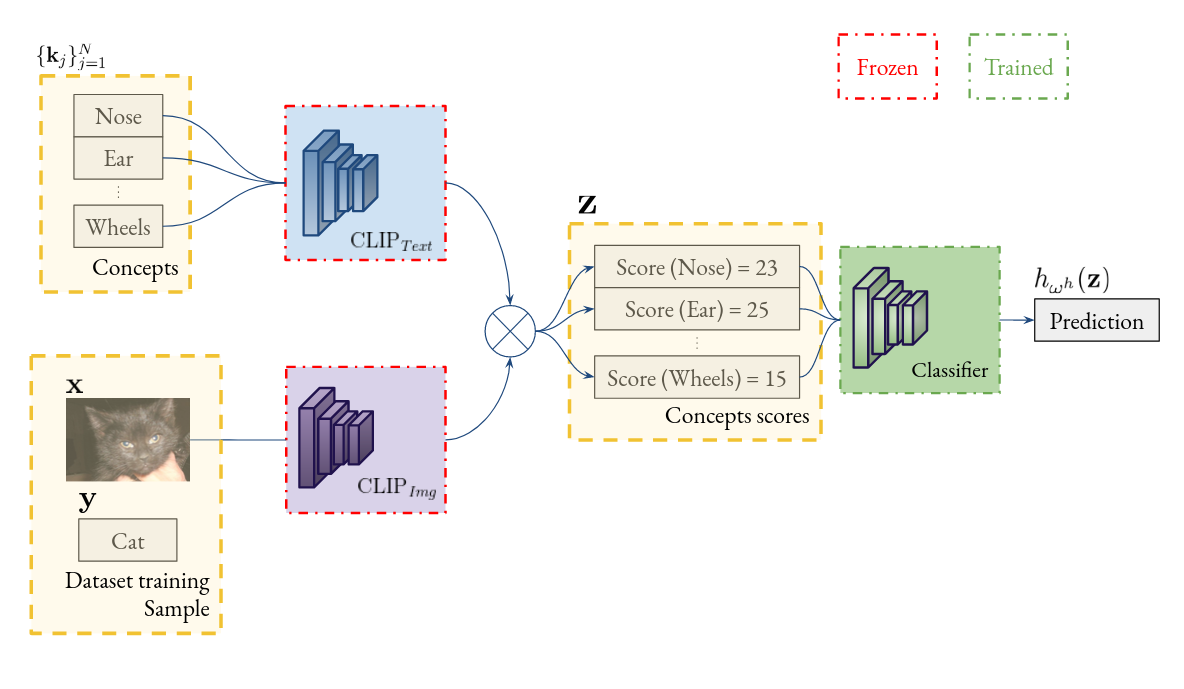}
    \caption{\textbf{Training procedure of \remi{the}
    \remi{general framework.}}
    \remi{First, CLIP scores $\vz$ are computed for each of the concepts $\{\vk^{j}\}_{j=1}^{N}$, then a classifier $h_{\vomega^h}(.)$, with parameters $\vomega^h$ is trained to classify the label from the concatenation of the CLIP scores. }}
    \label{fig:baseline}
\end{figure}

We build upon CLIP DNN \citep{radford2021learning}, which enables the creation of a multi-modal latent space through the fusion of image and text information. Rather than relying on a single text or prompt, we employ a set of diverse prompts, each representing distinct concepts. These concepts remain consistent across the dataset and are not subject to alterations. The purpose of CLIP's representation is to gauge the similarities between each concept and an image, thereby giving rise to a latent space. To prevent ambiguity, we denote the resulting space of CLIP scores as the ``CLIP latent space'', while the spaces generated by the text and image encoders are respectively referred to as the ``CLIP text embedding space'' and the ``CLIP image embedding space''. Here, ``CLIP score'' denotes the value derived from a cosine distance computation between the image and text encodings.

The selection of concepts is guided by expert input and acts as a hyperparameter within our framework. For comprehensive examples of concept sets, please refer to Section \ref{concepts_sets}. Notably, there is no requirement for individual image annotation with these concepts. This is due to CLIP's inherent design, which allows it to score \remi{concepts} in a zero-shot manner.

\remi{Following the acquisition of the CLIP latent space, it is given as an input to a classifier head, which is responsible for learning to predict the class. } %
\remi{Thanks to the low dimension of the latent space and the clear semantics of each component (concepts), it is possible to design simple and explainable classifiers.}

\subsection{\remi{CLIP Latent Space Analysis}}

\subsubsection{Notations and formalism}

Let us introduce the following notations used in the rest of the paper.
$X$ and $Y$ represent two random variables (RVs) with joint distribution $(X,Y) \sim \mathcal{P}_{X,Y}$. A realization of this distribution is a pair $(\vx,y)$ that concretely represents one image and its label. In particular, $y$ takes values in $\llbracket 1,C \rrbracket $, with $C \in \mathbb{N}$ the number of classes. From this distribution, we can deduce the marginal distributions $ X \sim \mathcal{P}_X$ and  $ Y \sim \mathcal{P}_Y$. We can also describe for each class $c$, an RV $X_{Y=c} \sim \mathcal{P}_{X_{Y=c}}$ that represents the \remi{conditional} distribution of images that have the class $c$.

Let $\{\vk^{j}\}_{j=1}^{N}$ denote a set of $N \in \mathbb{N}$ concepts, \remi{where each $\vk^{j}$ is a character string representing the concept in natural language}.
We consider ``CLIP’s DNN'' to refer to the vector of its pre-trained weights, denoted as $\vomega^g$, and a function $g$ that represents the architecture of the deep neural network (DNN). Given an image $\vx$ and a concept $\vk^{j}$, the output of CLIP’s DNN is represented as $z^{j}=g_{\vomega^g}(\vx,\vk^{j})$.  The projection in the multi-modal latent space of an image $\vx$ is the vector 
$\vz= \begin{bmatrix} g_{\vomega^g}(\vx,\vk^{\remi{1}}) & \ldots & g_{\vomega^g}(\vx,\vk^{\remi{N}}) \end{bmatrix}$. 
We define~$Z^{j}$ as the random variable associated with the observation~$z^{j}$. It should be noted that 
$\vZ=\begin{bmatrix} Z^{\remi{1}} & \ldots & Z^{\remi{N}} \end{bmatrix}$ is the random variable representing the concatenation of the CLIP scores associated with the $N$ concepts. 
\remi{Furthermore}, we denote the \remi{conditional} distributions of $\vz$ having class~$c$ as $\vZ_{Y=c}=\begin{bmatrix} Z^{\remi{1}}_{Y=c} & \ldots & Z^{\remi{N}}_{Y=c} \end{bmatrix}$.

Finally, we define the classifier as a function $h_{\vomega^h}(\vz)$ with parameters~$\vomega^h$ that, given a vector $\vz$, outputs the predicted class.   

\subsubsection{Gaussian modeling of CLIP\remi{'s latent space}} \label{gauss_modeling}%

To analyze the behavior of the CLIP latent space, we conduct a thorough examination of the distribution of CLIP scores. To elucidate our modeling approach, we suggest 
\remi{to visualize a large set of samples from $Z^{j}$ by observing the CLIP scores of an entire set of images taken from a toy example, which consists of images representing only cats and cars (see the Cats/Dogs/Cars dataset in Section \ref{setup}). In this instance, the concept denoted by ``$j$'' corresponds to~``Pointy-eared''.}

In Figure \ref{fig:Zi1}, which illustrates the histogram of CLIP scores, we observe that the distribution exhibits characteristics that can be effectively modeled as a mixture of two Gaussians. The underlying intuition here is that the distribution $Z^{j}$ represents two types of images: those without pointy ears, resulting in the left mode (low scores) of CLIP scores, and those with pointy ears, resulting in the right mode (high scores). Since this \remi{concept} uniquely characterizes the classes -- cats have \remi{ears} but not cars -- we can assign each mode to a specific class. This intuition is corroborated by the visualization of the distribution $Z^{j}_{Y=1}$ (Car) in violet and the distribution $Z^{j}_{Y=2}$ (Cat) in red. Since the extracted distributions exhibit similarities to normal distributions, we are motivated to describe $\vZ$ as a mixture of Gaussians. Yet, we also discuss the validity and limitations of this modeling approach in Section \ref{gauss_hyp_test}.

\begin{figure}[H]
     \centering
     \begin{subfigure}[b]{0.45\textwidth}
         \centering
         \includegraphics[width=\textwidth]{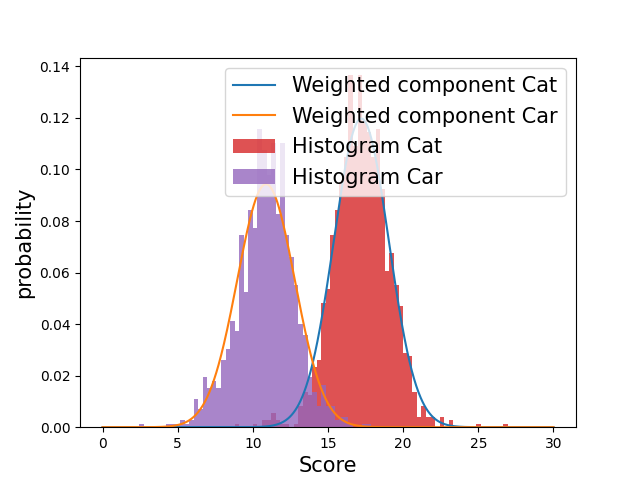}
         \caption{Modeling of $p_c Z_{Y=c}$.}
         \label{fig:Zi1_comp}
     \end{subfigure}
     \hfill
     \begin{subfigure}[b]{0.45\textwidth}
         \centering
         \includegraphics[width=\textwidth]{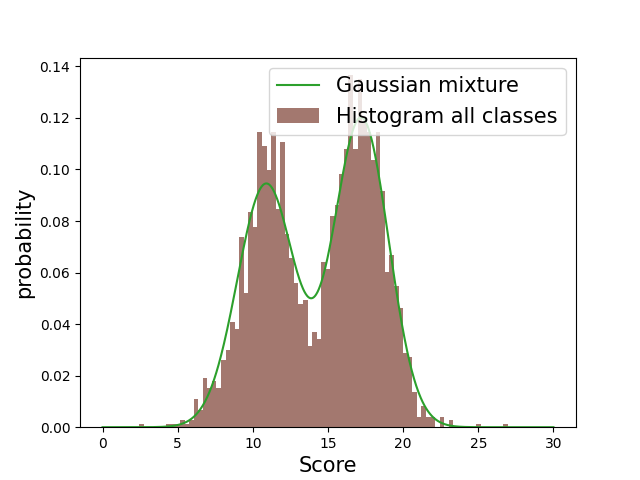}
         \caption{Modeling of $Z$.}
         \label{fig:Zi1_gmm}
     \end{subfigure}
    \caption{\remi{\textbf{Normalized histogram of scores $Z^{j}$ specifically for the concept ``Pointy-eared''.} On the left, we observe that the different classes can be modeled as weighted Gaussians. On the right, we show the resulting Gaussian mixture modeling.}}
    \label{fig:Zi1}
\end{figure}

Mathematically, the Gaussian prior assumption is equivalent to:
\begin{align}
 p(\vZ=\vz \mid Y=c) = \mathcal{N}(\vz \mid \vmu_c, \vSigma_c) ,\label{equ:gaussian}   \
\end{align}
where $\vSigma_c$ and $\vmu_c$ are the mean vectors and the covariance matrices, different for each class. Moreover, given the multinomial distribution of $Y$, with the notation $p_c=P(Y=c)$, we can model $\vZ$ as a mixture of Gaussians:
\begin{align}
p(\vZ=\vz) = & \sum_{\updateremi{c_i}=1}^C p_{\updateremi{c_i}} \mathcal{N}(\vz \mid \vmu_{\updateremi{c_i}}, \vSigma_{\updateremi{c_i}})   \ .
\end{align}

\subsection{CLIP \remi{Quadratic Discriminant Analysis} \updateremi{(\method)}}

Based on the Gaussian distribution assumption described in Section \ref{gauss_modeling}, a natural choice for $h_{\vomega^h}$ (the classifier in Figure \ref{fig:baseline}) is the \remi{Quadratic Discriminant Analysis  (QDA) as defined in \citet{hastie2009elements}}. To compute it, we need to estimate the parameters \remi{$(\vSigma_c, \vmu_c, p_c)$ of the probability distributions $\vZ_{Y=c}$ and $Y$,  which is done by computing the maximum likelihood estimators on the training data.}

Subsequently, with the knowledge of the functions $p(\vZ=\vz \mid Y=c)$ and $p(Y=c)$, we can apply Bayes theorem to make an inference on $p(Y=c \mid \vZ=\vz)$:
\begin{equation}
p(Y=c \mid \vZ=\vz) = \frac{p_c \mathcal{N}(\vz \mid  \vmu_c, \vSigma_c)}{\sum_{\updateremi{c_i}=1}^{N} p_{\updateremi{c_i}} \mathcal{N}(\vz \mid  \vmu_{\updateremi{c_i}}, \vSigma_{\updateremi{c_i}})}  \label{CLIP-QDA_proba} \, .
\end{equation}

\remi{Then, the output of the QDA classifier can be described as:}
\remi{\begin{equation}
h_{\vomega^h}(\vz) = \underset{c}{\arg\max} ~~ \frac{p_c}{(2\pi)^{N/2}|\vSigma_c|^{1/2}} e^{-\frac{1}{2}(\vz-\vmu_c)^T\vSigma_c^{-1}(\vz-\vmu_c)}  \label{CLIP-QDA} \, .
\end{equation}}

In practice, we leverage the training data to estimate $\vomega^h = (\vSigma_c, \vmu_c, p_c)$,
which enables us to bypass the standard stochastic gradient descent process, resulting in an immediate ``training time''. 
Furthermore, this classifier offers the advantage of transparency, akin to the approach outlined by \citet{arrieta2020explainable}, with its parameters comprising identifiable statistical values and its output values representing probabilities.

\subsection{Explainable AI for Concept Bottleneck Models} \label{XAI_methods}

\updateremi{Our \method model is founded upon a transparent probabilistic framework, hence, we have at our disposal a variety of statistical tools to explain the functioning of our classifier as illustrated in Figure \ref{organigram_XAI}. In this section, we present two distinct types of explanations: \methoddata, offering a \textit{global} perspective that sheds light on the classifier's behavior across the entire dataset (refer to Section \ref{dataset_wise}), and \methodsample, providing a \textit{local} explanation tailored to elucidate the model's actions on individual samples (refer to Section \ref{sample_wise}).}

\updateremi{Furthermore, we propose the adaptation of two well-established XAI post-hoc methods, LIME \citep{lime} and SHAP \citep{lundberg2017unified}, to the unique characteristics of CBMs. These methods are denoted as \CBMshap and \CBMLIME, respectively (see Section \ref{SHAPLIME-CBM}). An overview of the methods incorporated in our investigation is depicted in Figure \ref{organigram_XAI}. 
}

\begin{figure}[H]
     \centering
         \centering
         \includegraphics[width=0.5\textwidth]{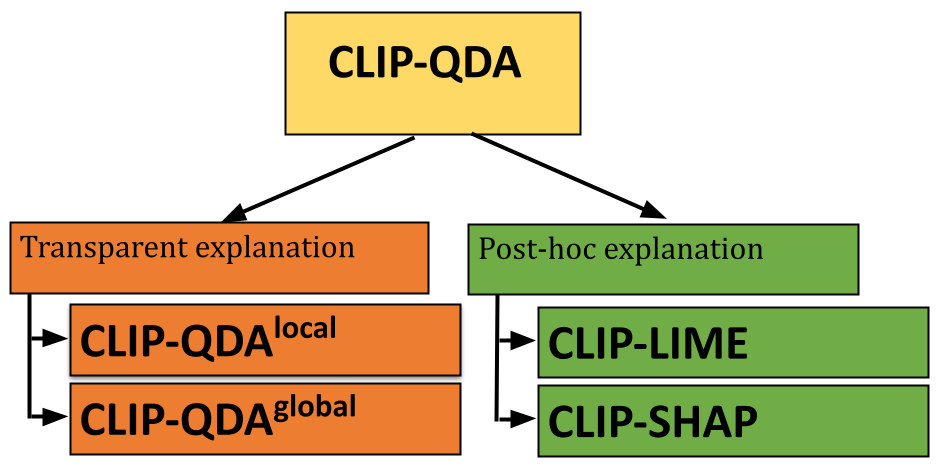}
    \caption{\textbf{The different source of explanation of \method. }
 Global (\methoddata)  vs. Local (\methodsample)  explanations offer insights into classifier behavior across the dataset and on individual samples, respectively.  Post-hoc explanations with \CBMshap and \CBMLIME use  traditional  XAI techniques.}
    \label{organigram_XAI}
\end{figure}

\subsubsection{\updateremi{Dataset-wise explanation with \methoddata }} \label{dataset_wise}

As we have access to priors that describe the distribution of each class, a valuable insight to gain an understanding of which concept our classifier aligns with is the measurement of distances between these distributions.
Specifically, we focus on the \remi{conditional} distributions of two classes of interest~$c_1$ and~$c_2$, that we denote by $Z^{j}_{Y=c_1}$ and $Z^{j}_{Y=c_2}$. The underlying intuition behind measuring the distance between these distributions is that the larger the distance, the more the attribute $j$ can differentiate between the classes~$c_1$ and~$c_2$. 

To accomplish this, we propose to use the Wasserstein-2 distance \updateremi{\citep{ramdas2017wasserstein}} as a metric for quantifying the separation between the two conditional distributions. It is worth noting that calculating the Wasserstein-2 distance can be a complex task in general. However, for Gaussian distributions, there exists a closed-form solution for computing the Wasserstein-2 distance. In addition, we sign the distance to keep the information of the position of $c_1$ relative to $c_2$:
\begin{align}
\tilde{W_2}(Z^{j}_{Y=c_1},Z^{j}_{Y=c_2}) & = \text{sign}( [\vmu_{c_1}]_{(j)}-[\vmu_{c_2}]_{(j)}) \left( ([\vmu_{c_1}]_{(j)}-[\vmu_{c_2}]_{(j)})^2+ \Lambda^{j}_{c_1, c_2} \right) , \nonumber
\end{align}
where $\Lambda^{j}_{c_1, c_2}=[\vSigma_{c_1}]_{(j,j)}+[\vSigma_{c_2}]_{(j,j)}-2 \remi{ \sqrt{[\vSigma_{c_1}]_{(j,j)} [\vSigma_{c_2}]_{(j,j)} }} $.

Note that the resulting value is no longer a distance since we lost the commutativity property. Examples of explanations based on this metric are given in Sections~\ref{examples} and \ref{examples_bis}. \updateremi{Also, an alternative way to produce dataset-wise explanations, oriented on example-based explanations, is also available in Section \ref{centroid_dist}.}

\subsubsection{\updateremi{Sample-wise explanation with \methodsample}} \label{sample_wise}

One would like to identify the key concepts associated with a particular image that plays a pivotal role in achieving the task's objective. To delve deeper into the importance and relevance of concepts in the decision-making process of the classifier, a widely accepted approach is to generate counterfactuals \citep{plumb2022finding,luo2023zero,kim2023grounding}. \remi{If a small perturbation of a concept score changes the class, the concept is considered important. We now formalize this mathematically.}

Consider a pre-trained classifier denoted as $h_{\vomega^h}(\cdot)$. In this context, $\vomega^h$ represents the set of weights associated with the \method, specifically $\vomega^h = (\vSigma_c,\vmu_c,p_c)$. 
\remi{Given a score vector $\vz$, }we define counterfactuals as hypothetical values $\vz + \vepsilon^{j}_{\remi{s}}$, \remi{$\vepsilon^{j}_{\remi{s}}$ being called the perturbation. This perturbation aims to be of minimal magnitude and is obtained by solving the following optimization problem:}
\remi{\begin{align}
\min \|\vepsilon^{j}_{\remi{s}}\|^2 ~~ \textnormal{s.t.} ~~ h_{\vomega^h}(\vz+\epsilon^{j}_{s}) \neq h_{\vomega^h}(\vz) \, .
\label{equ:Causal_Concept_Order}
\end{align}}
The idea behind this equation is to find the minimal perturbation $\vepsilon^{j}_{\remi{s}}$ of the input $\vz$ that makes the classifier produce a different label than $h_{\vomega^h}(\vz)$.
However, in our case, two important restrictions are applied to~$\vepsilon^{j}_{\remi{s}}$:

\begin{enumerate}
\item  \textbf{Sparsity}: \textcolor{black}{for interpretability,} we only change one attribute at a time, indicated by the index $j$. Then $\vepsilon^{j}_{\remi{s}} = [0,..,0,\epsilon^{j}_{\remi{s}},0,...,0]$. 
\item \textbf{Sign}: \remi{we take into account} the sign $s \in \{ -,+ \}$ of the perturbation. Then, we separate the positive counterfactuals $\vepsilon^{j}_{\remi{+}} = [0,..,0,\epsilon^{j}_{+},0,...,0],~\epsilon^{j}_{+} \in \mathbb{R^+}$ and the negative counterfactuals. $\vepsilon^{j}_{-} = [0,..,0,\epsilon^{j}_{-},0,...,0],~\epsilon^{j}_{-} \in \mathbb{R^-} $.
\end{enumerate}

These two constraints are imposed to generate concise and, consequently, more informative counterfactuals. In this context, if a solution to \eqref{equ:Causal_Concept_Order}, denoted as $\vepsilon^{j}_{s,*}$, exists, it represents the minimal modification (addition or subtraction)  \textcolor{black}{to the coordinate $j$ of} the original vector $\vz$ that results in a change from $h_{\vomega^h}(\vz)$ to $h_{\vomega^h}(\vz+\vepsilon^{j}_{s,*})$. \textcolor{black}{Note that this approach allows for an explicit evaluation of the effect of an \textit{intervention}, denoted as 
\remi{$do(\vZ=\vz+\epsilon^{j}_{s})$}
using a common notation in causal inference~\citep{peters2017elements}. Concretely, this emulates  answers to questions of the form: ``Would the label of my cat's image change if I removed a certain amount of its pointy ears?''.}

Another important point to notice is that to obtain all possible counterfactuals, this equation must be solved for all concepts $j$ and both signs $s$.
\remi{A practical way to compute counterfactuals is given below.}

\begin{proposition} \label{demo_sample_wise}
Let us consider a pre-trained \remi{QDA} classifier $h_{\vomega^h}(.)$ \remi{with parameters $\vomega^h$. Assume that the input data is drawn from the corresponding Gaussian Mixture model,} as defined in \ref{CLIP-QDA}, and that \remi{$\vepsilon^{j}_{s}$ a perturbation with the above sparsity and sign restrictions}. Then, there is a closed-form solution to problem~\ref{equ:Causal_Concept_Order}, 
which is a function of the parameters~$(\vSigma_c, \vmu_c, p_c)_{c=1}^C$.
\end{proposition}

The proof \remi{and expression are} given in Appendix~\ref{close_form_solution}.
\remi{We} illustrate the behavior of our classifier and our \textit{\remi{local}} metric with a toy example which consists of two Gaussians ($C=2$) among two concepts $Z_1$ and $Z_2$ ($N=2$). We find the counterfactuals for both signs following the first concept ($\vepsilon^{\remi{1}}_{-,*}$ and $\vepsilon^{\remi{1}}_{+,*}$). Results are presented in Figure~\ref{fig:illus_conterfact}.

\begin{figure}
     \centering
     \begin{subfigure}[t]{0.3\textwidth}
         \centering
         \includegraphics[width=\textwidth]{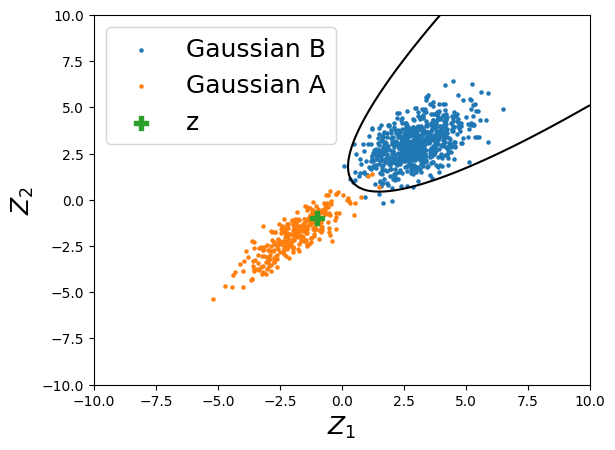}
         \caption{In this case, the input sample (in green) does not have counterfactuals by only changing $z_1$.}
         \label{fig:illus_conterfact_0s}
     \end{subfigure} \hfill
     \begin{subfigure}[t]{0.3\textwidth}
         \centering
         \includegraphics[width=\textwidth]{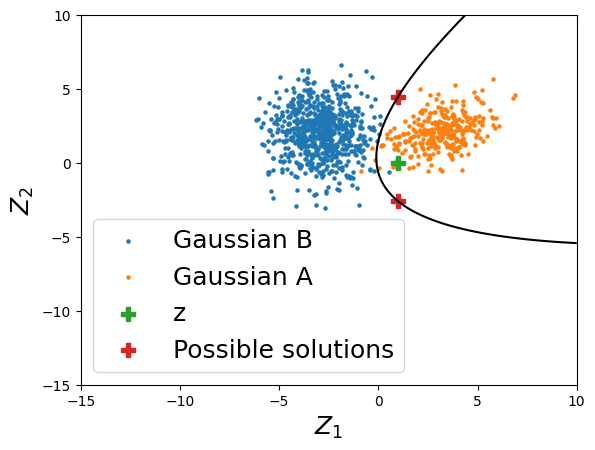}
         \caption{In this case, the input sample (in green) has counterfactuals $\vepsilon^1_{-,*}$ and $\vepsilon^1_{+,*}$. Changing $z_1$ leads to two intersections with the equiprobability line, one by adding score, the other by removing score.}
         \label{fig:illus_conterfact_2s}
     \end{subfigure} \hfill
     \begin{subfigure}[t]{0.3\textwidth}
         \centering
         \includegraphics[width=\textwidth]{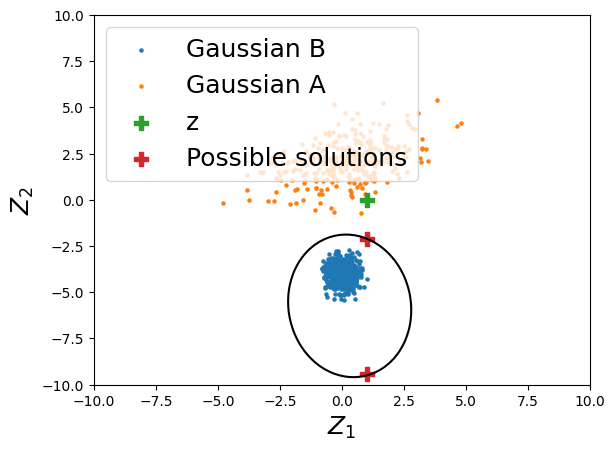}
         \caption{In this case, the input sample (in green) only has one counterfactual $\vepsilon^1_{-,*}$ and no counterfactual $\vepsilon^1_{+,*}$. Changing $z_1$ leads to two intersections with the equiprobability line, but since both are by removing score, only the one of minimal magnitude is conserved.}
         \label{fig:illus_conterfact_1s}
     \end{subfigure} \hfill
        \caption{\textbf{Visualization of the counterfactuals in the two Gaussians toy example.} Samples of the two distributions are plotted in blue and orange. The equiprobability line is plotted in black. }
        \label{fig:illus_conterfact}
\end{figure}

It is worth noting that these counterfactual values are initially expressed in CLIP score units, which may not inherently provide meaningful interpretability. To mitigate this limitation, we introduce scaled counterfactuals, denoted as $\vepsilon^{j}_{s,*,scaled}$, obtained by dividing each counterfactual by the standard deviation associated with its respective distribution:
\begin{equation}
    \vepsilon^{j}_{s,*,scaled} = \frac{\vepsilon^{j}_{s,*}}{\sqrt{[\vSigma_{c}]_{(j,j)}}} \, .
    \label{scaled_counterfact}
\end{equation}
Then, the value of each counterfactual can be expressed as ``the addition (or subtraction) of standard deviations in accordance to $\vZ_{Y=h_{\vomega^h}(\vz)}$ that changes the label''. Examples of such explanations are given in Sections  \ref{examples} and \ref{examples_bis}.

\subsubsection{\updateremi{\CBMLIME and  \CBMshap }} \label{SHAPLIME-CBM}

\paragraph{\updateremi{\CBMLIME.}}

\updateremi{ To adapt LIME to the operation of CBMs, we begin with the image input $x$. From this input, we calculate the projection in the latent space $\vz= \begin{bmatrix} g_{\vomega^g}(\vx,\vk^{\remi{1}}) & \ldots & g_{\vomega^g}(\vx,\vk^{\remi{N}}) \end{bmatrix}$. Subsequently, following the LIME method, we train a surrogate model to approximate $h_{\vomega^h}$ in the vicinity of $\vz$ 
by training it on a dataset comprised of perturbed inputs around $\vz$. Finally, the explanation is derived from the importance weights of the resulting surrogate model.}

\paragraph{\updateremi{\CBMshap.}}

\updateremi{To adapt SHAP to CBMs, we also consider the projection $\vz= \begin{bmatrix} g_{\vomega^g}(\vx,\vk^{\remi{1}}) & \ldots & g_{\vomega^g}(\vx,\vk^{\remi{N}}) \end{bmatrix}$. To compute statistical values relevant to \CBMshap, we incorporate the projections of all the images in the training set. Given that the number of concepts can be high, we employ the Kernel version of SHAP (Kernel SHAP) to reduce the computational cost. The resulting explanation comprises the Shapley values associated with each of the concepts.}

\section{Experimental Setup} \label{setup}

\subsection{Datasets}

We evaluate our methods on ImageNet \citep{deng2009imagenet}, PASCAL-Part \citep{DBLP:journals/ia/DonadelloS16}, MIT Indoor Scenes dataset \citep{quattoni2009recognizing}, \updateremi{MonuMAI} \citep{lamas2021monumai} and Flowers102 \citep{Nilsback08}. In addition to these well-established datasets, we introduce a custom dataset \remi{named} Cats/Dogs/Cars dataset \updateremi{(Section \ref{Cats_Dogs})}. To construct this dataset, we concatenated two widely recognized datasets, namely, the Kaggle Cats and Dogs Dataset 
\remi{\citep{dogs-vs-cats}} and the Standford Cars Dataset \remi{\citep{KrauseStarkDengFei-Fei_3DRR2013}}. Subsequently, we filtered the dataset to contain images of white and black animals and cars exclusively. This curation resulted in six distinct subsets: ``Black Cars'', ``Black Dogs'', ``Black Cats'', ``White Cars'', ``White Dogs'', ``White Cats''. The primary objective of this dataset is to facilitate experiments under conditions of substantial data bias, such as classifying white cats when the training data has only encountered white dogs and black cats. \updateremi{Specifically, we refer to two distinct scenarios (Table \ref{dataset-composition}): one containing cats and cars of both colors (referred to as the unbiased setup) and the other one with only black cats and white cars (referred to as the biased setup). When none of theses scenarios are referred to (for example in the $Del$ values of Table \ref{perf_cats_dogs} and \ref{perf_pascalpart}), the complete dataset is used.} In its final form, the dataset comprises 6,436 images. \revisremi{Additional information is available in Section \ref{Cats_Dogs}}.

\begin{table}[h]
\updateremi{
\caption{\updateremi{\textbf{Cats/Dogs/Cars compositions used in our study}}}
\label{dataset-composition}
\centering
\begin{tabular}{lccc}
\toprule
 & \textbf{Complete dataset} & \textbf{Biased setup} & \textbf{Unbiased setup} \\
\midrule 
Composition & \thead{Black Cats, White Cats, Black Cars, \\  White Cats, Black Dogs, White Dogs} & Black Cats, White Cars & \thead{Black Cats, White Cats, \\ Black Cars, White Cars} \\
Num samples & 6436 & 2536 & 4031 \\
\bottomrule
\end{tabular}}
\end{table}

\subsection{Baselines} \label{baselines}

\paragraph{\revisremi{Classifiers.}}

\revisremi{Here, we present the baseline used for comparing performance in terms of inference (i.e., accuracy) with other algorithms performing classification tasks. First, we evaluate various classifiers, which are represented as the trainable component illustrated in Figure \ref{fig:baseline}. We denote the method proposed by Yan et al., which involves training a linear layer as a classifier from CLIP scores, as one of the baselines. As used notably in \citep{yan2023learning,yan2023robust}, this approach is referred to as linear probe. Additionally, we adopt LaBo, introduced by \citet{yang2023language}, which employs a class-concept matrix as the classifier of the CBMs. For all these methods, the same set of concepts is used, with a comprehensive description of the procedure for acquiring the concept set provided in Appendix \ref{concepts_sets}. We also incorporate previous CBMs such as Greybox XAI \citep{bennetot2022greybox}, X-NeSyL \citep{diaz2022explainable}, and the method proposed by \citet{morales4402768fusion} into our study, which provides explainability but requires \revisremi{additional} training data \revisremi{related} with annotated concepts. It is noteworthy that these methods, requiring training from images, entail significantly longer training times. Additionally, we assess methods widely used as classifiers that are not CBMs, including the use of CLIP as a zero-shot classifier. We also include ResNet \citep{he2016deep} and ViT \citep{dosovitskiy2020image} trained in a supervised manner using images as inputs. Further implementation details can be found in Section \ref{impl_details}.}

\paragraph{XAI Methods.}

\revisremi{To compare our proposed XAI methods (\methoddata, \methodsample, \CBMLIME, and \CBMshap) with existing works, we evaluated several methods, including those providing explanations at both the image and concept levels. At the concept level, we examined explanations provided by LaBo \citep{yang2023language} and Yan et al. \citep{yan2023robust}. For image-based explanations, we assessed SHAP \citep{lundberg2017unified} and LIME \citep{lime}. Further implementation details can be found in Section \ref{impl_details}.}

\subsection{CBM Quantitative Evaluation Process} \label{CBM_benchmark}

\updateremi{Assessing the quality of explanations has long been a challenge, given its subjective nature. While quantitative evaluation processes exist for evaluating performance in a general setup \citep{hedstrom2023quantus}, they may not be well-suited to the specific characteristics of CBMs. Consequently, we present a novel method designed to evaluate XAI solutions, tailored to CBMs. This method comprises two metrics: the Deletion metric, gauging faithfulness to the model, and the Detection metric, quantifying faithfulness to the data.}

\paragraph{\updateremi{Deletion metric.} }

\updateremi{The deletion metric is adapted from the methodology introduced by \citet{petsiuk2018rise}. The procedure involves taking each sample from the test set and nullifying (i.e., setting to the average value of the score across classes) a certain number, $N_{null}$, of concepts. We nullify the concepts based on their importance order as determined by each explanation method. If nullifying the concepts leads to a significant decrease in performance, we consider it a successful selection of concepts that influenced the classifier's decision. \updateremi{The intuition behind this idea is that if the concept is important, its absence will result in a loss of performance.} Given the different values of accuracy obtained by nullifying $N_{null}$ concepts, denoted by $Acc(N_{null})$, we deduce the deletion score $Del$ by computing the area under the curve of $Acc(N_{null})$. \updateremi{The interest is to probe the ability to correctly order the important concepts in the explanation}: }
\updateremi{
\begin{equation}
    Del = \frac{1}{Acc(0)}\sum_{i=1}^{N_{max}} \frac{Acc(i-1)+Acc(i)}{2} \, .
    \label{deletion_score}
\end{equation}}
\updateremi{Here, $N_{max}$ represents a hyperparameter that defines the maximum number of deletions considered. The selection of this hyperparameter is critical: it balances between capturing the metric's capacity to fit the model and ensuring that the computed inputs remain plausible. A default value of 9 is assigned to maintain this equilibrium. \updateremi{Note that we normalize the result by the maximum accuracy $Acc(0)$ to allow for better comparisons among different classifiers.}} 

\updateremi{The evaluation framework is structured into two setups: the first one uses the concept set outlined in Table~\ref{tableconcepts} (referred to as Set 1), while the second one uses an equivalent number of concepts randomly chosen from a dictionary of words (referred to as Set 2). } 

\paragraph{\updateremi{Detection metric.}}

\updateremi{
The detection metric evaluates the model's ability to identify relevant concepts. For each sample $s$, a set of ground truth concepts $\mathcal{S}_s$ to detect \updateremi{is defined by an oracle.} Then, we construct a set of concepts $\mathcal{T}_s$ consisting of the top $|\mathcal{S}_s|$ concepts in the explanation to be tested. 
The resulting detection metric is the average ratio of agreements among all samples in the test dataset:
\begin{equation}
    Det = \frac{1}{S} \sum_{s=1}^S \frac{|\mathcal{S}_s \cap \mathcal{T}_s|}{|\mathcal{S}_s|} ,
    \label{detection_score}
\end{equation}
where $S$ is the total number of samples in the dataset. 
The selection of ground truth concepts depends on the available options within each dataset, as elaborated in Section \ref{results_CBM_XAI}.
}

\section{Experiments}

\subsection{Assessing the Gaussian Prior Hypothesis} \label{gauss_hyp_test}

In this section, we investigate to which extent the Gaussian prior hypothesis (Equation \ref{equ:gaussian}) holds. To assess this, we use Chi-Square Q-Q plots \citep{chambers2018graphical, mahalanobis2018generalized}, \textcolor{black}{a normality assessment method adapted to multidimensional data.}  We display Chi-Square Q-Q plots on the conditional distribution $({\vZ \mid Y=c})$, with data sourced from PASCAL-Part and class $c$ = ``aeroplane''.   %
 \remi{First, we compute a set of concepts adapted to the output class, following the procedure described in Appendix~\ref{concepts_sets}.  
Subsequently, we perform the same experiment with a set of words specifically dedicated to the class of interest (Figure \ref{fig:Aeroplanesuited}). In addition, we show a visualization with randomly chosen concepts from the PASCAL-Part set of concepts, as depicted in Figure \ref{fig:Aeroplane5}.}

\begin{figure}[H]
     \centering
     \begin{subfigure}[b]{0.4\textwidth}
         \centering
         \includegraphics[width=\textwidth]{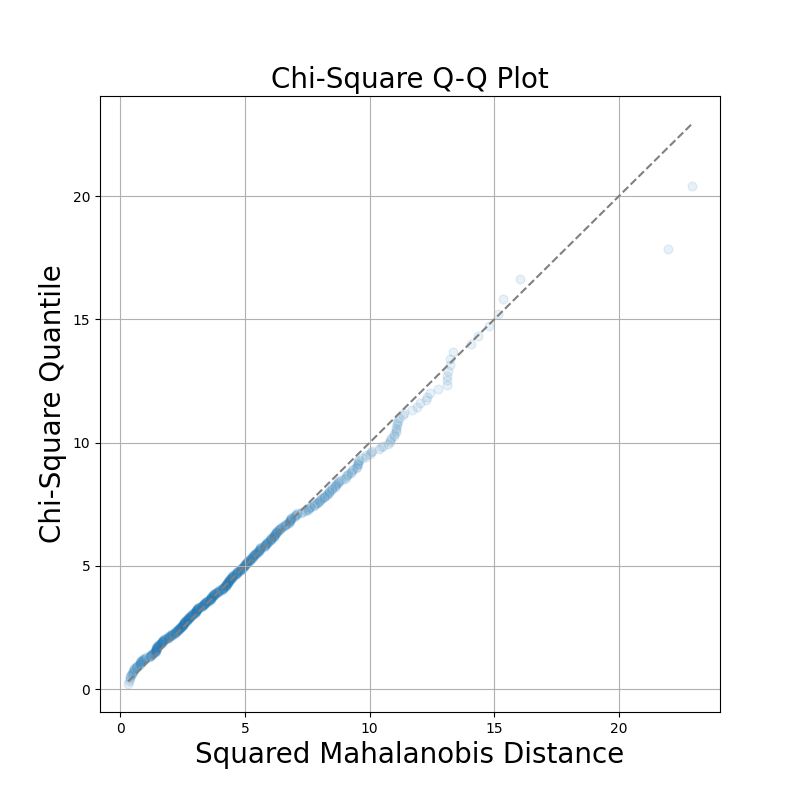}
         \caption{\textbf{\remi{CBM with a set of concepts related to the label ``Aeroplane''}}: [Winged, Jet engines, Tail fin, Fuselage, Landing gear] }
         \label{fig:Aeroplanesuited}
     \end{subfigure}
     \hfill
     \begin{subfigure}[b]{0.4\textwidth}
         \centering
         \includegraphics[width=\textwidth]{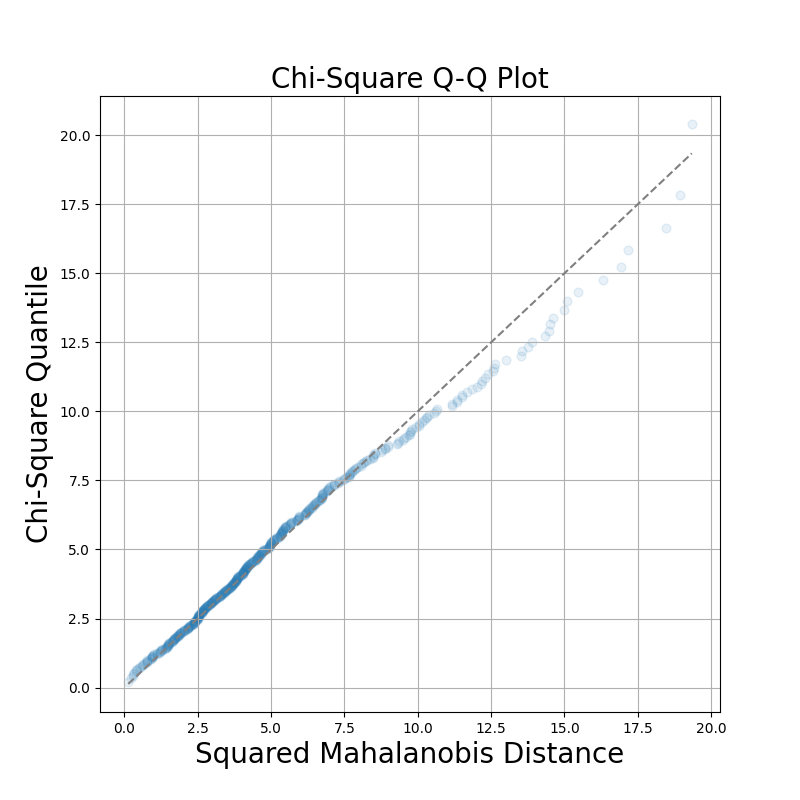}
         \caption{\textbf{\remi{CBM with a set of concepts unrelated to the label ``Aeroplane''}}: [Furry, Equine, Container or pot, Saddle or seat, Multi-doored]}
         \label{fig:Aeroplane5}
     \end{subfigure}
        \caption{\revisremi{\textbf{Multivariate Q-Q plot illustrating the Gaussian fit for conditional modeling of the latent space, given a fixed label.} The plot is specifically for PASCAL-Part images labeled as "Aeroplane".}}
        \label{fig:aeroplanesuit}
\end{figure}

\remi{Notably,} employing a less precise set of concepts can introduce disturbances, as shown in the observations. As indicated in Section \ref{gauss_modeling}, the ambiguity associated with certain concepts, such as ``Multi-doored'', can lead to bimodal distributions (an airplane having one, multiple, or no doors). \remi{In Figure \ref{fig:z_multidoored}, we show the histogram of the clip scores $\vz$ of the images that have the class ``aeroplane''. Compared to the histogram of less ambiguous cases (like in Figure~\ref{fig:Zi1}), we observe that the histogram presents anomalies, especially around the mean.}

\begin{figure}  
    \centering
    \includegraphics[width=0.5\textwidth]{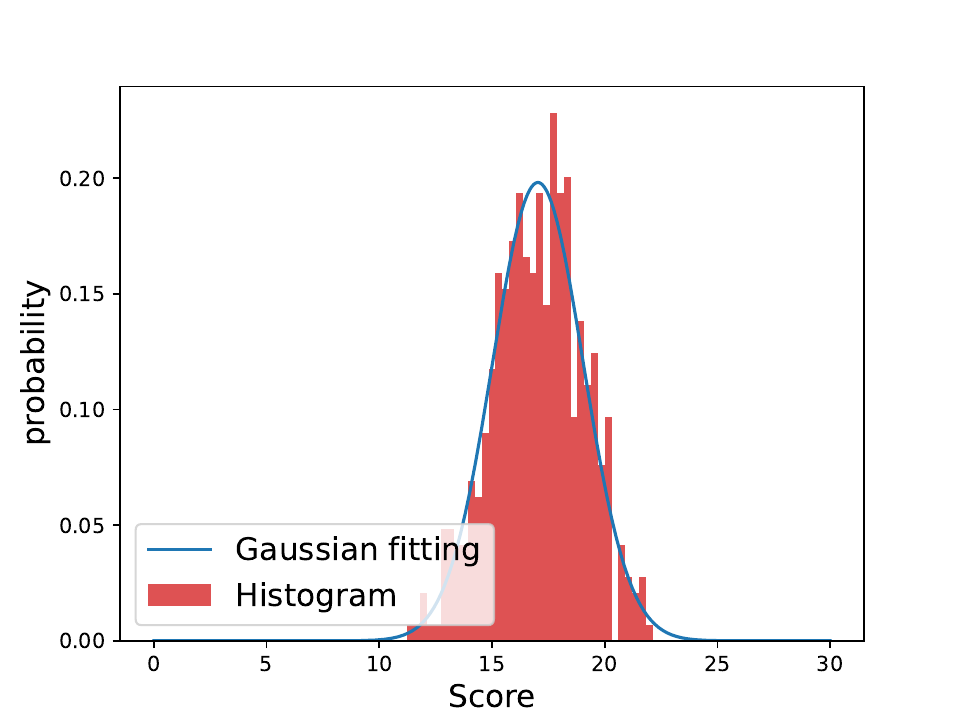}
    \caption{\remi{\textbf{Histogram and Gaussian fitting of CLIP scores $\vz$ of the attribute ``Multi-doored'', for images that have the class ``aeroplane''.}}}
    \label{fig:z_multidoored}
\end{figure}

\remi{Additionally, we performed a similar experiment with larger sets of concepts. We selected random subsets containing 10, 15, and 20 concepts from the PASCAL-Part set listed in Table \ref{tableconcepts}. The outcomes are displayed in Figure \ref{fig:aeroplanemulc}.} In this scenario, it becomes obvious that the Gaussian assumption is increasingly violated as the number of concepts grows. Indeed, as the number of concepts increases, the likelihood of encountering ambiguous concepts in each sample significantly increases, which undermines the feasibility of modeling the data as an unimodal Gaussian distribution.

\begin{figure}[H]
     \centering
     \begin{subfigure}[b]{0.3\textwidth}
         \centering
         \includegraphics[width=\textwidth]{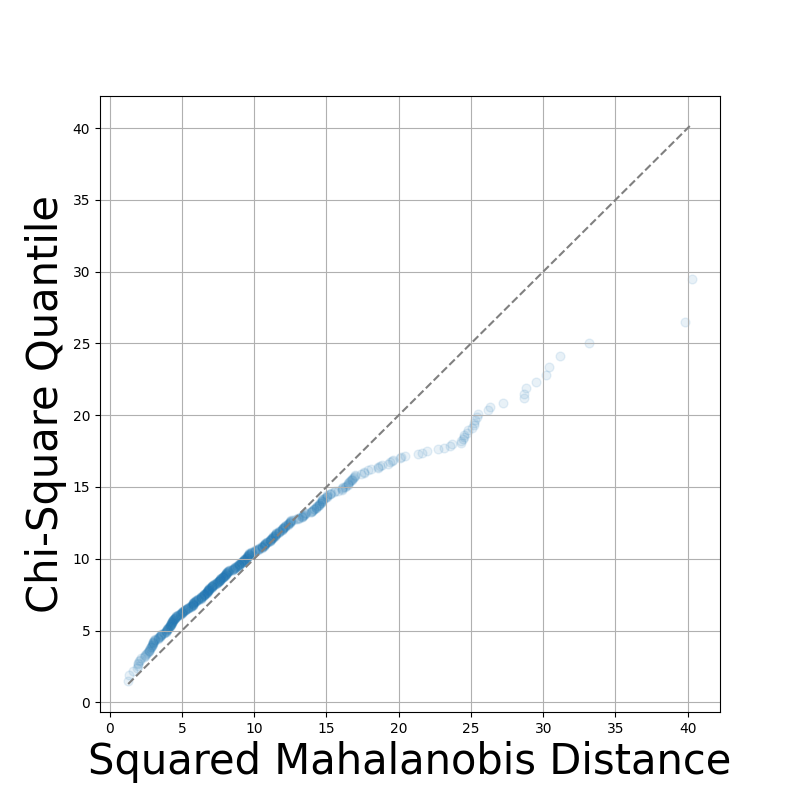}
         \caption{\remi{CBM with} 10 concepts}
         \label{fig:aeroplane10c}
     \end{subfigure}
     \hfill
     \begin{subfigure}[b]{0.3\textwidth}
         \centering
         \includegraphics[width=\textwidth]{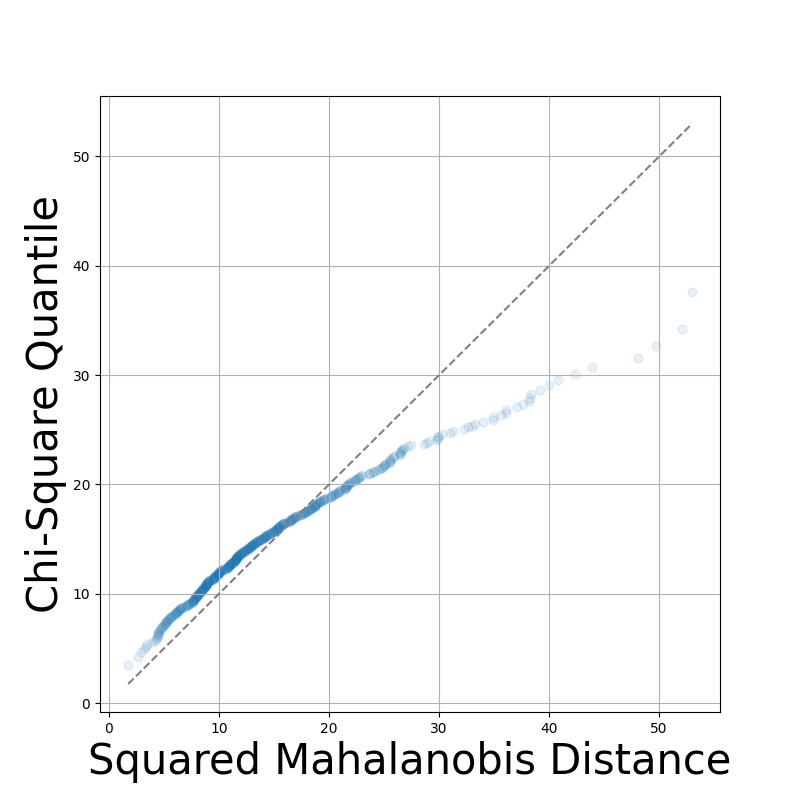}
         \caption{\remi{CBM with} 15 concepts}
         \label{fig:aeroplane15c}
     \end{subfigure}
     \hfill
     \begin{subfigure}[b]{0.3\textwidth}
         \centering
         \includegraphics[width=\textwidth]{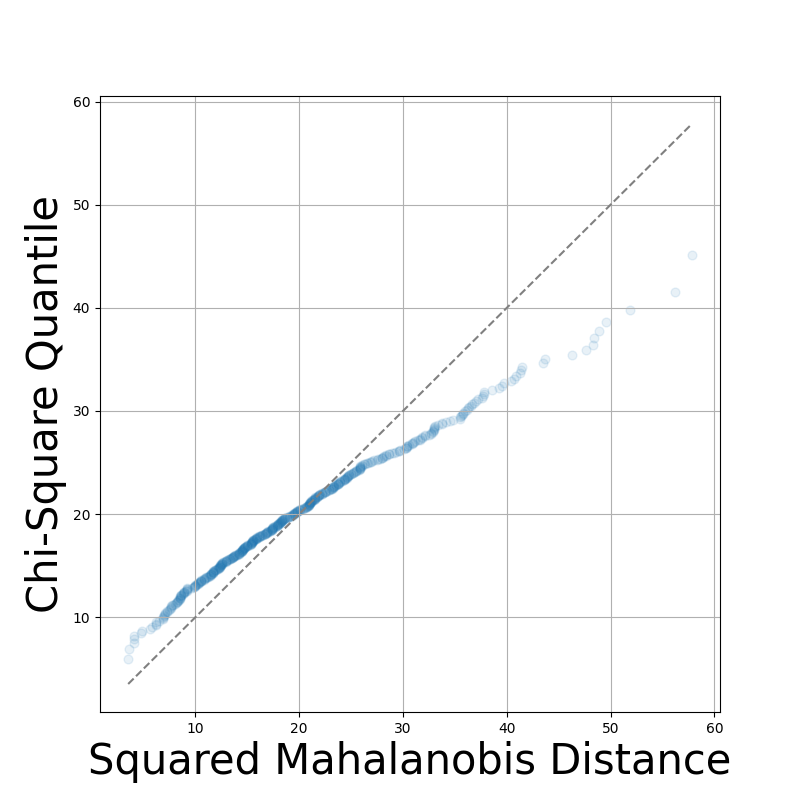}
         \caption{\remi{CBM with} 20 concepts}
         \label{fig:aeroplane20c}
     \end{subfigure}
        \caption{\revisremi{\textbf{Multivariate Q-Q plot illustrating the Gaussian fit for conditional modeling of the latent space, given a fixed label.} The plot is specifically for PASCAL-Part images labeled as "Aeroplane".}}
        \label{fig:aeroplanemulc}
\end{figure}

\subsection{Assessing the Accuracy}

\paragraph{Comparison with other classifiers. }

In this section, we undertake a comparative evaluation of the performance of our \method classifier in contrast to \revisremi{classifiers used in the baseline (see Section \ref{baselines}).} \updateremi{Results of these experiments are available in Table \ref{sample-table}. The image encoder used in all our experiments based on CLIP-based CBMs is $ViT-L/14@336px$ provided by  \revisremi{\href{https://github.com/openai/CLIP}{OpenAI’s public repository.}}}

\begin{table}[H]
\caption{\updateremi{\textbf{Test set accuracy.} On the top are methods that require full training on images, and on the bottom, CBMs. Because Greybox XAI, X-NeSyL and \citet{morales4402768fusion} require concept annotations, their results on MIT scenes and ImageNet are not available.}}
\label{sample-table}
\begin{center}
\updateremi{
\begin{tabular}{clllll}
\toprule
&\multicolumn{1}{c}{\bf Method}  &\multicolumn{1}{c}{\bf PASCAL-Part} $\uparrow$  &\multicolumn{1}{c}{\bf MIT scenes} $\uparrow$ &\multicolumn{1}{c}{\bf \updateremi{MonuMAI}} $\uparrow$ &\multicolumn{1}{c}{\bf ImageNet} $\uparrow$
\\ \midrule \\
\multirow{2}{*}{\rotatebox[origin=c]{90}{Non XAI}}&\thead{Resnet 50 \\ \citep{he2016deep}}  &0.84&0.86&0.95&0.80\\
&\thead{ViT-L 336px \\ \citep{dosovitskiy2020image}}    &0.95&0.94&0.94&0.85 \\
\\ \midrule \\
\multirow{10}{*}{\rotatebox[origin=c]{90}{CBMs}}&\textit{\thead{ Greybox XAI \\ \citep{bennetot2022greybox}}} & 0.88 & - & 0.94 & - \\
& \textit{\thead{X-NeSyL \\ \citep{diaz2022explainable}}} & 0.82 & - & 0.90 & - \\
& \textit{\citet{morales4402768fusion}} & 0.86 & - & \textbf{0.98} & - \\
&\textit{\thead{ CLIP (zero-shot) \\ \citep{radford2021learning}}}   &0.81&0.63&0,52&0.76\\
&\textit{\thead{LaBo \\ \citep{yang2023language}}}  &0.83&0.75&0.74&0.69\\
&\textit{\thead{\citet{yan2023robust}, \\ \citet{yan2023learning} }} &\textbf{0.91}&0.77&0.77 & \textbf{0.81} \\
&\thead{\textit{\method} \updateremi{(ours)}}  &0.90&\textbf{0.81}&0.89&0.60 
\end{tabular}}
\end{center}
\end{table}

\remi{Our findings reveal that fine-tuning using CBM, either as a linear or a QDA probe, significantly improves performance, as evidenced by the increase in accuracy compared to using CLIP as a zero-shot classifier. This improvement is particularly pronounced on datasets dedicated to specialized tasks, such as \updateremi{MonuMAI}. Additionally, CLIP CBMs tend to achieve performances comparable to networks trained from raw images, making these models appealing for image classification due to their reduced training cost in both time and resources, as well as their interpretability.
Notably, \method demonstrates competitive performance, when compared to linear probe techniques.}

\revisremi{However, \method faces challenges in delivering competitive results for datasets like PASCAL-Part and ImageNet, which feature a significantly larger number of labels and concepts. %
The accuracy decline could be attributed to the use of a broader set of concepts tailored specifically for these datasets. This challenges the Gaussian assumption and potentially impacts the effectiveness of our classifier. Notably, there appears to be a correlation between the tests conducted on the different datasets and the observations depicted in Figure \ref{fig:aeroplanemulc}. For instance, while MonuMAI uses 20 concepts, MIT scenes use 25, PASCAL-Part employs 80, and ImageNet involves a staggering 5000 concepts.}

\paragraph{Influence of the number of concepts. }

To assess the influence of the number of concepts $C$ on the accuracy, we conduct experiments on the PASCAL-Part dataset. These experiments involve testing accuracy for both \remi{QDA} and linear probe with concept sets of different lengths, all generated following the methodology described in Section \ref{concepts_sets}. As seen in Figure~\ref{fig:ablation_nconcepts}, \method \remi{performs better than linear probe} when the number of concepts is relatively low. In contrast, the linear probe outperforms \method as the number of concepts increases. This observation aligns with the insights gained from the discussion on Gaussian modeling in Section~\ref{gauss_hyp_test}, where a higher number of concepts challenges the grounding assumptions of \method.

\begin{figure}[h]
    \centering
    \includegraphics[width=0.5\textwidth]{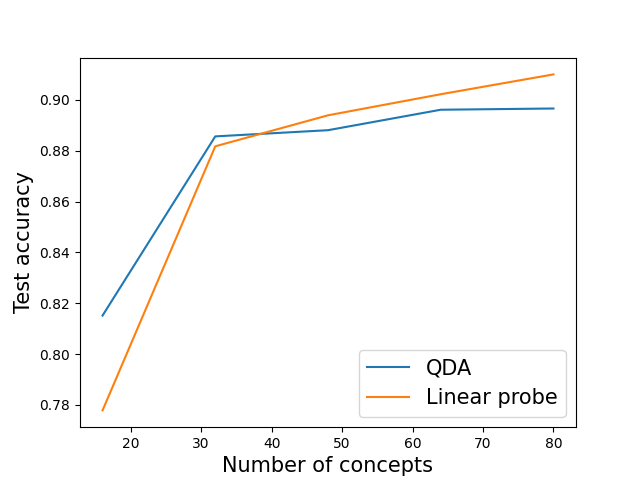}
    \caption{\textbf{Accuracy of the classifiers \updateremi{on PASCAL-Part} for different concept sizes.} QDA refers to the use of \method as a classifier. Linear probe refers to the use of a linear layer as a classifier.}
    \label{fig:ablation_nconcepts}
\end{figure}

\subsection{\updateremi{Assessing the Sample-wise Interpretability of XAI Methods}}

\updateremi{The objective of this subsection is to present our method within the context of current XAI methods (see Section \ref{baselines}) applied to CBMs. A comparative overview of the distinctive features of these methods is presented in Table \ref{feature XAI methods}. Since Greybox XAI, X-NeSyL and \citet{morales4402768fusion} require additional annotations, we chose to omit these methods from our study.
}

\begin{table}[H]
\caption{\updateremi{\textbf{Comparison of the different features of the XAI methods.} Top: existing methods. Bottom: ours. \textit{Dataset-wise} refers to methods that provide dataset-wise explanations. \textit{Sample-wise} refers to methods that provide sample-wise explanations. \textit{Closed-form solution} refers to methods that do not require an optimization process to produce explanations. \textit{Computable from weights} refers to methods that produce explanations from the model parameters and input values. \textit{Image level} refers to methods that produce explanations from the image input. \textit{Concept level} refers to methods that produce explanations from the concept input. \textit{Compatible with \method} refers to methods that can produce explanations with \method as a classifier. \textit{Compatible with CLIP linear probe} refers to methods that can produce explanations with CLIP linear probe as a classifier.}}

\label{feature XAI methods}
\centering
\updateremi{
\begin{tabular}{|c|c|c|c|c|c|c|c|c|}
\hline
\bf Metric & \textit{\thead{Dataset \\ wise}} & \textit{\thead{Sample \\ wise}} & \textit{\thead{Closed-form \\ solution}} & \textit{\thead{Computable \\ from weights}} & \textit{\thead{Image \\ level}} &  \textit{\thead{Concept \\ level}} & \textit{\thead{Compatible \\ with \\ \method}} & \textit{\thead{Compatible \\ with CLIP \\ linear probe}} \\
\hline
\bf GradCAM & & \checkmark & & & \checkmark & & \checkmark & \checkmark \\
\hline
\bf LIME & & \checkmark & & & \checkmark & & \checkmark & \checkmark \\
\hline
\bf SHAP & & \checkmark & & & \checkmark  & &  \checkmark & \checkmark \\
\hline
\bf LaBo & \checkmark & & \checkmark & \checkmark & & \checkmark &  & \\
\hline
\bf Yan et al & \checkmark & \checkmark & \checkmark & \checkmark & & \checkmark & & \checkmark \\
\hline
\bf Greybox XAI & & \checkmark & & & \checkmark & \checkmark & & \\
\hline
\bf X-NeSyL & & \checkmark & & & \checkmark & \checkmark & & \\
\hline
\bf Morales et al & & \checkmark & & & & \checkmark & & \\
\hhline{|=|=|=|=|=|=|=|=|=|}
\bf \CBMLIME & & \checkmark & & & & \checkmark & \checkmark & \checkmark \\
\hline
\bf \CBMshap & & \checkmark & & & & \checkmark & \checkmark & \checkmark \\
\hline
\bf QDA-CBM & \checkmark & \checkmark & \checkmark  & \checkmark & & \checkmark & \checkmark & \\
\hline
\end{tabular}}
\end{table}

\paragraph{\updateremi{Qualitative analysis.}} \label{qualitative_pascalpart}

\updateremi{To facilitate the comparison of various explanations generated for the same sample in the dataset, we computed the results for one image from PASCAL-Part \citep{DBLP:journals/ia/DonadelloS16}. The results are categorized into two sections: image-level explanations and concept-level explanations. The subsequent section focuses on explaining \updateremi{the model's prediction process given} the image \ref{image_person} \updateremi{as the input}, which is labeled as ``person''. The top two predicted labels by the classifier for this image are ``person'' and ``potted plant''. Note that additional samples are available in Section \ref{examples_bis}.}

\updateremi{Image-level explanations are presented in Figure \ref{fig:sample_person_image}. For each of these samples, a weight is assigned to each pixel. 
All the results are computed by applying these post-hoc explanation methods on \method.
Additional implementation and computation details can be found in \ref{impl_details}. While the results are easy to comprehend, criticism may arise due to potential misunderstandings regarding the specific focus of the method. While these methods successfully highlight the object of interest, it is challenging to discern the exact pattern that the method emphasizes. Another observation is that, compared to classical classifiers, the results on CLIP-based CBMs are more imprecise, highlighting large areas of interest. }

\begin{figure}[h]
     \centering
     \begin{subfigure}[t]{0.24\textwidth}
         \centering
         \includegraphics[width=\textwidth]{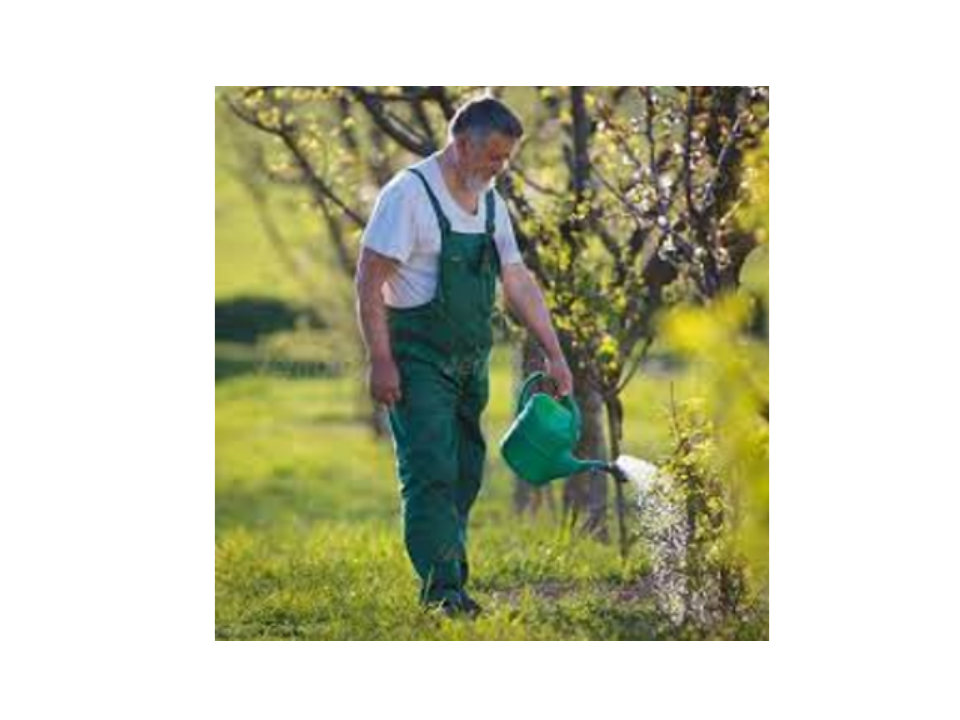}
         \caption{\textbf{\updateremi{Input image}}}
         \label{image_person}
     \end{subfigure}
     \hfill
     \begin{subfigure}[t]{0.24\textwidth}
         \centering
         \includegraphics[width=\textwidth]{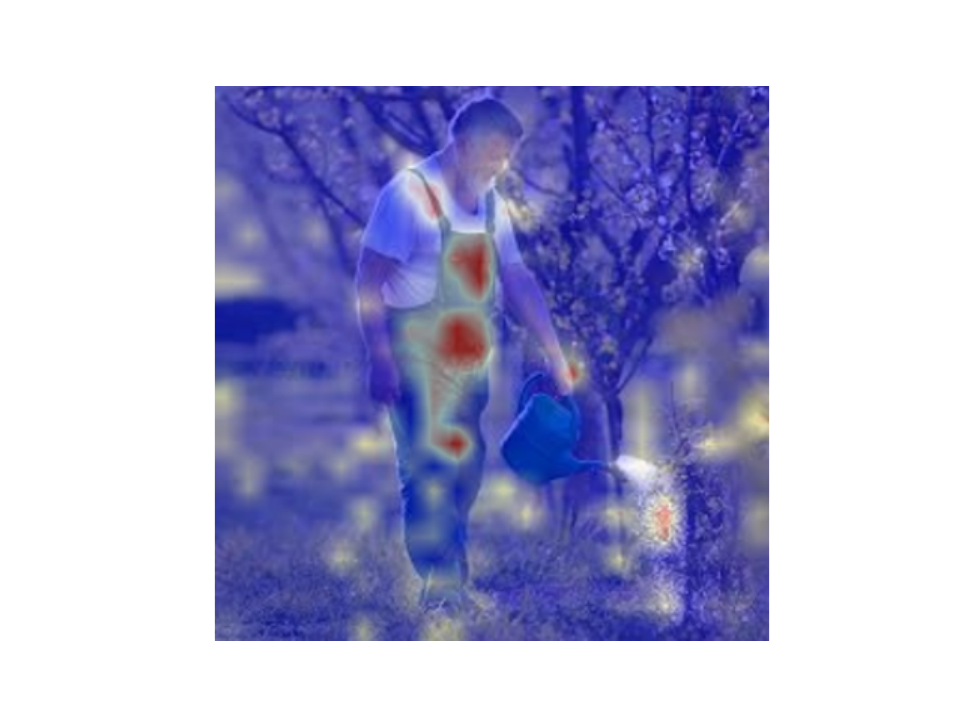}
         \caption{\textbf{\updateremi{GradCAM Explanation}} }
     \end{subfigure}
     \hfill
     \begin{subfigure}[t]{0.24\textwidth}
         \centering
         \includegraphics[width=\textwidth]{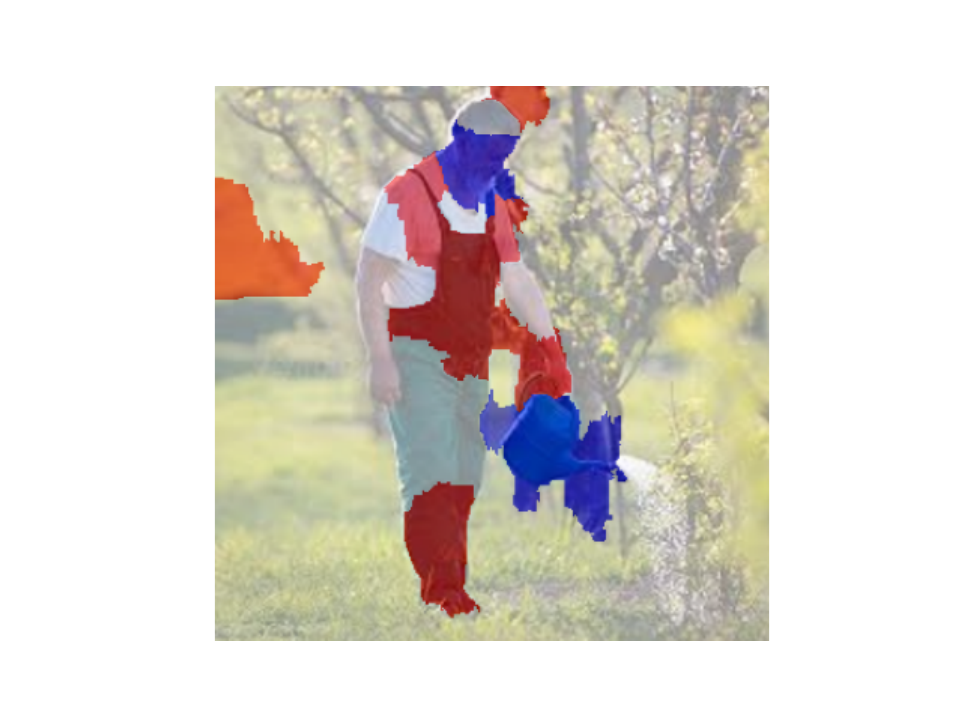}
         \caption{\textbf{\updateremi{LIME Explanation}}}
     \end{subfigure}
     \hfill
     \begin{subfigure}[t]{0.24\textwidth}
         \centering
         \includegraphics[width=\textwidth]{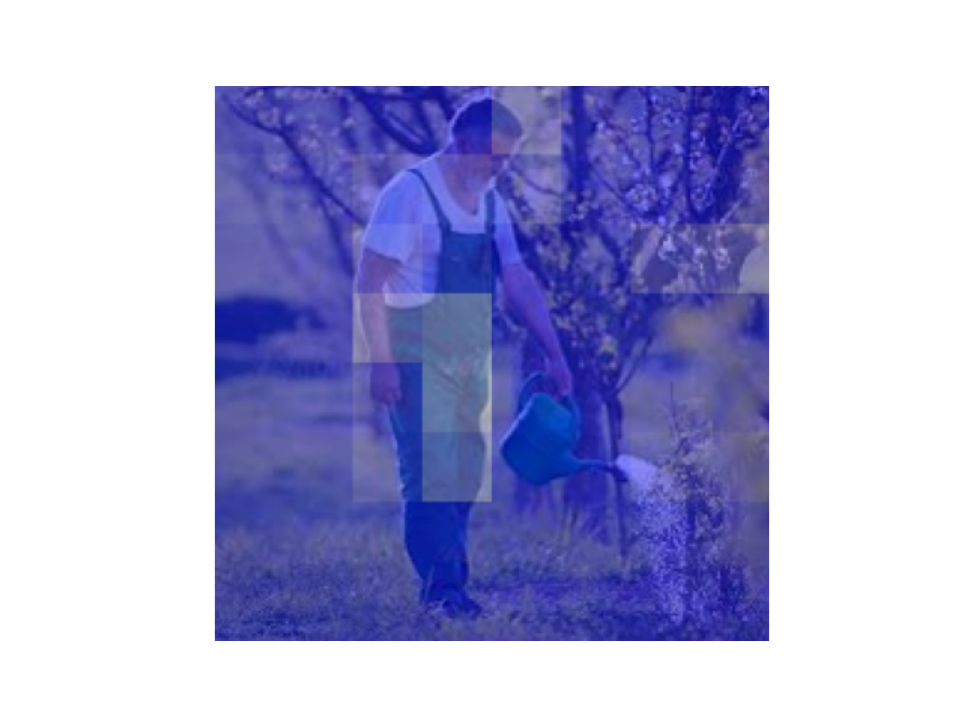}
         \caption{\textbf{\updateremi{SHAP explanation}}}
     \end{subfigure}
        \caption{\updateremi{\textbf{Sample-wise explanations (image level).} Note that the classifier classified the image correctly.}}
        \label{fig:sample_person_image}
\end{figure}

\updateremi{Concept-level sample explanations are computed in Figure \ref{fig:sample_person_concept}. On a sample-wise basis, \CBMLIME and \CBMshap are computed using \method. Given that the method is model-specific, the Yan et al. method is computed on a linear classifier.}

\begin{figure}[htb!]
     \centering
     \begin{subfigure}[t]{0.45\textwidth}
         \centering
         \includegraphics[width=\textwidth]{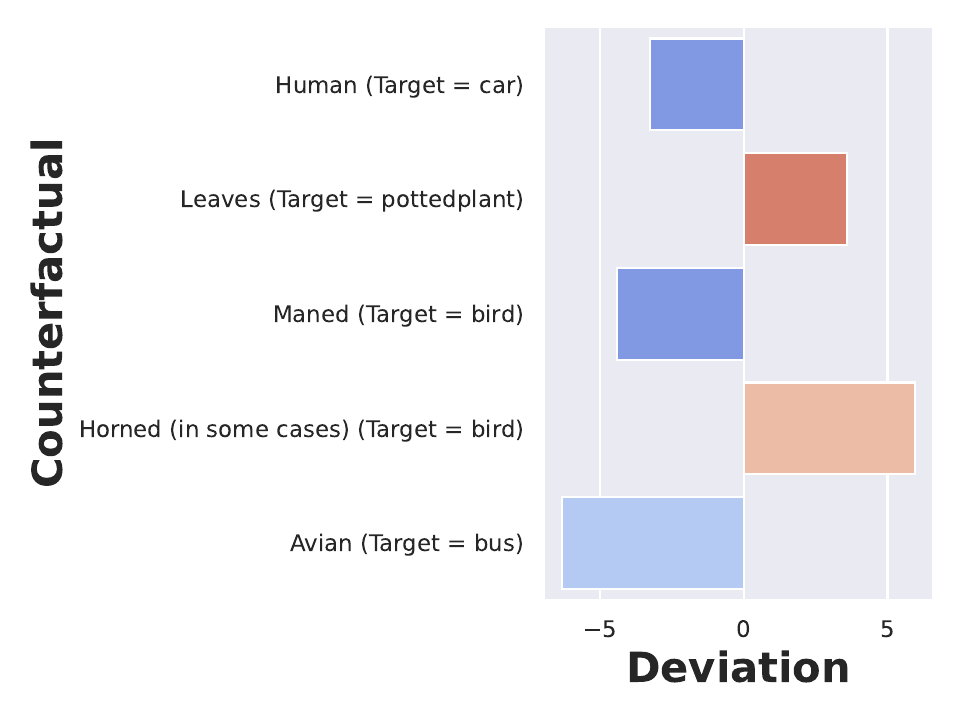}
         \caption{\updateremi{\textbf{\methodsample explanation.} The first row can be read as follows: removing a little of the concept ``Human'' to the vector $\vz$ would change the label to ``Car''.}}
     \end{subfigure}
     \hfill
     \begin{subfigure}[t]{0.45\textwidth}
         \centering
         \includegraphics[width=\textwidth]{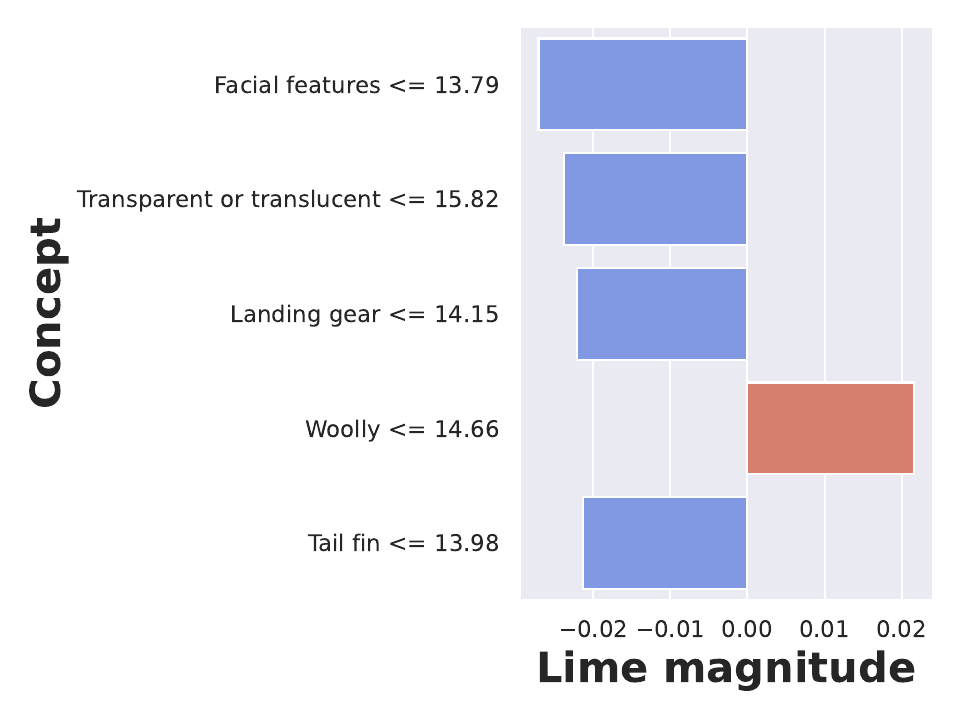}
         \caption{\updateremi{\textbf{\CBMLIME explanation.} The first row can be read as follows: the fact that the concept score of ``Facial features'' is below 13.79 has a negative impact on the predicted label.}}
     \end{subfigure}
     \hfill
     \begin{subfigure}[t]{0.45\textwidth}
         \centering
         \includegraphics[width=\textwidth]{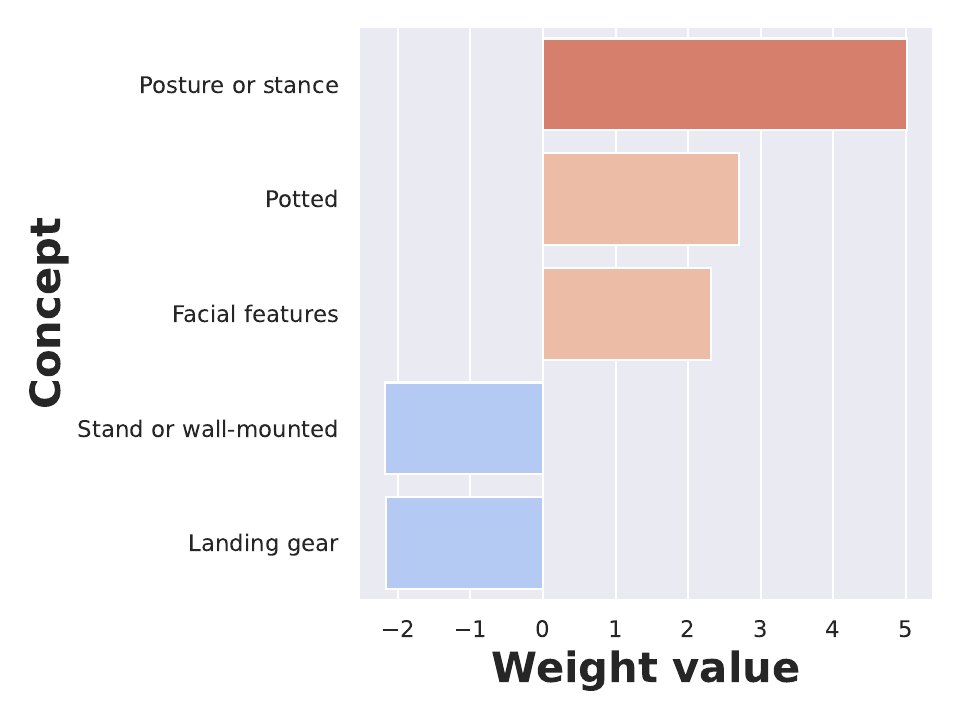}
        \caption{\updateremi{\textbf{Yan et al.~explanation (sample).} The first row can be read as follows: the concept ``Posture or stance'' has a positive impact on the predicted label.} %
        }
     \end{subfigure}
     \hfill
     \begin{subfigure}[t]{0.9\textwidth}
         \centering
         \includegraphics[width=\textwidth]{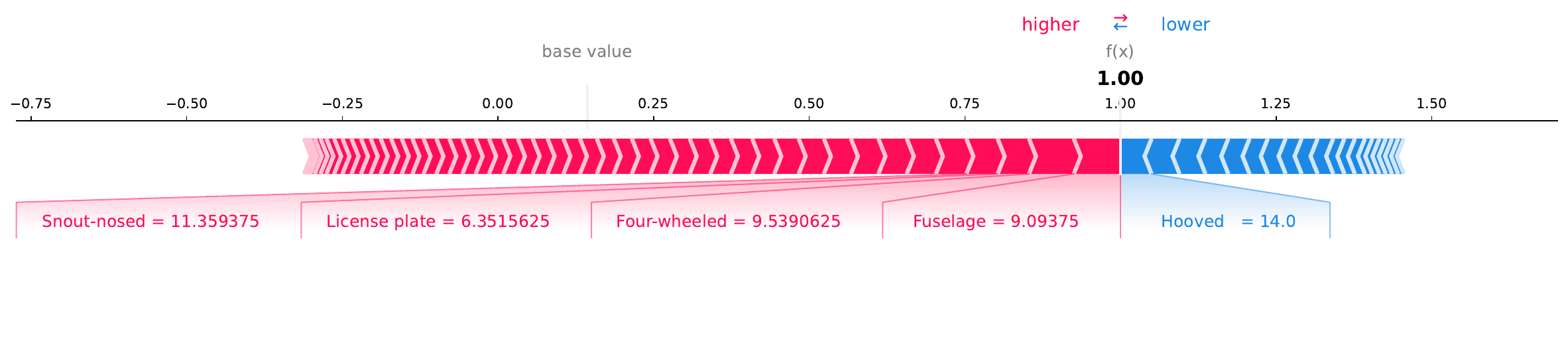}
         \caption{\updateremi{\textbf{\CBMshap explanation.} The highest value can be read as follows: the concept ``Fuselage'' has a positive impact on the predicted label.}}
     \end{subfigure}
     \hfill
    \caption{\updateremi{\textbf{Sample-wise explanations (concept level) of the image \ref{image_person}.} Except for \CBMshap, only the top 5 concepts are displayed. Note that the classifier correctly labeled the image.}}
     \label{fig:sample_person_concept}
     \hfill
\end{figure}

\updateremi{First, we notice that there is a high diversity of concepts displayed among the methods. Overall, there are concepts associated with the prediction label (``Posture or stance'', ``Human'') and concepts that arise from the specificities of the image (``Leaves''). Some concepts seem irrelevant (``Landing Gear'', ``Windshield'', ``Wooly''), which may be explained by the fact that their scores are useful to the classification process despite their absence in the image. %
In addition, it is worth noting that the context provided by the counterfactual analysis helps to unveil insights into the reasons behind the explanation.
For example, when taking into account the concept ``Leaves'' seems inappropriate, \methodsample explanations suggest that adding that concept influences the prediction of the class ``potted plant,'' indicating that this concept is important because it reveals the decision process behind predicting a person rather than a plant.}

\paragraph{\updateremi{Quantitative analysis.}} \label{results_CBM_XAI}

\updateremi{To conduct a more comprehensive analysis beyond qualitative observations, we subject the various methods to \updateremi{quantitative evaluations}. We apply the procedure outlined in Section \ref{CBM_benchmark} to the PASCAL-Part and Cats/Dogs/Cars datasets. In addition to \CBMLIME, \CBMshap, and QDA-CBM explanations, we also tested the explanations provided by LaBo and Yan et al. 
For the selection of $S_s$ \revisremi{in Equation \ref{detection_score}}, in the case of the PASCAL-Part dataset, we opted to designate the concepts to detect
as the concepts that are associated with the two highest probabilities of the model's inference
(refer to the Classes/Concepts association in Table \ref{tableconcepts}). In the Cats/Dogs/Cars dataset, we chose to identify the concepts ``Black'' and ``White'' in the biased setup to \revisremi{spot} potential biases. %
\revisremi{The resulting $Det$ score for this dataset is calculated by considering $S_s = \left[Black,White\right]$ for classifiers trained with biased setups as described in \ref{dataset-composition}. To reduce uncertainties, we do not only take the $Det$ score for the set [Black Cats, White Cars] but also for all possible similar biased binary classification tasks ([Black Cats, White Cars], [Black Dogs, White Cars], etc.).}
Additionally, we incorporated inference time as a parameter in our experiments \updateremi{referring to the time taken to produce explanations inferences on the entire validation set}. Results on PASCAL-Part are displayed in Table \ref{perf_pascalpart} and results on Cats/Dogs/Cars are displayed in Table \ref{perf_cats_dogs}.}

\begin{table}[ht!]
\caption{\updateremi{\textbf{Quantitative results for different XAI methods (Cats/Dogs/Cars dataset).} Top: existing methods. Bottom:
ours. Del refers to the deletion score (\ref{deletion_score}) on a random set of concepts (Set 1) and the 
set of concepts defined in \ref{tableconcepts} (Set 2). Det refers to the detection score (\ref{detection_score}). The explanations of the three upper methods being direct, they present no computation time.}}
\label{perf_cats_dogs}
\centering
\updateremi{
\begin{tabular}{lcccc}
\toprule
\textbf{Method} & \textbf{Del (Set 1)} $\downarrow$ & \textbf{Del (Set 2)} $\downarrow$ & \textbf{Det} $\uparrow$ & \textbf{Inference time (im/s)} $\downarrow$ \\
\midrule
\textit{Yan et al.} & 0.6253 & 0.7522 & 0.3183 & / \\
\textit{LaBo} & 0.5971 & 0.6224 & 0.2140& / \\
\textit{Random} & 0.8278 & 0.7474 & 0.1171& / \\
\\ \hline
\\
\methodsample (ours) & 0.7609 & 0.5820 & 0.2724 & \textbf{3.01} \\
\textit{\CBMLIME} (ours)& 0.5397 & 0.5646 & \textbf{0.4042} & 76.69 \\
\textit{\CBMshap} (ours)& \textbf{0.4821} & \textbf{0.3831} & 0.3696 & 256.76 \\
\bottomrule
\end{tabular}}
\end{table}

\begin{table}[ht!]
\caption{\updateremi{\textbf{Quantitative results for different XAI methods (PASCAL-Part dataset).} Top: existing methods. Bottom:
ours. Del refers to the deletion score (\ref{deletion_score}) on a random set of concepts (Set 1) and the usual set of concepts defined in \ref{tableconcepts} (Set 2). Det refers to the detection score (\ref{detection_score}). The explanations of the three upper methods being direct, they present no computation time.}}
\label{perf_pascalpart}
\centering
\updateremi{
\begin{tabular}{lcccc}
\toprule
\textbf{Method} & \textbf{Del (Set 1)} $\downarrow$ & \textbf{Del (Set 2)} $\downarrow$ & \textbf{Det} $\uparrow$ & \textbf{Inference time (im/s)}$\downarrow$ \\ %
\midrule
\textit{Yan et al.} & 0.6968 & 0.8824 & \textbf{0.4157} & / \\
\textit{LaBo} & \textbf{0.4446} & 0.8822 & 0.4213 & / \\
\textit{Random} & 0.8452 & 0.8999 & 0.1258 & / \\
\\ \hline
\\
\methodsample (ours) & 0.7983 & 0.8189 & 0.1564 & \textbf{2341.12} \\
\textit{\CBMLIME} (ours) & 0.7313 & 0.8511 & 0.0873 & 2857.28 \\
\textit{\CBMshap} (ours) & 0.5698 & \textbf{0.5510} & 0.2679 & 7207.58 \\
\bottomrule
\end{tabular}}
\end{table}

\updateremi{Regarding the Cats/Dogs/Cars dataset, our methods exhibit superior performance compared to the state of the art, emphasizing the reliability of \method on datasets with a low number of concepts. Particularly, QDA-CBM demonstrates significantly faster inference times than \CBMLIME and \CBMshap, albeit at the expense of slightly lower deletion and detection scores. On datasets with a higher number of concepts and classes, such as PASCAL-Part, existing methods maintain higher scores. 
}

\subsection{\updateremi{Assessing the Dataset-wise Interpretability of XAI Methods}} \label{examples}

\subsubsection{\revisremi{Evaluation of \methoddata on the PascalPART dataset}}

\updateremi{Following the setup expressed in Section \ref{qualitative_pascalpart}, for dataset-wise explanations (Figure \ref{dataset_wise_img}), we present the explanations for the \methoddata method, computed on \method, the LaBo explanation, computed on the LaBo classifier, and the Yan et al.~explanation, computed on a linear classifier. \revisremi{Additional samples are available in Section \ref{examples_bis}}}

\begin{figure}[htb!]
     \centering
     \hfill
     \begin{subfigure}[t]{0.49\textwidth}
         \centering
         \includegraphics[width=\textwidth]{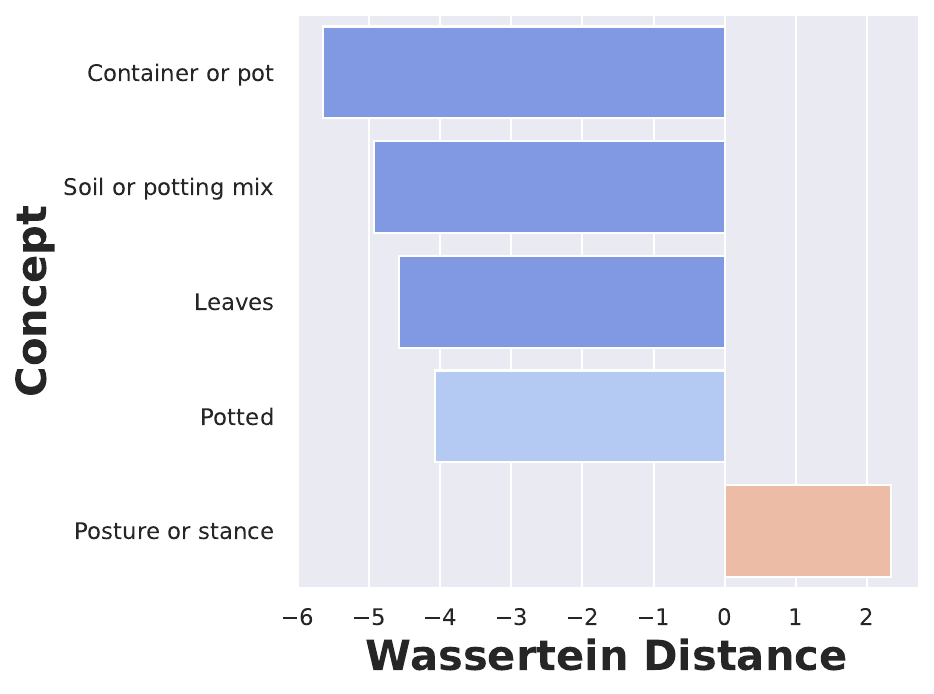}
         \caption{\updateremi{\textbf{\methoddata explanation associated with the top two predictions (``person'' and ``potted plant'').} The first row can be read as follows: the distribution of the concept ``Container or pot'' for the class ``person'' is mostly to the left (smaller values) of the one for the class ``potted plant''.}
         }
     \end{subfigure}
     \hfill
     \begin{subfigure}[t]{0.49\textwidth}
         \centering
         \includegraphics[width=\textwidth]{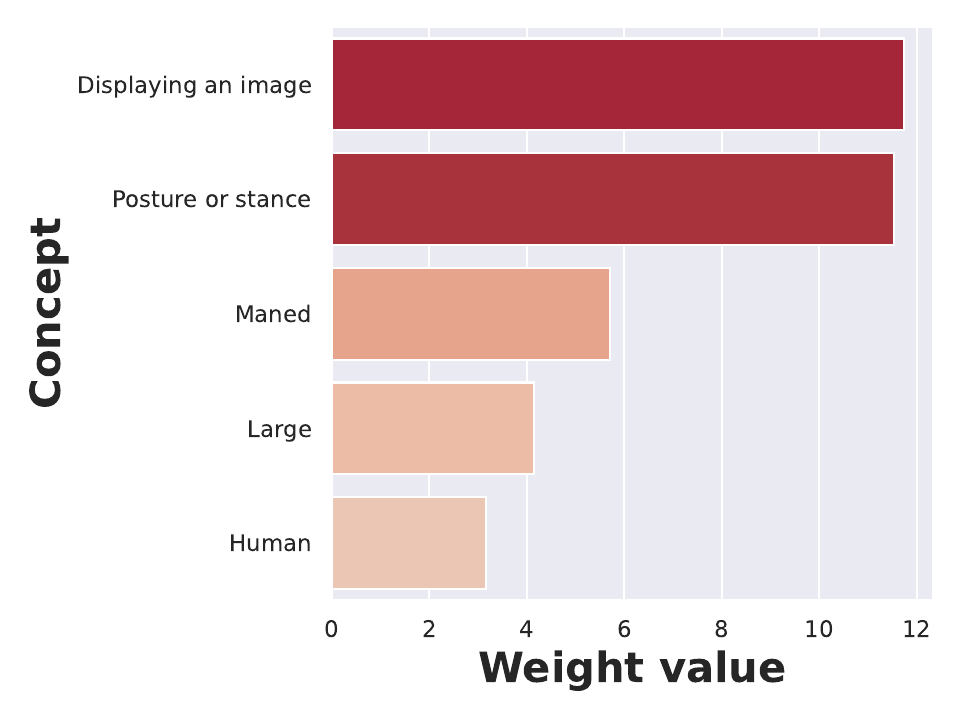}
         \caption{\updateremi{\textbf{LaBo explanation associated with the top prediction (``person'').} The first row can be read as follows: the weight associated with the concept ``Displaying an image'' has the highest value among the weights related to the class ``person''.}}
     \end{subfigure}
     \hfill
     \begin{subfigure}[t]{0.49\textwidth}
         \centering
         \includegraphics[width=\textwidth]{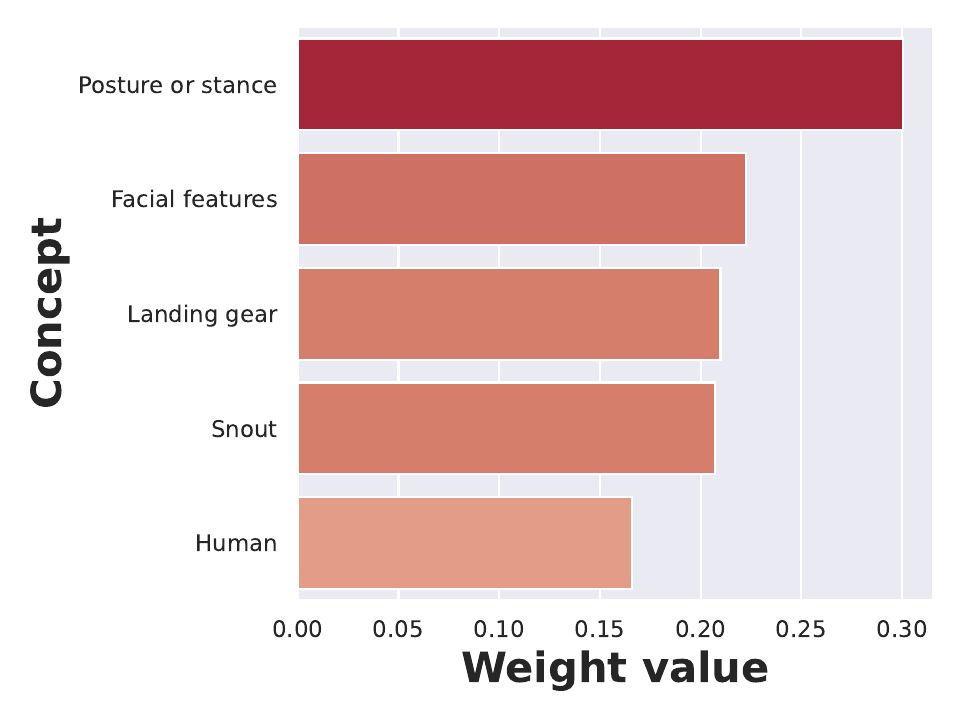}
         \caption{\updateremi{\textbf{Yan et al.~explanation (dataset) associated with the top prediction (``person'').} The first row can be read as follows: the weight associated with the concept ``Posture or stance'' has the highest value among the weights related to the class ``person''.}}
     \end{subfigure}
     \caption{\updateremi{\textbf{Dataset-wise explanations of the image \ref{image_person}.} Only the top 5 concepts are displayed. Note that the classifier correctly labeled the image.} 
     }
     \label{dataset_wise_img}
     \hfill
\end{figure}

\updateremi{We observe that the concepts obtained from all the dataset-wise methods are generally coherent, including concepts commonly associated with persons, such as ``Posture or stance'' or ``Facial features''. However, in contrast to the other two  methods that impose positive magnitudes, the \methoddata explanation allows for negative values, providing insights into whether the impact is negative or positive. 
Additionally, our \methoddata method emphasizes the comparison between two classes of interest, enabling a detailed exploration of specific disparities, like here between ``potted plant'' and ``person''.}

\subsubsection{\updateremi{Evaluation of \methoddata on the Cats/Dogs/Cars dataset}}

We also present a toy example from our Cats/Dogs/Cars dataset by constructing a CBM consisting of the concepts Table~\ref{tableconcepts}, plus the concepts ``Black'' and ``White''. Next, we display the \updateremi{10} most influential concepts according to our global metric (the top \updateremi{10} concepts that have the highest Wasserstein distance) \updateremi{in both biased and unbiased setups. Results are presented in Figure \ref{fig:expbiaised}.}

We can observe that the concepts ``Black'' and ``White'' hold significantly higher importance in the biased setup, indicating that the classifier is likely to be \remi{biased} about these concepts. \textcolor{black}{This shows that our global explanation method has the potential for detecting biases in datasets~\citep{tommasi2017deeper}.} Additional explanation samples for various use cases are available in Appendix \ref{examples_bis}.

\begin{figure}[htb!]
     \centering
     \begin{subfigure}[b]{0.45\textwidth}
         \centering
         \includegraphics[width=\textwidth]{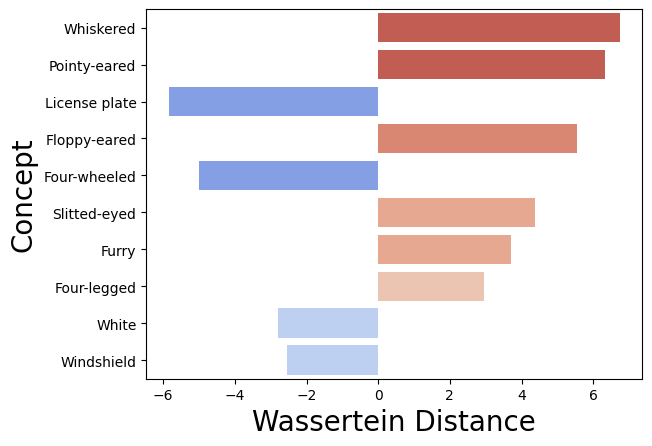}
         \caption{Unbiased setup}
         \label{fig:datasetunbiased}
     \end{subfigure}
     \hfill
     \begin{subfigure}[b]{0.45\textwidth}
         \centering
         \includegraphics[width=\textwidth]{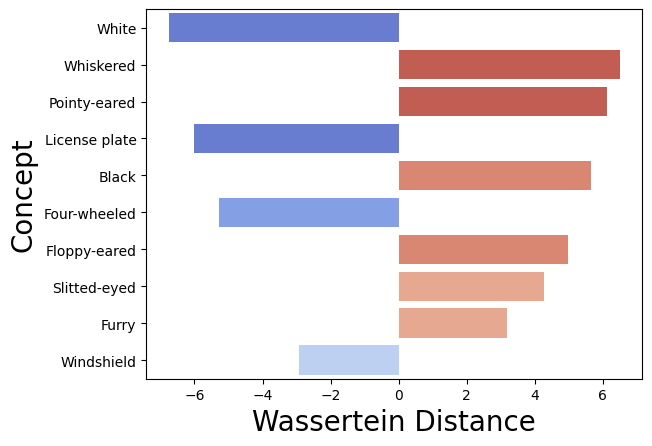}
         \caption{Biased setup}
         \label{fig:datasetbiased}
     \end{subfigure}
        \caption{\textbf{\remi{Global} explanation on subsets of Cats/Dogs/Cars.} Here, $c_1$=``Cat'' and $c_2$=``Car''. \remi{Positive values indicate concepts that are more prevalent in cat images than car images, while negative values indicate concepts that are more common in car images compared to cat images. We display here only the top 10 concepts that have the highest Wasserstein distance (the concept ``Black'' is positioned $15^{th}$ in the unbiased setup).}}
        \label{fig:expbiaised}
\end{figure}

\updateremi{
Finally, we show an application of our \remi{local} metric within the framework of our biased setup, as previously described. 
Subsequently, we feed an image of a white cat into our classifier (Figure \ref{fig:expbiaisedcat}). It is noteworthy that the image is misclassified as a car. Our \remi{local} metric demonstrates sensitivity to the dataset's color bias, corroborating the warning issued by the \remi{global} explanation. }

\begin{figure}[htb!]
     \centering
     \begin{subfigure}[b]{0.45\textwidth}
         \centering
         \includegraphics[width=\textwidth]{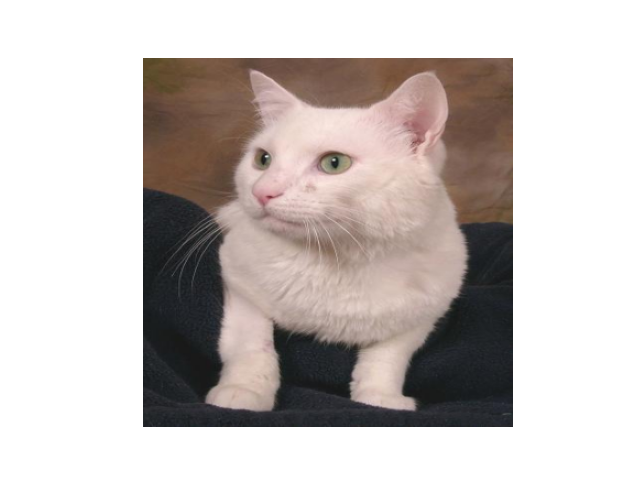}
         \caption{Input image}
         \label{fig:imagewhitecat}
     \end{subfigure}
     \hfill
     \begin{subfigure}[b]{0.45\textwidth}
         \centering
         \includegraphics[width=\textwidth]{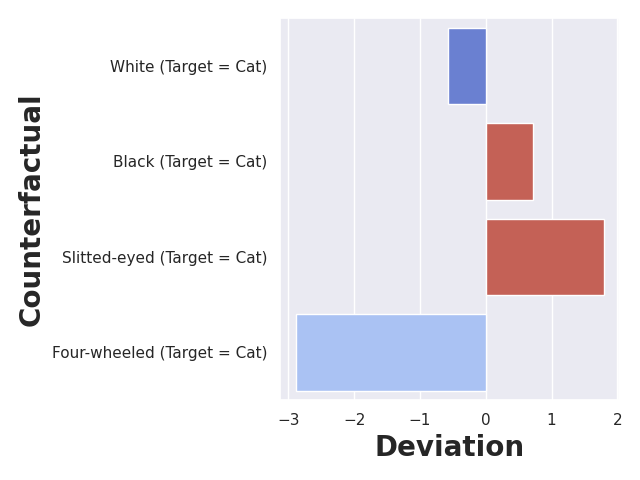}
         \caption{\updateremi{\textbf{\methodsample explanation.} Only the top 5 concepts are displayed. Note that the classifier misclassified the image as ``Car''. The first row must be read as follows: removing a little of the concept ``White'' to the vector $\vz$ induces a change of label to ``Cat''.}
         }
         \label{fig:expwhitecat}
     \end{subfigure}
        \caption{\updateremi{\textbf{\remi{Local} explanation on subsets of Cats/Dogs/Cars.} 
        \remi{On the right figure, the blue/red scale represents the scaled counterfactuals as in equation~\ref{scaled_counterfact}, the text in each box corresponds to the label predicted after the perturbation, followed to the concept changed to obtain its result (in parentheses).} 
        }}
        \label{fig:expbiaisedcat}
\end{figure}

\section{Conclusion}

In this paper, we introduce a modeling approach for the embedding space of CLIP, a foundation model trained on a vast dataset. Our observations reveal that CLIP \remi{can} organize information in a distribution that exhibits similarities with a mixture of Gaussians. Building upon this insight, we develop an adapted concept bottleneck model that demonstrates \remi{competitive} performance \textcolor{black}{along with} transparency.
While the model that we have presented offers the advantage of simplicity and a limited number of parameters, it does encounter challenges when dealing with a broader and more ambiguous set of \remi{concepts}. 
As a suggestion for future research, we propose to extend this modeling approach to incorporate other priors, such as Laplacian distributions, and to explore more complex models, including those with multiple components, \remi{\textit{i.e.}, using more than one Gaussian to describe a class for example}. \updateremi{Another avenue to potentially enhance the performance of our model is to explore guiding the latent space \revisremi{of CLIP's scores} towards a Gaussian distribution.}
Additionally, our research is centered around a specific embedding space (CLIP scores). Exploring similar work on other latent spaces, particularly those associated with multimodal foundation models, could be valuable to determine if similar patterns exist in those spaces.

\section{\revisremi{Acknoledgements}}

\revisremi{This work was performed using HPC resources from
GENCI-IDRIS (Grant 2023 - AD011014675).}

\bibliography{main}

\bibliographystyle{tmlr}

\appendix

\section{Appendix}

\subsection{Implementation details} \label{impl_details}

\paragraph{GradCAM.}

\updateremi{To compute GradCAM explanations, we applied the method to the %
$21^{th}$ block of the image encoder in \method, \revisremi{using the PyTorch package \href{https://github.com/jacobgil/pytorch-grad-cam}{\textit{pytorch-grad-cam}} \citep{jacobgilpytorchcam}. Specifically, we generated the heatmap by backpropagating the gradient from the classifier's inference class, considering the entire network (CLIP+classifier).}
}

\paragraph{LIME.}

\revisremi{The LIME image-level explanations are generated using the superpixel approach defined in \citep{lime} using the \revisremi{\href{https://github.com/marcotcr/lime}{official author's repository}}. All of our samples are generated using the default parameters for image-level explanations, i.e. an exponential kernel of width of 0.25 on an image segmented using quickshift clustering. Concerning the visualization, each explanation underlines the top 10 superpixels that exert the most influence on the prediction. Superpixels that positively contribute to the prediction of the label are in red and superpixels that positively contribute to the prediction of the label are in blue.}

\paragraph{SHAP.}

\revisremi{The SHAP image-level explanations are generated using the partitioning approach defined in the \revisremi{\href{https://github.com/shap/shap}{official author's repository}}. All of our samples are generated using the default parameters for image-level explanations, i.e., running a blur image masker to partition the input image with rectangular masks. The estimation of Shapley values is obtained by running 200 evaluations. Partitions that positively contribute to the prediction of the label are in red and superpixels that negatively contribute to the prediction of the label are in blue. }

\paragraph{\CBMLIME.}

\revisremi{The \CBMLIME explanations are generated using the approach used for tabular data as defined in \citep{lime} using the \revisremi{\href{https://github.com/marcotcr/lime}{official author's repository}}. All of our samples are generated using the default parameters for tabular-level explanations, i.e., an exponential kernel of width  0.75$\times N$, ($N$ being the number of concepts).}

\paragraph{\CBMshap.}

\revisremi{The \CBMshap explanations are generated using the approach used for tabular data as in the \revisremi{\href{https://github.com/shap/shap}{official author's repository}}. All of our samples are generated using the default parameters for explanations, i.e., DeepLIFT algorithm using the training data as a background dataset.}

\paragraph{LaBo.}

\updateremi{For LaBo, we train a LaBo classifier by using Adam optimizer with a learning rate of $0.5$, a weight decay of $0$, and a batch size of $8192$. The resulting explanations follow the resulting weight matrix.}

\paragraph{Yan et al.}
 
\updateremi{For the Yan et al. method, we train a linear classifier by using Adam optimizer with a learning rate of $5\times 10^{-3}$, a weight decay of $10^{-4}$, and a batch size of $512$. The sample-wise (concept level) explanation results from the product between the concept score and its associated weight. The dataset-wise explanation results from the weights alone.}

\paragraph{Resnet.}

\revisremi{To train the Resnet classifier on PASCAL-Part, MIT scenes and MonuMAI, we initialized the network with ImageNet pertaining. Then, we trained the network using Adam optimizer with a learning rate of $10^{-3}$ for the probe, $10^{-4}$ for the backbone, a batch size of $64$, a momentum of $0.9$ and a weight decay of $10^{-4}$. For ImageNet results, we use results provided by the authors.}

\paragraph{ViT.}

\revisremi{To train the ViT classifier on PASCAL-Part, MIT scenes and MonuMAI, we initialized the network with ImageNet pertaining. Then, we trained the network using Adam optimizer with a learning rate of $10^{-2}$ for the probe, a frozen backbone, a batch size of $128$, a momentum of $0.9$, and no weight decay. For ImageNet results, we use results provided by the authors.}

\subsection{Cats/Dogs/Cars} \label{Cats_Dogs}

\remi{The Cats/Dogs/Cars dataset is available on \revisremi{\href{https://huggingface.co/datasets/RemiKaz/Cats_Dogs_Cars}{hugging face}}. To create this dataset, we combined two well-known datasets: the Kaggle Cats and Dogs Dataset \citep{dogs-vs-cats} and the Standford Cars Dataset \citep{KrauseStarkDengFei-Fei_3DRR2013}. We then filtered the dataset to include only images featuring black and white animals and cars, leading to six different categories. The primary objective behind assembling this dataset is to facilitate research in scenarios characterized by significant data bias, such as the classification of white cats when the training data predominantly consists of images of white dogs and black cats.}

\subsection{Set of concepts} \label{concepts_sets}

Inspired by \cite{yang2023language}, we use large language models to provide concept sets. Concretely, we use GPT-3 \citep{brown2020language} with the following preprompt: ``In this task, you have to give visual descriptions that describe an image. Respond as a list. Each item being a word.'' Then, we generate the sets of words by the following prompt: ``What are [N] useful visual descriptors to distinguish a [class] in a photo?''. By doing so, we generated 5 concepts per class, presented in Table \ref{tableconcepts}. \updateremi{An additional ordering of concepts by subcategories is available in Table \ref{tableconcepts_bis}.}

\newpage

\begin{table}[H]
\caption{List of concepts used\updateremi{, ordered by classes of interest.}}
\label{tableconcepts}
\begin{center}
\begin{tabular}{p{0.15\textwidth} p{0.82\textwidth}}
\multicolumn{1}{c}{\bf Dataset}  &\multicolumn{1}{c}{\bf Set of concepts} 
\\ \hline \\
\textit{PASCAL-Part}  & 'aeroplane':[Winged, Jet engines, Tail fin, Fuselage, Landing gear],'bicycle':[Two-wheeled, Pedals, Handlebars, Frame, Chain-driven],'bird':[Feathery, Beaked, Wingspread, Perched, Avian],'bottle':[Glass or plastic, Cylindrical, Necked, Cap or cork, Transparent or translucent],'bus':[Large, Rectangular, Windows, Wheels, Multi-doored],'cat':[Furry, Whiskered, Pointy-eared, Slitted-eyed, Four-legged],'car':[Metallic, Four-wheeled, Headlights, Windshield, License plate],'dog':[Snout, Wagging-tailed, Snout-nosed, Floppy-eared, Tail-wagging],'cow':[Bovine, Hooved, Horned (in some cases), Spotted or solid-colored, Grazing (if in a field)],'horse':[Equine, Hooved, Maned, Tailed, Galloping (if in motion)],'motorbike':[Two-wheeled , Engine, Handlebars , Exhaust, Saddle or seat],'person':[Human, Facial features, Limbs (arms and legs), Clothing, Posture or stance],'potted plant':[Potted, Green, Leaves, Soil or potting mix, Container or pot],'sheep':[Woolly, Hooved, Grazing, Herded (if in a group), White or colored fleece],'train':[Locomotive, Railcars, Tracks, Wheels , Carriages],'tvmonitor':[Screen, Rectangular , Frame or bezel, Stand or wall-mounted, Displaying an image] \\
\textit{\updateremi{MonuMAI}} & 'Baroque':[Ornate, Elaborate sculptures, Intricate details, Curved or asymmetrical design, Historical or aged appearance],'Gothic':[Pointed arches, Ribbed vaults, Flying buttresses, Stained glass windows, Tall spires or towers],'Hispanic muslim':[Mudejar style, Intricate geometric patterns, Horseshoe arches, Decorative tilework (azulejos), Islamic-inspired motifs],'Rennaissance':[Classical proportions, Symmetrical design, Columns and pilasters, Human statues and sculptures, Dome or dome-like structures] \\
\textit{MIT scenes}  & 'Store':[Building or structure, Signage or banners, Glass windows or doors, Displayed products or merchandise, People entering or exiting (if applicable)],'Home':[Residential, Roofed, Windows, Landscaping or yard, Front entrance or door],'Public space':[Open area, Crowds (if people are present), Benches or seating, Pathways or walkways, Architectural features (e.g., buildings, statues)],'Leisure':[Recreational, Play equipment (if applicable), Greenery or landscaping, Picnic tables or seating, Relaxing atmosphere],'Working space':[Office equipment (e.g., desks, computers), Task-oriented, Office chairs, Organized or structured, People working (if applicable)] \\
\textit{Cats/Dogs/Cars}  &  'Cat':[Furry, Whiskered, Pointy-eared, Slitted-eyed, Four-legged],'Car':[Metallic, Four-wheeled, Headlights, Windshield, License plate],'Dog':[Snout, Wagging-tailed, Snout-nosed, Floppy-eared, Tail-wagging] \\
\end{tabular}
\end{center}
\end{table}

\begin{table}[H]
\caption{List of concepts used\updateremi{, ordered by subcategories.}}
\label{tableconcepts_bis}
\begin{center}
\updateremi{
\begin{tabular}{p{0.15\textwidth} p{0.82\textwidth}}
\multicolumn{1}{c}{\bf Dataset}  &\multicolumn{1}{c}{\bf Set of concepts} 
\\ \hline \\
\textit{PASCAL-Part}  & 'Objects':[Screen, Rectangular, Frame or bezel, Stand or wall-mounted, Displaying an image, Glass or plastic, Cylindrical, Necked, Cap or cork, Transparent or translucent, Potted, Green, Leaves, Soil or potting mix, Container or pot],
'Transportation-related':[Two-wheeled, Engine, Handlebars, Exhaust, Saddle or seat, Four-wheeled, Headlights, Windshield, License plate, Wheels, Multi-doored, Locomotive, Railcars, Tracks, Carriages, Chain-driven],'Aircraft-related':[Winged, Jet engines, Tail fin, Fuselage, Landing gear],'Building/Structure-related':[Large, Rectangular, Windows]
'Human Characteristics':[Human, Facial features, Limbs (arms and legs), Clothing, Posture or stance],'Avian Characteristics':[Feathery, Beaked, Wingspread, Perched, Avian],'Animal Characteristics':[Furry, Whiskered, Pointy-eared, Slitted-eyed, Four-legged, Equine, Hooved, Maned, Tailed, Galloping (if in motion), Snout, Wagging-tailed, Snout-nosed, Floppy-eared, Tail-wagging, Woolly, Grazing, Herded (if in a group), White or colored fleece, Bovine, Horned (in some cases), Spotted or solid-colored]\\
\textit{\updateremi{MonuMAI}} & 'Architectural Styles and Elements':[Mudejar style, Intricate geometric patterns, Horseshoe arches, Decorative tilework (azulejos), Islamic-inspired motifs, Stained glass windows, Pointed arches, Ribbed vaults],'Artistic Details and Features':[Ornate, Elaborate sculptures, Intricate details, Curved or asymmetrical design, Human statues and sculptures, Historical or aged appearance],'Architectural Components':[Flying buttresses, Classical proportions, Symmetrical design, Columns and pilasters, Dome or dome-like structures, Tall spires or towers]\\
\textit{MIT scenes}  & 'Work Environment':[Office equipment (e.g., desks, computers), Task-oriented, Office chairs, Organized or structured, People working (if applicable)],
'Leisure Environment':[Play equipment (if applicable), Greenery or landscaping, Picnic tables or seating, Relaxing atmosphere, Recreational],'Community Environment':[Roofed, Windows, Landscaping or yard, Front entrance or door, Building or structure, Signage or banners, Glass windows or doors, Displayed products or merchandise, People entering or exiting (if applicable), Open area, Crowds (if people are present), Benches or seating, Pathways or walkways, Architectural features (e.g., buildings, statues), Residential]\\
\textit{Cats/Dogs/Cars}  & 'Vehicle Features':[License plate, Headlights, Windshield, Four-wheeled], 'Animal Features':[Wagging-tailed, Tail-wagging, Snout, Whiskered, Pointy-eared, Slitted-eyed, Floppy-eared, Furry, Snout-nosed, Four-legged],'Material/Texture':[Metallic]\\
\end{tabular}}
\end{center}
\end{table}

\subsection{\updateremi{Study of images according to their position in the latent space}} \label{centroid_dist}

\updateremi{In addition to the \methoddata method, we introduce another approach leveraging multivariate Gaussian modeling to generate dataset-wise visual explanations for our model. Specifically, for a given class $c$, we compute the Mahalanobis distance with respect to $\mathcal{N}(\vz \mid \vmu_c, \vSigma_c)$ for each sample's latent space projection~$\vz$ labeled as~$c$ within a dataset of interest:}
\updateremi{
\begin{equation}
d_M(\vz, \vmu_c, \vSigma_c) = \sqrt{(\vz-\vmu_c)^{\top} \vSigma_c^{-1}(\vz-\vmu_c)}.
\label{mahalaobis}
\end{equation}}
\updateremi{Then, we plot in Figures \ref{fig:centroids_exps_close_pascalpart}, \ref{fig:centroids_exps_far_pascalpart},  \ref{fig:centroids_exps_close_monumai} and \ref{fig:centroids_exps_far_monumai}, the closest and farthest images with the dataset, according to this distance for PascalPART and MonuMAI.}

\begin{figure}[H]
     \hfill
     \centering
     \begin{subfigure}[t]{0.24\textwidth}
         \centering
         \includegraphics[width=\textwidth]{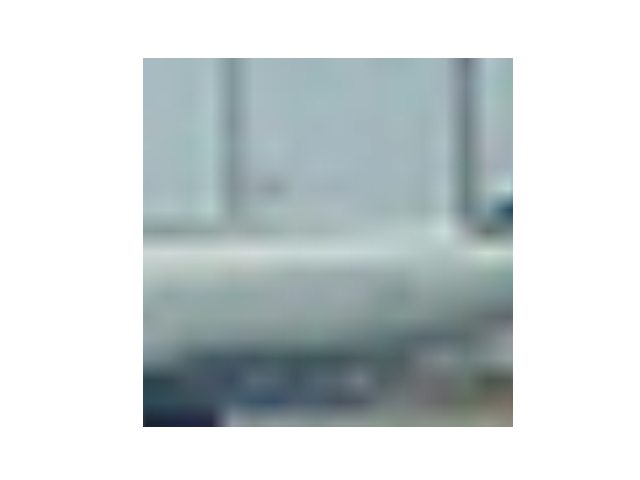}
         \caption{\updateremi{\textbf{``Aeroplane''}}}
     \end{subfigure}
     \begin{subfigure}[t]{0.24\textwidth}
         \centering
         \includegraphics[width=\textwidth]{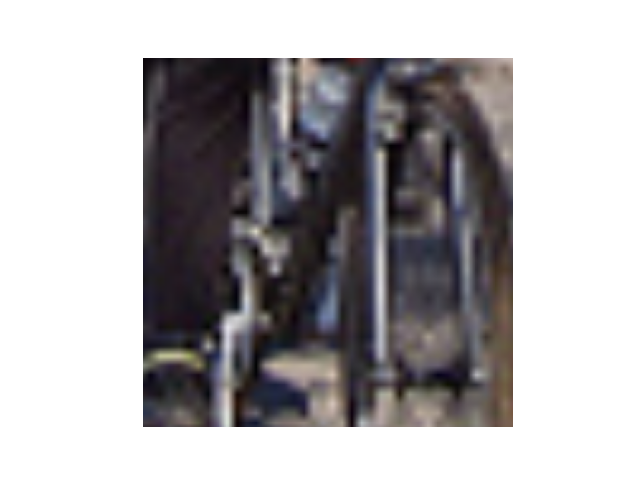}
         \caption{\updateremi{\textbf{`Bicycle''}}}
     \end{subfigure}
     \begin{subfigure}[t]{0.24\textwidth}
         \centering
         \includegraphics[width=\textwidth]{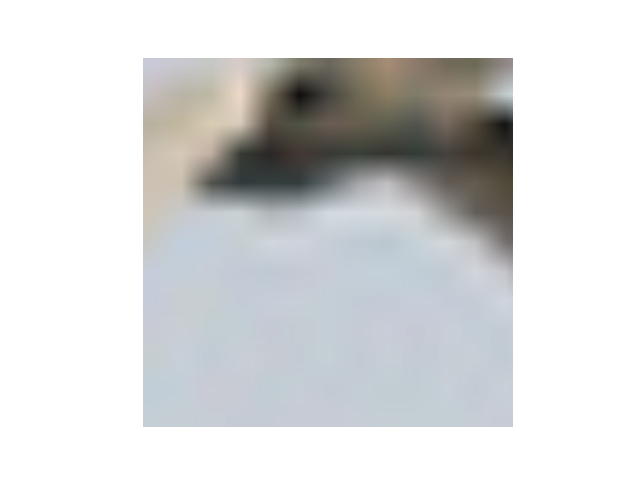}
         \caption{\updateremi{\textbf{``Bird''}}}
     \end{subfigure}
     \begin{subfigure}[t]{0.24\textwidth}
         \centering
         \includegraphics[width=\textwidth]{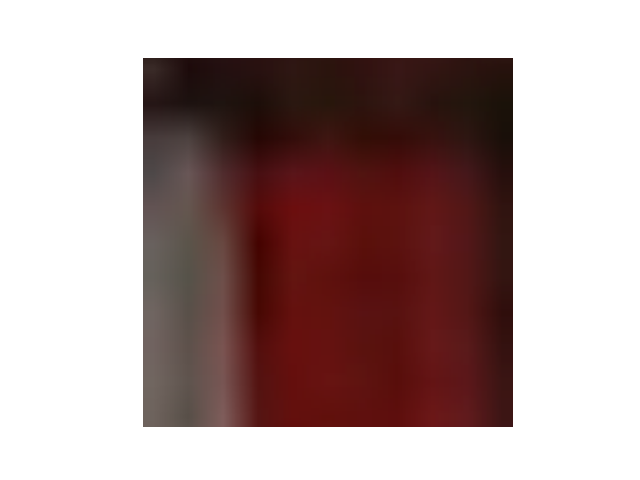}
         \caption{\updateremi{\textbf{``Bottle''}}}
     \end{subfigure}
     \begin{subfigure}[t]{0.24\textwidth}
         \centering
         \includegraphics[width=\textwidth]{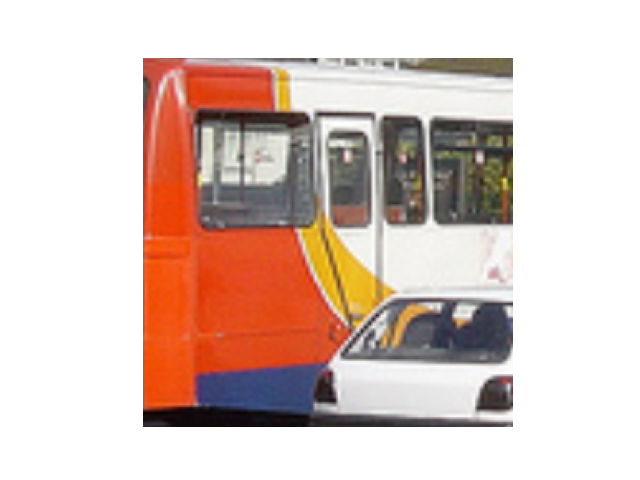}
         \caption{\updateremi{\textbf{``Bus''}}}
     \end{subfigure}
     \begin{subfigure}[t]{0.24\textwidth}
         \centering
         \includegraphics[width=\textwidth]{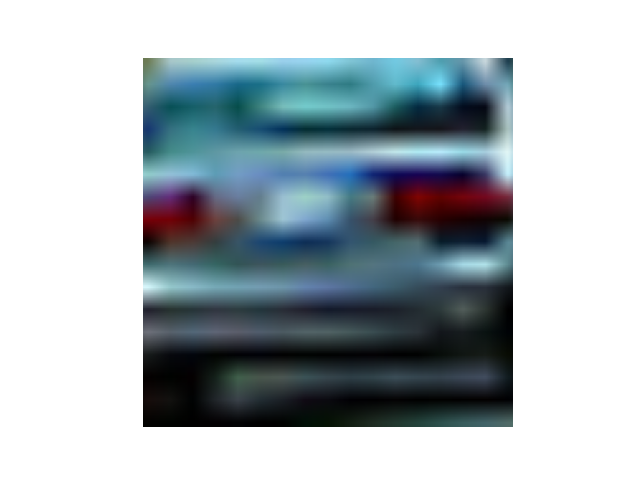}
         \caption{\updateremi{\textbf{``Car''}}}
     \end{subfigure}
     \begin{subfigure}[t]{0.24\textwidth}
         \centering
         \includegraphics[width=\textwidth]{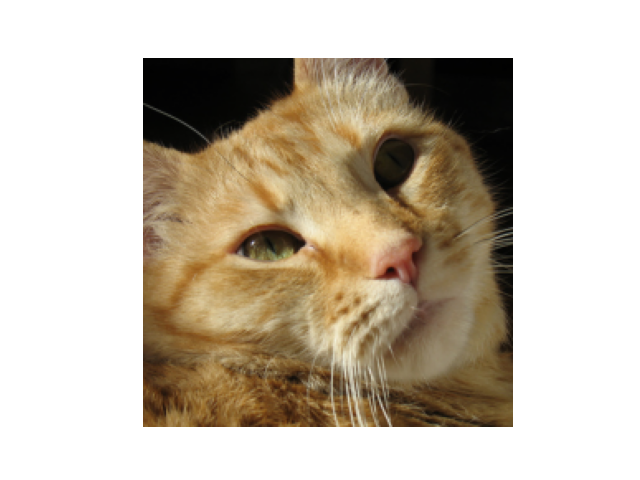}
         \caption{\updateremi{\textbf{``Cat''}}}
     \end{subfigure}
     \begin{subfigure}[t]{0.24\textwidth}
         \centering
         \includegraphics[width=\textwidth]{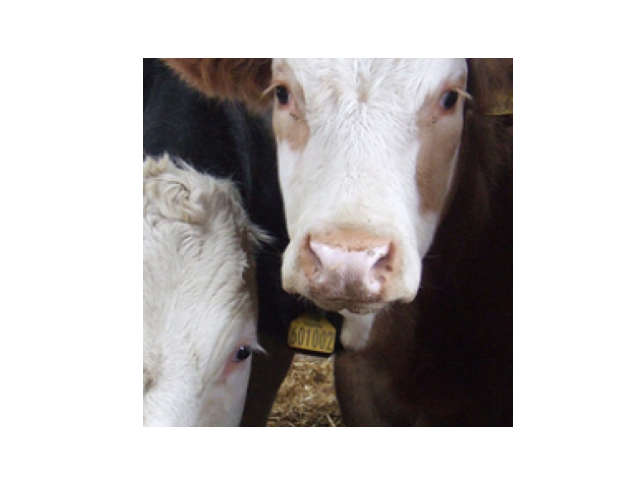}
         \caption{\updateremi{\textbf{``Cow''}}}
     \end{subfigure}
     \begin{subfigure}[t]{0.24\textwidth}
         \centering
         \includegraphics[width=\textwidth]{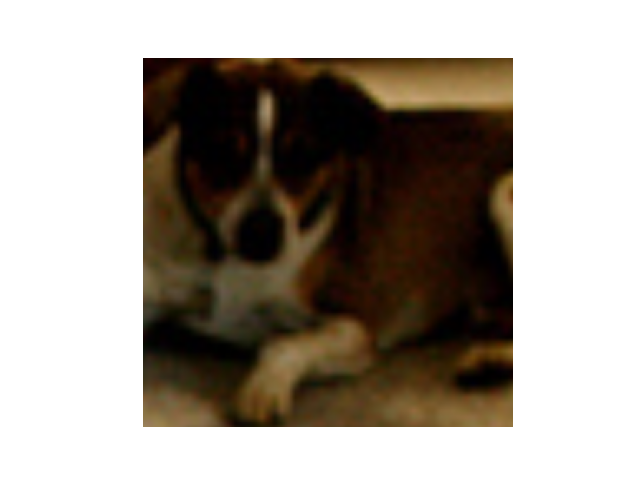}
         \caption{\updateremi{\textbf{``Dog''}}}
     \end{subfigure}
     \begin{subfigure}[t]{0.24\textwidth}
         \centering
         \includegraphics[width=\textwidth]{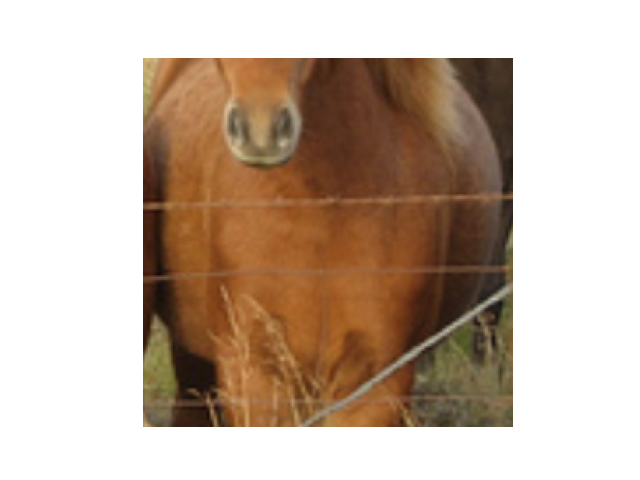}
         \caption{\updateremi{\textbf{``Horse''}}}
     \end{subfigure}
     \begin{subfigure}[t]{0.24\textwidth}
         \centering
         \includegraphics[width=\textwidth]{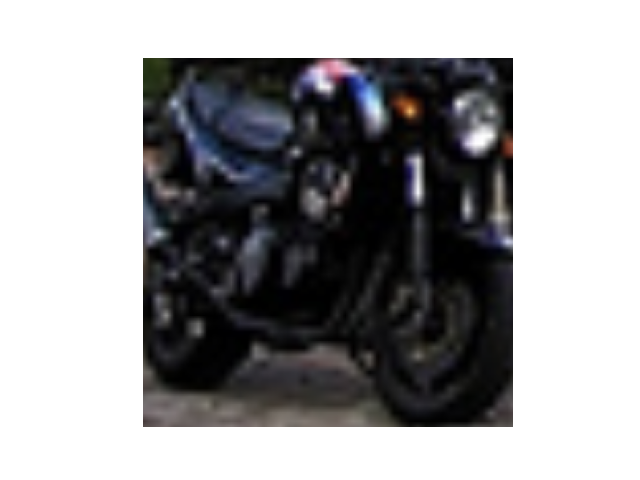}
         \caption{\updateremi{\textbf{``Motorbike''}}}
     \end{subfigure}
     \begin{subfigure}[t]{0.24\textwidth}
         \centering
         \includegraphics[width=\textwidth]{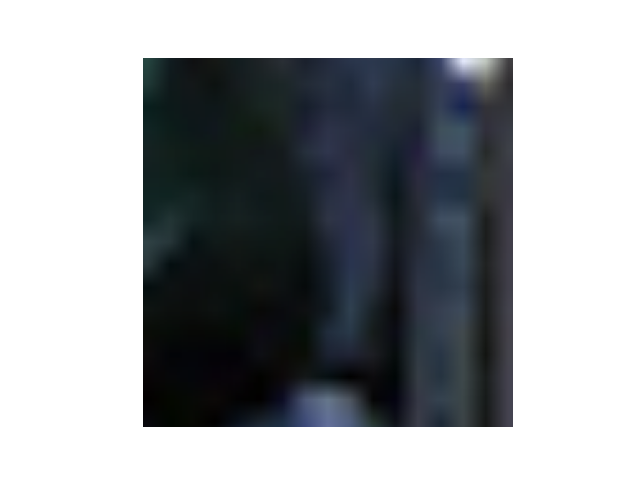}
         \caption{\updateremi{\textbf{``Person''}}}
     \end{subfigure}
     \begin{subfigure}[t]{0.24\textwidth}
         \centering
         \includegraphics[width=\textwidth]{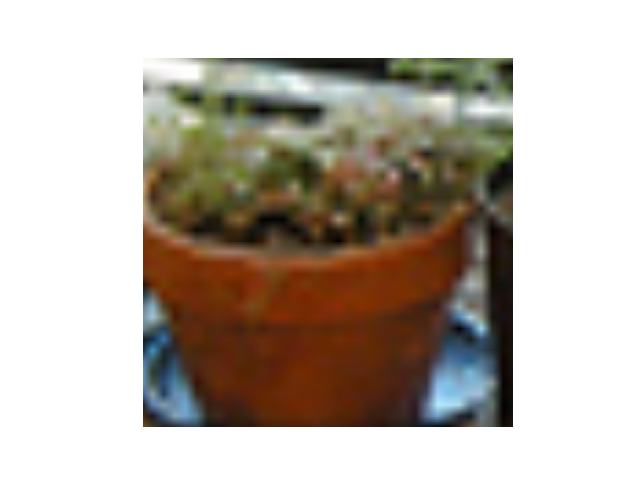}
         \caption{\updateremi{\textbf{``Potted plant''}}}
     \end{subfigure}
     \begin{subfigure}[t]{0.24\textwidth}
         \centering
         \includegraphics[width=\textwidth]{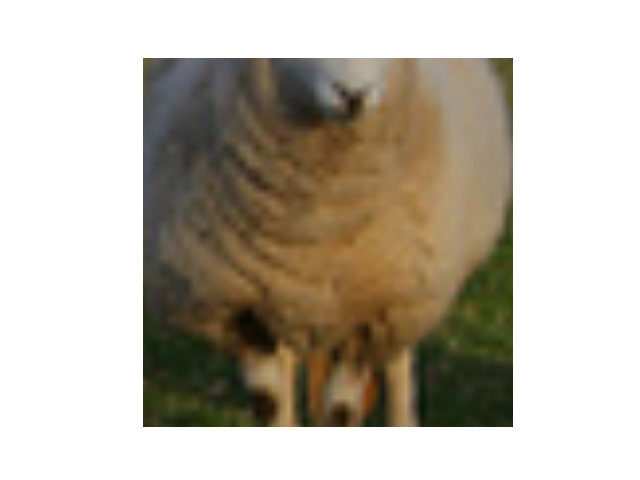}
         \caption{\updateremi{\textbf{``Sheep''}}}
     \end{subfigure}
     \begin{subfigure}[t]{0.24\textwidth}
         \centering
         \includegraphics[width=\textwidth]{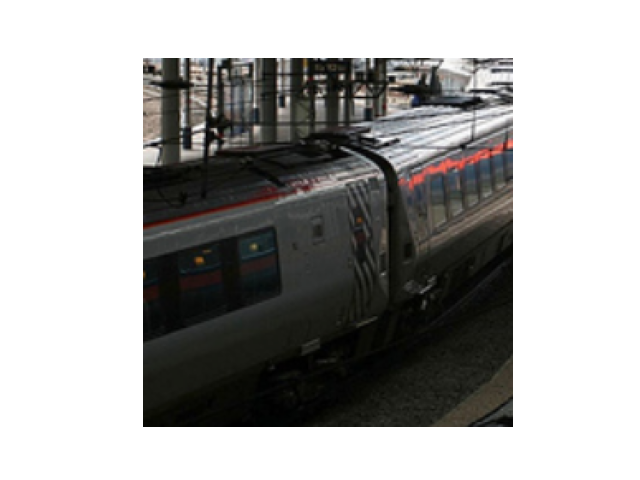}
         \caption{\updateremi{\textbf{``Train''}}}
     \end{subfigure}
     \begin{subfigure}[t]{0.24\textwidth}
         \centering
         \includegraphics[width=\textwidth]{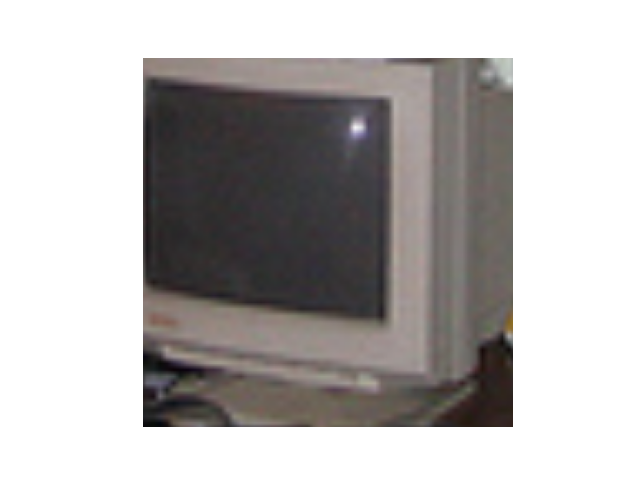}
         \caption{\updateremi{\textbf{``TV monitor''}}}
     \end{subfigure}
     \hfill
    \caption{\updateremi{\textbf{Images of minimal distance for each class of PascalPART.}} \revisremi{Each subfigure in the plot corresponds to the image of class $c$ from the test set that exhibits the minimal Mahalanobis distance, as defined in Equation \ref{mahalaobis}, with respect to the Gaussian representation of the given class $c$.}}
     \label{fig:centroids_exps_close_pascalpart}
     \hfill
\end{figure}

\begin{figure}[H]
     \hfill
     \centering
     \begin{subfigure}[t]{0.24\textwidth}
         \centering
         \includegraphics[width=\textwidth]{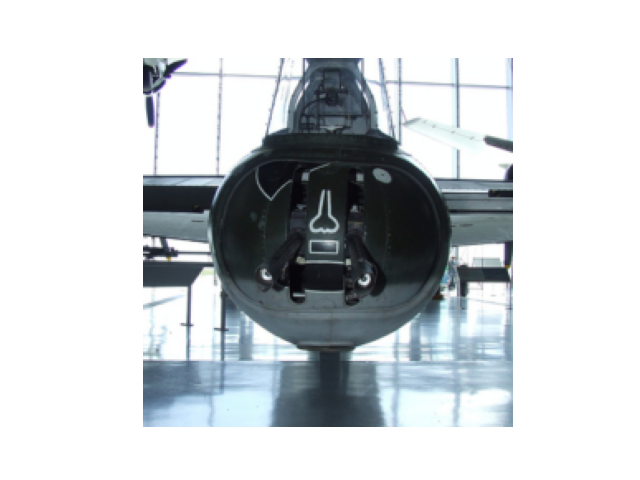}
         \caption{\updateremi{\textbf{``Aeroplane''}}}
     \end{subfigure}
     \begin{subfigure}[t]{0.24\textwidth}
         \centering
         \includegraphics[width=\textwidth]{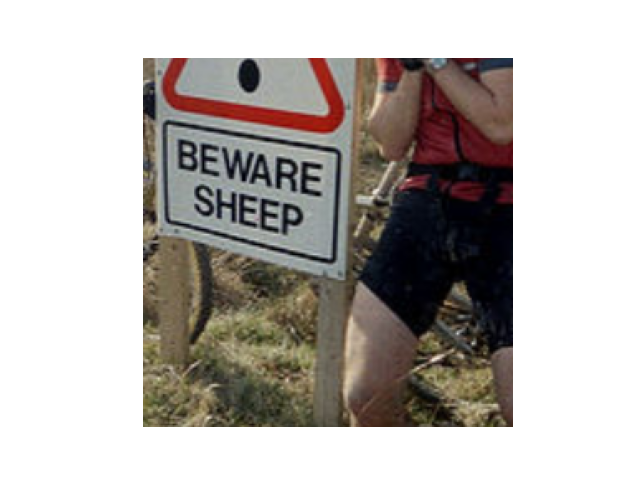}
         \caption{\updateremi{\textbf{`Bicycle''}}}
     \end{subfigure}
     \begin{subfigure}[t]{0.24\textwidth}
         \centering
         \includegraphics[width=\textwidth]{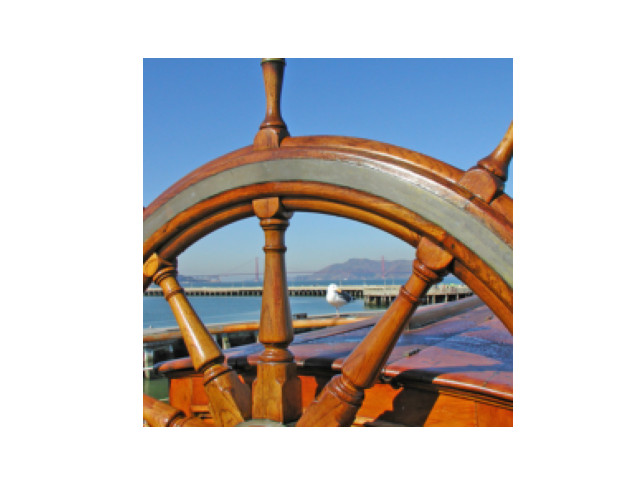}
         \caption{\updateremi{\textbf{``Bird''}}}
     \end{subfigure}
     \begin{subfigure}[t]{0.24\textwidth}
         \centering
         \includegraphics[width=\textwidth]{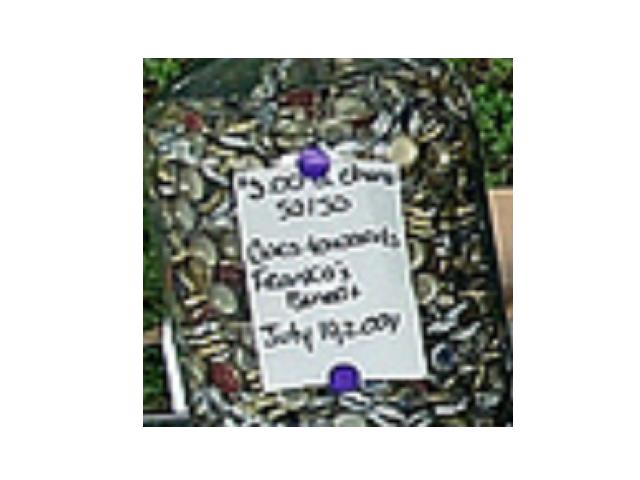}
         \caption{\updateremi{\textbf{``Bottle''}}}
     \end{subfigure}
     \begin{subfigure}[t]{0.24\textwidth}
         \centering
         \includegraphics[width=\textwidth]{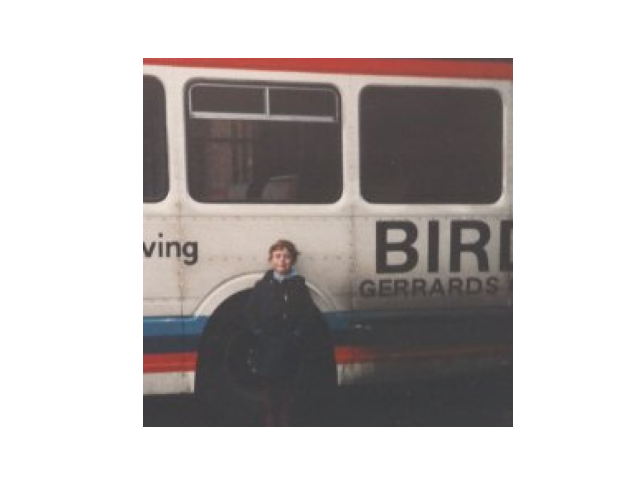}
         \caption{\updateremi{\textbf{``Bus''}}}
     \end{subfigure}
     \begin{subfigure}[t]{0.24\textwidth}
         \centering
         \includegraphics[width=\textwidth]{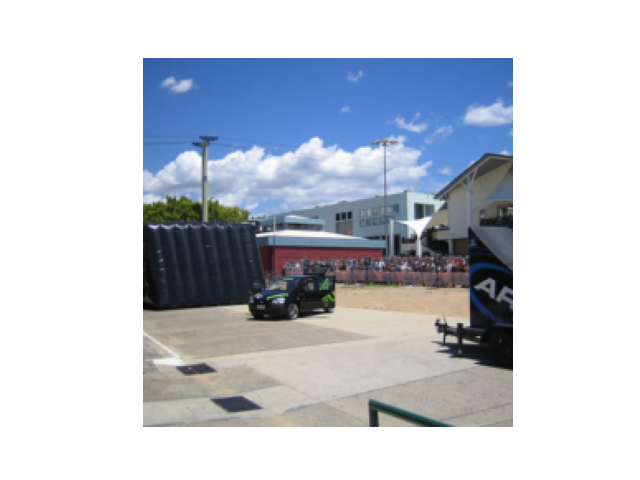}
         \caption{\updateremi{\textbf{``Car''}}}
     \end{subfigure}
     \begin{subfigure}[t]{0.24\textwidth}
         \centering
         \includegraphics[width=\textwidth]{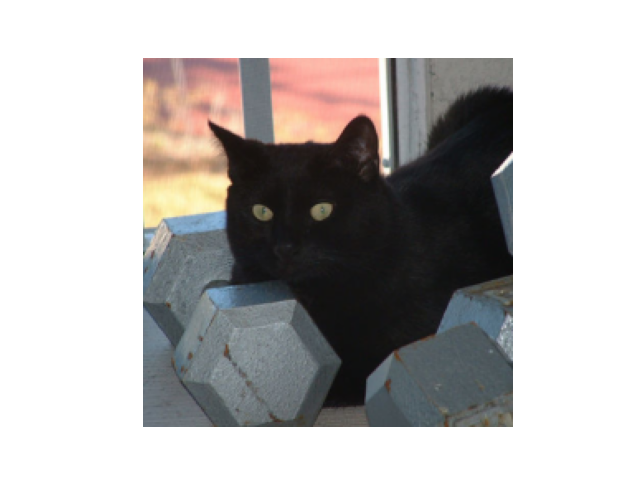}
         \caption{\updateremi{\textbf{``Cat''}}}
     \end{subfigure}
     \begin{subfigure}[t]{0.24\textwidth}
         \centering
         \includegraphics[width=\textwidth]{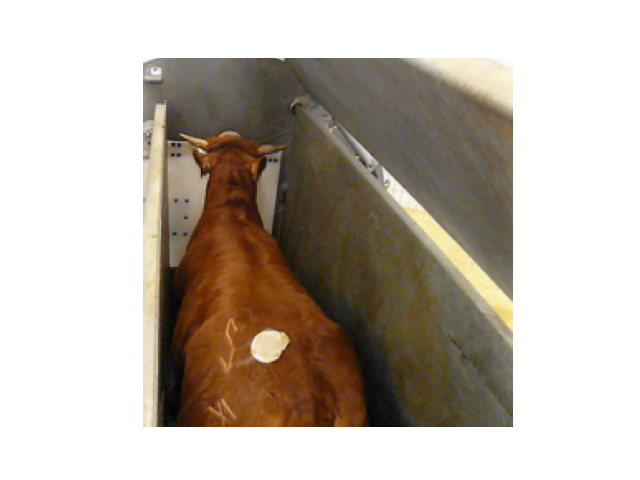}
         \caption{\updateremi{\textbf{``Cow''}}}
     \end{subfigure}
     \begin{subfigure}[t]{0.24\textwidth}
         \centering
         \includegraphics[width=\textwidth]{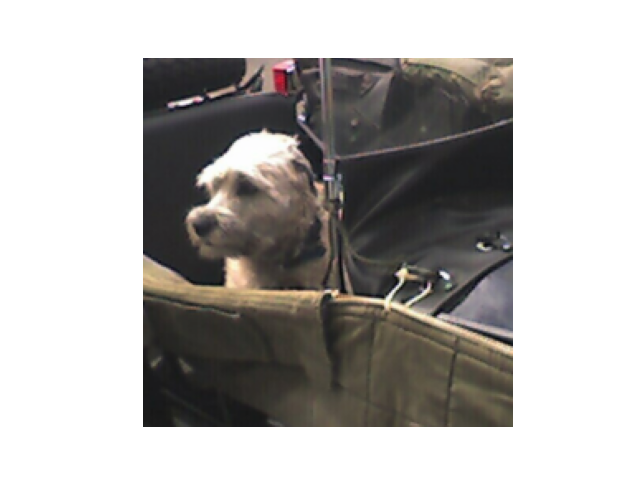}
         \caption{\updateremi{\textbf{``Dog''}}}
     \end{subfigure}
     \begin{subfigure}[t]{0.24\textwidth}
         \centering
         \includegraphics[width=\textwidth]{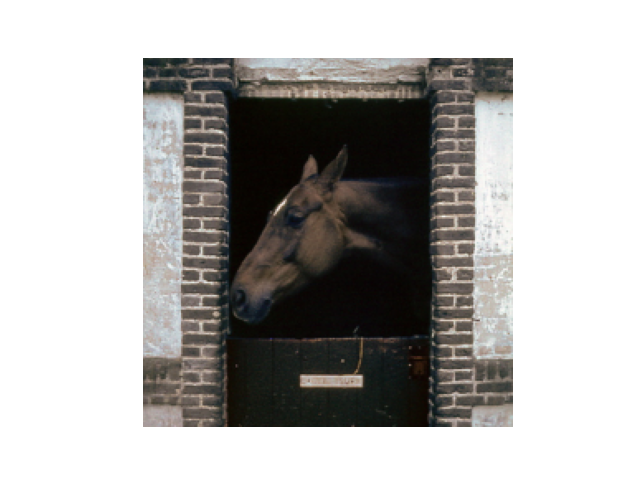}
         \caption{\updateremi{\textbf{``Horse''}}}
     \end{subfigure}
     \begin{subfigure}[t]{0.24\textwidth}
         \centering
         \includegraphics[width=\textwidth]{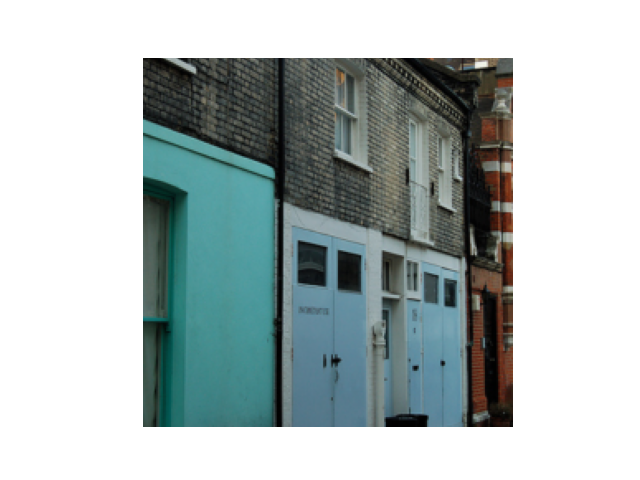}
         \caption{\updateremi{\textbf{``Motorbike''}}}
     \end{subfigure}
     \begin{subfigure}[t]{0.24\textwidth}
         \centering
         \includegraphics[width=\textwidth]{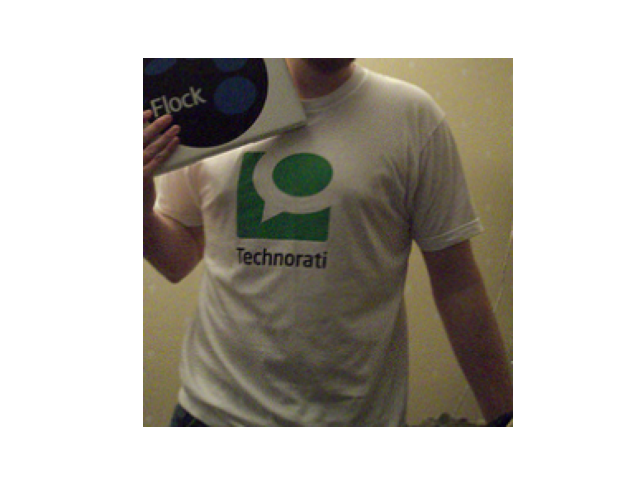}
         \caption{\updateremi{\textbf{``Person''}}}
     \end{subfigure}
     \begin{subfigure}[t]{0.24\textwidth}
         \centering
         \includegraphics[width=\textwidth]{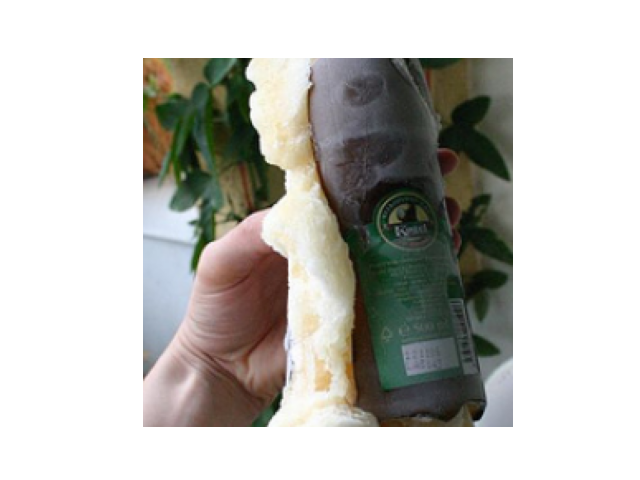}
         \caption{\updateremi{\textbf{``Potted plant''}}}
     \end{subfigure}
     \begin{subfigure}[t]{0.24\textwidth}
         \centering
         \includegraphics[width=\textwidth]{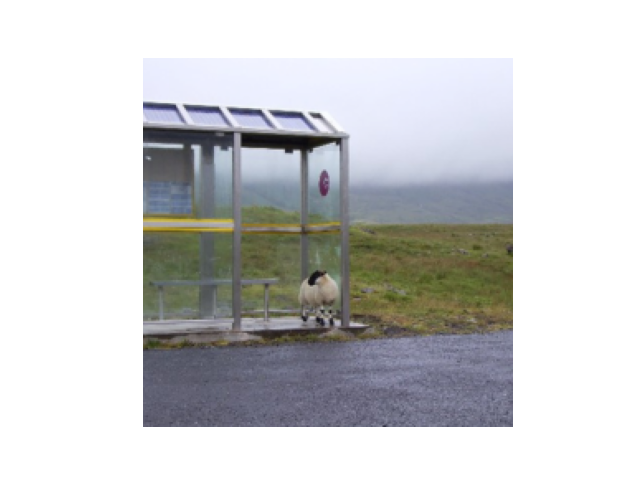}
         \caption{\updateremi{\textbf{``Sheep''}}}
     \end{subfigure}
     \begin{subfigure}[t]{0.24\textwidth}
         \centering
         \includegraphics[width=\textwidth]{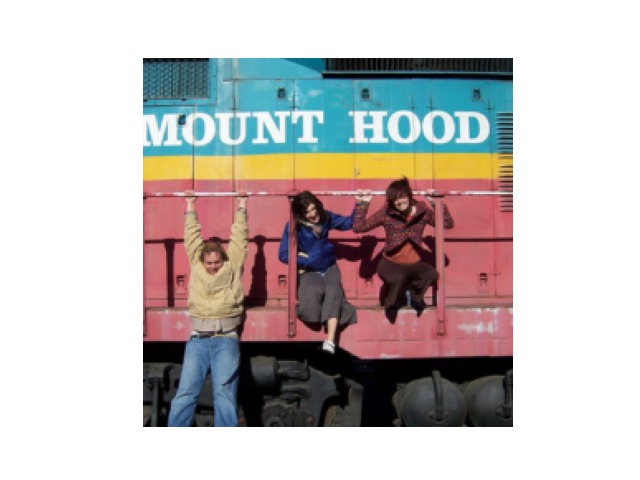}
         \caption{\updateremi{\textbf{``Train''}}}
     \end{subfigure}
     \begin{subfigure}[t]{0.24\textwidth}
         \centering
         \includegraphics[width=\textwidth]{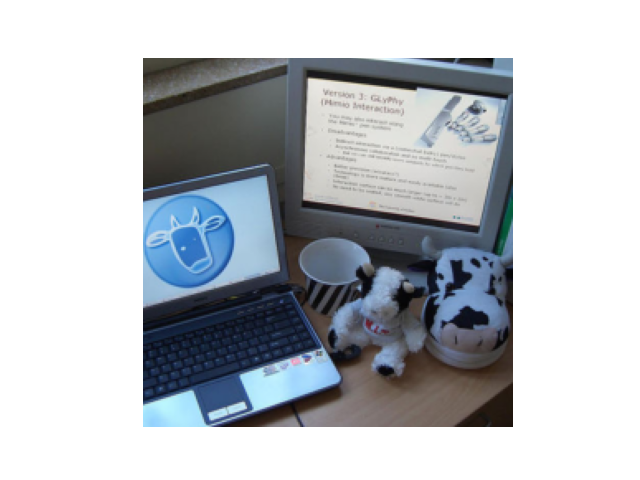}
         \caption{\updateremi{\textbf{``TV monitor''}}}
     \end{subfigure}
     \hfill
    \caption{\updateremi{\textbf{Images of maximal distance for each class of PascalPART.} \revisremi{Each subfigure in the plot corresponds to the image of class $c$ from the test set that exhibits the maximal Mahalanobis distance, as defined in Equation \ref{mahalaobis}, with respect to the Gaussian representation of the given class $c$.}}}
     \label{fig:centroids_exps_far_pascalpart}
     \hfill
\end{figure}

\begin{figure}[H]
     \hfill
     \centering
     \begin{subfigure}[t]{0.24\textwidth}
         \centering
         \includegraphics[width=\textwidth]{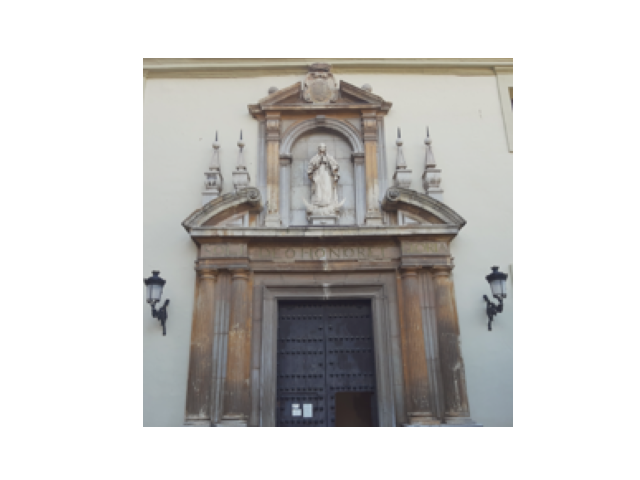}
         \caption{\updateremi{\textbf{``Baroque''}}}
     \end{subfigure}
     \begin{subfigure}[t]{0.24\textwidth}
         \centering
         \includegraphics[width=\textwidth]{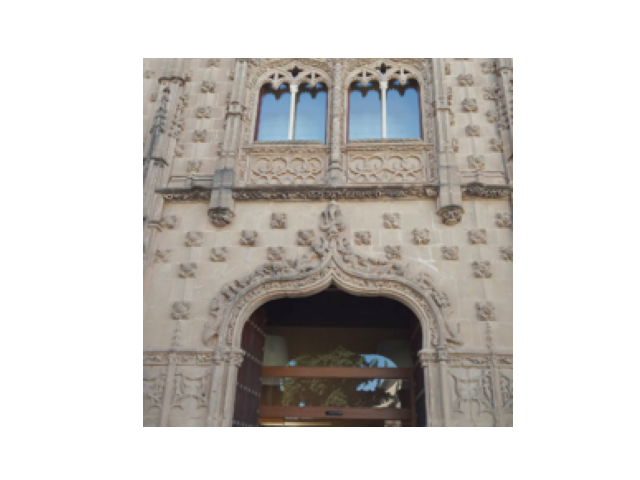}
         \caption{\updateremi{\textbf{``Gothic''}}}
     \end{subfigure}
     \begin{subfigure}[t]{0.24\textwidth}
         \centering
         \includegraphics[width=\textwidth]{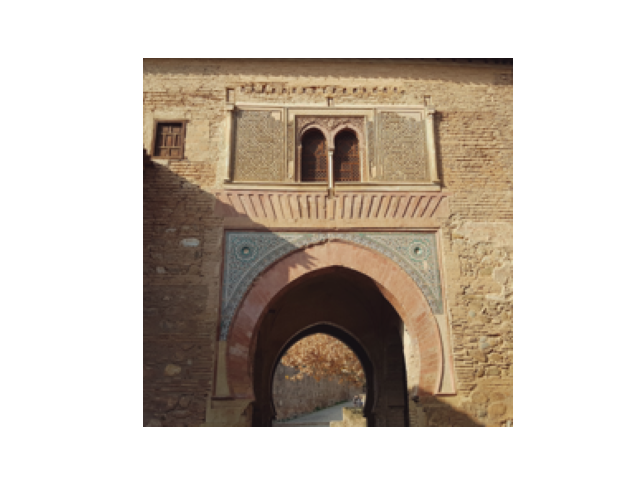}
         \caption{\updateremi{\textbf{``Hispanic-Muslim''}}}
     \end{subfigure}
     \begin{subfigure}[t]{0.24\textwidth}
         \centering
         \includegraphics[width=\textwidth]{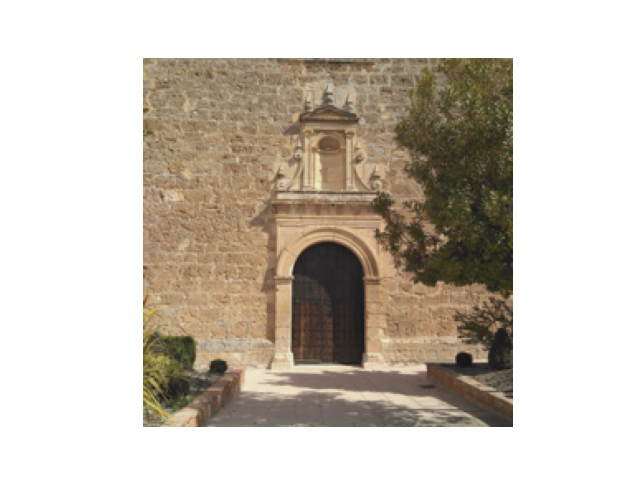}
         \caption{\updateremi{\textbf{``Renaissance''}}}
     \end{subfigure}
     \hfill
    \caption{\updateremi{\textbf{Images of minimal distance for each class of MonuMAI.} \revisremi{Each subfigure in the plot corresponds to the image of class $c$ from the test set that exhibits the minimal Mahalanobis distance, as defined in Equation \ref{mahalaobis}, with respect to the Gaussian representation of the given class $c$.}}}
     \label{fig:centroids_exps_close_monumai}
     \hfill
\end{figure}

\begin{figure}[H]
     \hfill
     \centering
     \begin{subfigure}[t]{0.24\textwidth}
         \centering
         \includegraphics[width=\textwidth]{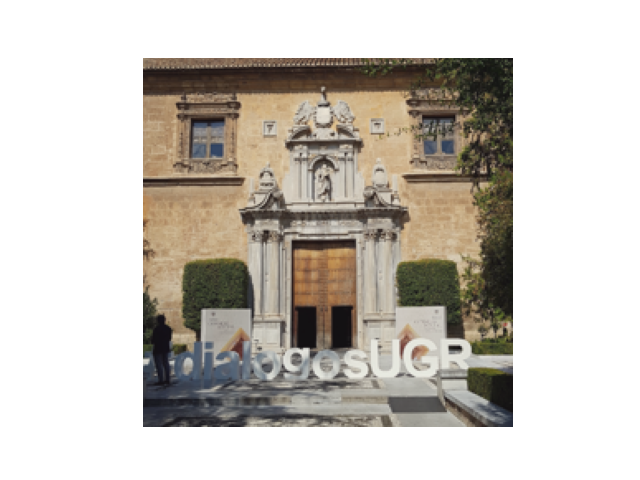}
         \caption{\updateremi{\textbf{``Baroque''}}}
     \end{subfigure}
     \begin{subfigure}[t]{0.24\textwidth}
         \centering
         \includegraphics[width=\textwidth]{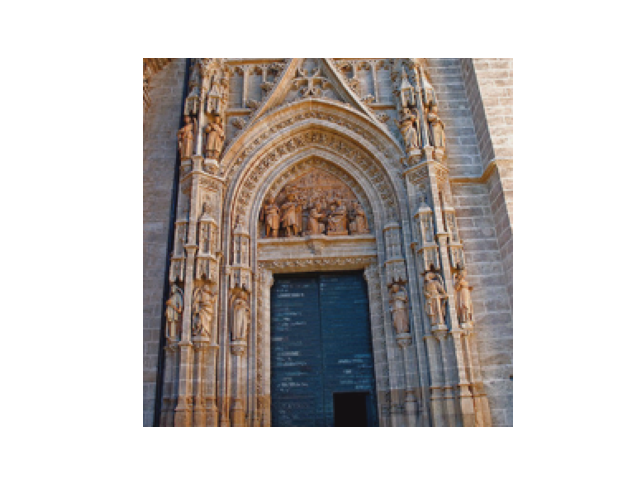}
         \caption{\updateremi{\textbf{`Gothic''}}}
     \end{subfigure}
     \begin{subfigure}[t]{0.24\textwidth}
         \centering
         \includegraphics[width=\textwidth]{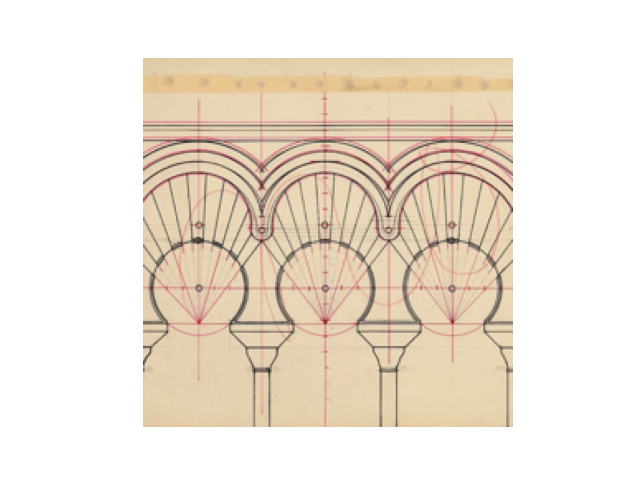}
         \caption{\updateremi{\textbf{``Hispanic-Muslim''}}}
     \end{subfigure}
     \begin{subfigure}[t]{0.24\textwidth}
         \centering
         \includegraphics[width=\textwidth]{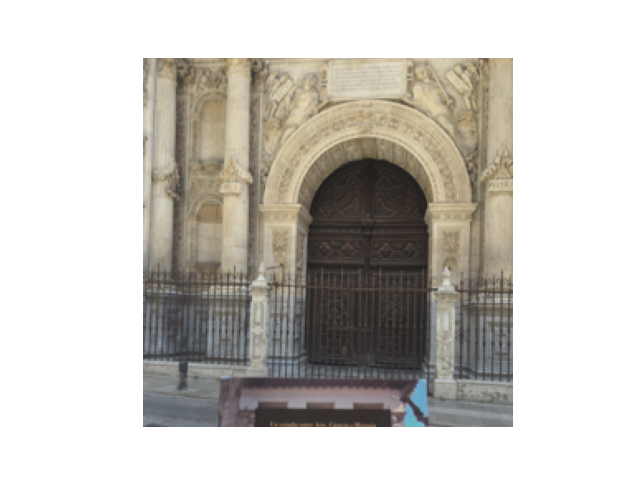}         \caption{\updateremi{\textbf{``Renaissance''}}}
     \end{subfigure}
     \hfill
    \caption{\updateremi{\textbf{Images of maximal distance for each class of MonuMAI.} \revisremi{Each subfigure in the plot corresponds to the image of class $c$ from the test set that exhibits the minimal Mahalanobis distance, as defined in Equation \ref{mahalaobis}, with respect to the Gaussian representation of the given class $c$.}}}
     \label{fig:centroids_exps_far_monumai}
     \hfill
\end{figure}

\updateremi{For PascalPART, the closest images tend to be predominantly close-up shots of the object in question. However, this occasionally presents challenges due to the inherent pixelation of the images, leading to somewhat counterintuitive results. Conversely, the farthest images are typically those where the object is distant. Another scenario of failure arises when there are two classes present within the image, which complicates the retrieval process. On the other hand, for MonumAI, the closest images appear to be the simplest ones, without sign of disturbance. In contrast, the farthest images may include extraneous objects, such as signs.}

\subsection{Closed-form solution of counterfactuals for CLIP \remi{QDA}} \label{close_form_solution}

First, we will derive the resolution of \eqref{equ:Causal_Concept_Order} for the binary case before extending it to the multiclass case.
Let us begin with a binary classifier, where ${Y \in \{ h_{\vomega^h}(\vz), \overline{h_{\vomega^h}(\vz)} \}} $ (for convenience, we note $h_{\vomega^h}(\vz) = c_{h} $ and $\overline{h_{\vomega^h}(\vz)} = c_{\overline{h}} $ ). 

\begin{proposition} \label{demo_sample_wise_binary}
Let $h_{\vomega^h}$ a binary classifier %
  and $\vZ$ following a mixture of Gaussians such as $\vZ_{Y=c_{h}} \sim \mathcal{N}(\vmu_{{c_{h}}},\,\vSigma_{{{{c_{h}}}}})$ and $\vZ_{Y=c_{\overline{h}}} \sim \mathcal{N}(\vmu_{{c_{\overline{h}}}},\,\vSigma_{c_{\overline{h}}}^{-1})$, and $\vepsilon^{j}_{s}$ a perturbation with the above sparsity and sign restrictions. Then, there is a closed form solution to  problem \ref{equ:Causal_Concept_Order}, given by:
\begin{equation} \epsilon_s^j = 
\begin{cases}
\emptyset ~~ \textnormal{if} ~~ b^2-4 P c < 0  ~ \\
\textnormal{or} ~ ( s \neq \textnormal{sign}(b_1) ~ \textnormal{and} ~ s \neq \textnormal{sign}(b_2) ) \\
\\
b_1 ~~ \textnormal{if} ~~ b^2-4 P c > 0 ~~ \textnormal{and} ~~ \textnormal{sign}(b_1) = s \\
\textnormal{and} ~~ (\textnormal{sign}(b_2) \neq s ~~ \textnormal{or} ~~ |b_2| \geq |b_1|) \\
\\
b_2 ~~ \textnormal{if} ~~ b^2-4 P c > 0 ~~ \textnormal{and} ~~ \textnormal{sign}(b_2) = s \\
\textnormal{and} ~~ (\textnormal{sign}(b_1) \neq s ~~ \textnormal{or} ~~ |b_1| > |b_2|), 
\end{cases}
\label{equ:lagrangien_sol}
\end{equation}
with:
\begin{align}
& P=\frac{1}{2} \left([\vSigma_{c_{\overline{h}}}^{-1}]_{(j,j)}-[\vSigma_{c_{h}}^{-1}]_{(j,j)} \right) \nonumber \\
& b=\sum_{k=1}^N \left( ([\vz]_{(k)}-[\vmu_{{c_{\overline{h}}}}]_{(k)}) [\vSigma_{c_{\overline{h}}}^{-1}]_{(j,k)} \nonumber - ([\vz]_{(k)}-[\vmu_{{c_{h}}}]_{(k)}) [\vSigma_{c_{h}}^{-1}]_{(j,k)} \right) \nonumber \\
& c=\frac{1}{2} \log \left[\frac{|\vSigma_{c_{\overline{h}}}|}{|\vSigma_{c_{h}}|} \right] + \log \left[\frac{p_{c_{h}}}{p_{c_{\overline{h}}}} \right] + \frac{1}{2}(\vz-\vmu_{c_{\overline{h}}})^{\top} \vSigma_{c_{\overline{h}}}^{-1}(\vz-\vmu_{c_{\overline{h}}}) \nonumber - \frac{1}{2}(\vz-\vmu_{c_{h}})^{\top} {{\vSigma_{c_{h}}^{-1}}}(\vz-\vmu_{c_{h}}) \nonumber \\
& b_1 = \frac{-b - \sqrt{b^2-4 P c}}{2 P} \\ 
& b_2 = \frac{-b + \sqrt{b^2-4 P c}}{2 P} \, . \nonumber 
\end{align} 
\end{proposition}

\begin{proof}

\remi{For the binary case, \eqref{equ:Causal_Concept_Order} can be written as :}
\remi{\begin{align}
\min \|\vepsilon^{j}_{\remi{s}}\|^2 ~~ \textnormal{s.t.}  \quad  \frac{p_{c_{\overline{h}}}}{(2\pi)^{N/2}|\vSigma_{c_{\overline{h}}}|^{\frac{1}{2}}} e^{-\frac{1}{2}(\vz+\vepsilon^{j}_{\remi{s}}-\vmu_{c_{\overline{h}}})^T \vSigma_{c_{\overline{h}}}^{-1}(\vz+\vepsilon^{j}_{\remi{s}}-\vmu_{c_{\overline{h}}})} \leqslant  \frac{p_{c_{h}}}{(2\pi)^{N/2}|\vSigma_{c_h}|^{\frac{1}{2}}} e^{-\frac{1}{2}(\vz+\vepsilon^{j}_{\remi{s}}-\vmu_{{c_{h}}})^T\vSigma_{c_h}^{-1}(\vz+\vepsilon^{j}_{\remi{s}}-\vmu_{{c_{h}}})}  .
\label{equ:Causal_Concept_Order_classif}
\end{align}}

Let us focus on the inequality constraint:
\begin{multline*}
\hfill \frac{p_{c_{\overline{h}}}}{|\vSigma_{c_{\overline{h}}}|^{\frac{1}{2}}} ~~ e^{-\frac{1}{2}(\vz+\vepsilon^{j}_{s}-\vmu_{{c_{\overline{h}}}})^T\vSigma_{c_{\overline{h}}}^{-1}(\vz+\vepsilon^{j}_{s}-\vmu_{{c_{\overline{h}}}})} \leqslant ~~ \frac{p_{c_{h}}}{|\vSigma_{c_{h}}|^{\frac{1}{2}}} ~~ e^{-\frac{1}{2}(\vz+\vepsilon^{j}_{s}-\vmu_{{c_{h}}})^T \vSigma_{c_{h}}^{-1}(\vz+\vepsilon^{j}_{s}-\vmu_{{c_{h}}})} \hfill \\
\iff \log (p_{c_{\overline{h}}}) - \frac{1}{2} \log |\vSigma_{c_{\overline{h}}}| - \frac{1}{2} (\vz+\vepsilon^{j}_{s}-\vmu_{c_{\overline{h}}})^T\vSigma_{c_{\overline{h}}}^{-1}(\vz+\vepsilon^{j}_{s}-\vmu_{c_{\overline{h}}}) \hfill \\  
\hfill \leqslant ~~ \log  (p_{c_{h}}) - \frac{1}{2} \log |\vSigma_{c_{h}}| - \frac{1}{2} (\vz+\vepsilon^{j}_{s}-\vmu_{{c_{h}}})^T \vSigma_{c_{h}}^{-1}(\vz+\vepsilon^{j}_{s}-\vmu_{{c_{h}}}) \\
\iff  \log (p_{c_{\overline{h}}}) - \frac{1}{2} \log |\vSigma_{c_{\overline{h}}}| - \frac{1}{2} (\vz-\vmu_{{c_{\overline{h}}}})^T \vSigma_{c_{\overline{h}}}^{-1}  (\vz-\vmu_{{c_{\overline{h}}}}) - \frac{1}{2}  {\vepsilon^j_s}^T \vSigma_{c_{\overline{h}}}^{-1} \vepsilon^j_s  - {\vepsilon^j_s}^T \vSigma_{c_{\overline{h}}}^{-1}(\vz-\vmu_{{c_{\overline{h}}}}) \hfill \\ 
\hfill \leqslant ~~ \log  (p_{c_{h}}) - \frac{1}{2} \log |\vSigma_{c_{h}}| - \frac{1}{2} (\vz-\vmu_{{c_{h}}})^T \vSigma_{c_{h}}^{-1} (\vz-\vmu_{{c_{h}}}) - \frac{1}{2}   {\vepsilon^j_s}^T \vSigma_{c_{h}}^{-1} \vepsilon^j_s -  {\vepsilon^j_s}^T \vSigma_{c_{h}}^{-1}(\vz-\vmu_{{c_{h}}}) \\
\iff  \log (p_{c_{\overline{h}}}) - \frac{1}{2} \log |\vSigma_{c_{\overline{h}}}| - \frac{1}{2} (\vz-\vmu_{{c_{\overline{h}}}})^T \vSigma_{c_{\overline{h}}}^{-1}  (\vz-\vmu_{{c_{\overline{h}}}}) - \frac{1}{2}[\vSigma_{c_{\overline{h}}}^{-1}]_{(j,j)} (\epsilon^j_s)^2 - \epsilon^j_s \sum_{k=1}^N ([\vz]_{(k)}-[\vmu_{{c_{\overline{h}}}}]_{(k)}) [\vSigma_{c_{\overline{h}}}^{-1}]_{(j,k)} \hfill \\
\hfill \leqslant ~~ \log  (p_{c_{h}}) -  \frac{1}{2} \log | \vSigma_{c_{h}} | -  \frac{1}{2} (\vz-\vmu_{{c_{h}}})^T \vSigma_{c_{h}}^{-1} (\vz-\vmu_{{c_{h}}}) -  \frac{1}{2} [\vSigma_{c_{h}}^{-1}]_{(j,j)} (\epsilon^j_s)^2 - \epsilon^j_s  \sum_{k=1}^N ([\vz]_{(k)}-[\vmu_{{c_{h}}}]_{(k)}) [\vSigma_{c_{h}}^{-1}]_{(j,k)} \\
 \iff  \hfill P(\epsilon^{j}_{s})^2 + b\epsilon^{j}_{s} + c \geqslant 0 \, . \hfill
\end{multline*}

Then we can rewrite the problem as :
\begin{align}
\min~ (\vepsilon^{j}_{s})^2 ~~ \textnormal{s.t.} ~~ P(\epsilon^{j}_{s})^2 + b\epsilon^{j}_{s} + c & \leqslant 0 . \label{equ:optim_binary_bis}
\end{align}

To solve this problem, we introduce slack variables  $\lambda$ and~$I$, and define a Lagrangian as:
\begin{align}
L(\vepsilon^{j}_{s},\lambda,I) = (\epsilon^{j}_{s})^2 + \lambda(P(\epsilon^{j}_{s})^2 + b\epsilon^{j}_{s} + c + I^2) \, .
\end{align}

Then, we can find the possible solutions by solving the system :
\begin{equation}
\begin{cases}
\frac{\partial L}{\partial \epsilon^{j}_{s}} = 0 \\
\frac{\partial L}{\partial \lambda} = 0 \\
\frac{\partial L}{\partial I } = 0,
\end{cases} \nonumber 
\end{equation}
which corresponds to:
\begin{equation}
\begin{cases}
2(\lambda P+1)\epsilon^{j}_{s} + \lambda b = 0 \\
P(\epsilon^{j}_{s})^2 + b\epsilon^{j}_{s} + c + I^2 = 0 \\
2 \lambda I = 0 .
\end{cases}
\label{equ:lagrangien_sys}
\end{equation}

The third equation of \ref{equ:lagrangien_sys} indicates whether the inequality condition is saturated or not. If it is not saturated ($\lambda = 0$), then the label $c_{\overline{h}}$ is already achieved for $\vz$, resulting in a solution of $\epsilon^{j}_{s,*}=0$. This being impossible by construction, we only focus on the case where $\lambda \neq 0$.

If $\lambda \neq 0$, the condition is saturated, the second equation leads  to:
\begin{equation}
P (\epsilon^{j}_{s})^2 + b \epsilon_s^{j} + c = 0,  \nonumber  
\end{equation}
whose solutions are:
\begin{align}
 \epsilon^{j}_{s} \in  \left\{ \frac{-b - \sqrt{b^2-4 P c}}{2 P},\frac{-b + \sqrt{b^2-4 P c}}{2 P} \right\}   ~~\textnormal{if} ~~ b^2-4 P c > 0   , \nonumber
\end{align}
 which we rewrite as:
\begin{align}
  \epsilon^{j}_{s} \in \{ b_1,b_2 \} ~~\textnormal{if} ~~ b^2-4 P c > 0 . \label{solution_border}
\end{align}

Considering \ref{solution_border}, the validity of the results is conditioned by the sign $s$ and the condition of the magnitude of $\epsilon^{j}_{s,*}$. Then, the final result of the problem \ref{equ:Causal_Concept_Order} is either $\emptyset$, $b_1$ or $b_2$ depending on the conditions:
\begin{equation}
\begin{cases}
\emptyset ~~ \textnormal{if} ~~ b^2-4 P c < 0  ~ \\
\textnormal{or} ~ ( s \neq \textnormal{sign}(b_1) ~ \textnormal{and} ~ s \neq \textnormal{sign}(b_2) ) \\
\\
b_1 ~~ \textnormal{if} ~~ b^2-4 P c > 0 ~~ \textnormal{and} ~~ \textnormal{sign}(b_1) = s \\
\textnormal{and} ~~ (\textnormal{sign}(b_2) \neq s ~~ \textnormal{or} ~~ |b_2| \geq |b_1|) \\
\\
b_2 ~~ \textnormal{if} ~~ b^2-4 P c > 0 ~~ \textnormal{and} ~~ \textnormal{sign}(b_2) = s \\
\textnormal{and} ~~ (\textnormal{sign}(b_1) \neq s ~~ \textnormal{or} ~~ |b_1| > |b_2|)  . \\
\end{cases}
\end{equation}
\end{proof}

To expand problem \ref{equ:Causal_Concept_Order} to multiclass classification $C>2$, we consider it as a succession of $C-1$ binary classifications between each $i' \neq c_{h}$ and $c_h$. Given a concept index $j$ and a sign $s$, if we denote the solutions of theses problems as the set $\{ \vepsilon^{j}_{+,*,1},...,\vepsilon^{j}_{s,*, c_{h}-1},\vepsilon^j_{s,*,c_{h}+1},...,\vepsilon^{j}_{s,*,C} \}$, the final solution $\vepsilon^{j}_{s,*}$ is if it exists, the value of minimal magnitude among this set.

Examples of explanations based on this metric are given in Sections \ref{examples} and \ref{examples_bis}.

\subsection{\textcolor{black}{\updateremi{Sample-wise}} explanation time} \label{exptime}

In Table \ref{tabletimeexp}, we display the amount of time taken to produce a sample-wise explanation for each image of the test set of PASCAL-Part and Cats/Dogs/Cars.

\begin{table}[H]
\caption{Time (in seconds) to produce explanations on the entire test set.}
\label{tabletimeexp}
\begin{center}
\begin{tabular}{llll}
\hline
\multicolumn{1}{c}{\bf Method}  &\multicolumn{1}{c}{\bf \updateremi{\methodsample} }  &\multicolumn{1}{c}{\bf \updateremi{CLIP-LIME} } &\multicolumn{1}{c}{\bf \updateremi{CLIP-SHAP} }
\\ \hline \\
\textit{Cats/Dogs/Cars} & 3.01 & 76.69 & 256.76 \\
\textit{PASCAL-Part} & 2341.12 & 2857.28 & 7207.58 \\
\textit{\updateremi{MonuMAI}} & 2.19 & 19.94 & 56.25 \\
\textit{MITscenes} & 33.41 & 509.33 & 1398.64 \\
\end{tabular}
\end{center}
\end{table}

In this table, we can observe that using the \remi{local} explanation is the fastest, especially for the \remi{cases} where the number of \remi{concepts} and classes are low. \remi{This is justified by the fact that our method consists of using Proposition \ref{demo_sample_wise_binary} for each concept, sign, and class other than the inference one. Hence the complexity of this computation  for each sample is $O(2(C-1)N)$, where LIME's complexity does not depend on $C$.}

\subsection{Additional samples} \label{examples_bis}

\remi{We display here additional samples \remi{of both \updateremi{sample-wise} and \updateremi{dataset-wise} explanations} from the PASCAL-Part and Cats/Dogs/Cars dataset.}

\begin{figure}[htb]
\begin{center}
\begin{tabular}{l l l l }
(a) Input image & (b) GradCAM Explanation & (c) LIME Explanation  & (d) SHAP explanation\\
\includegraphics[width=0.24\columnwidth]{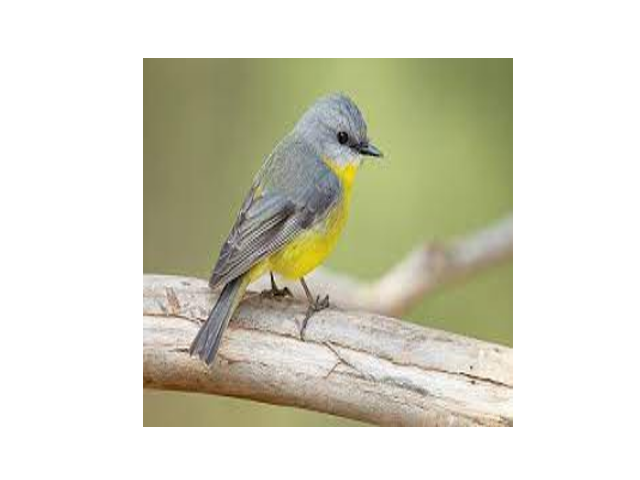}&
\includegraphics[width=0.24\columnwidth]{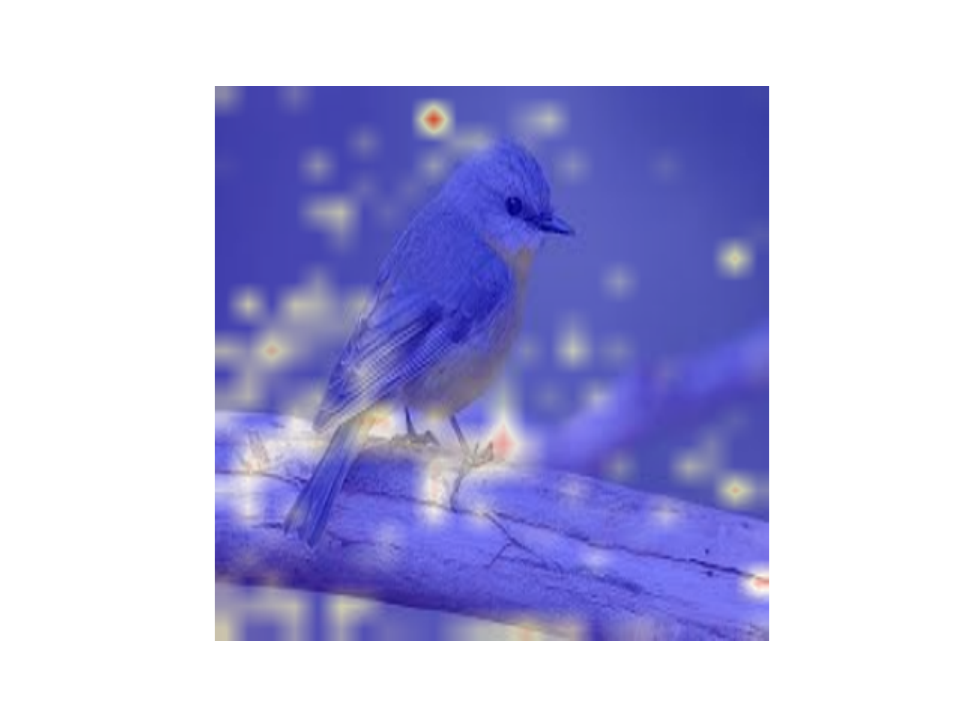}&
\includegraphics[width=0.24\columnwidth]{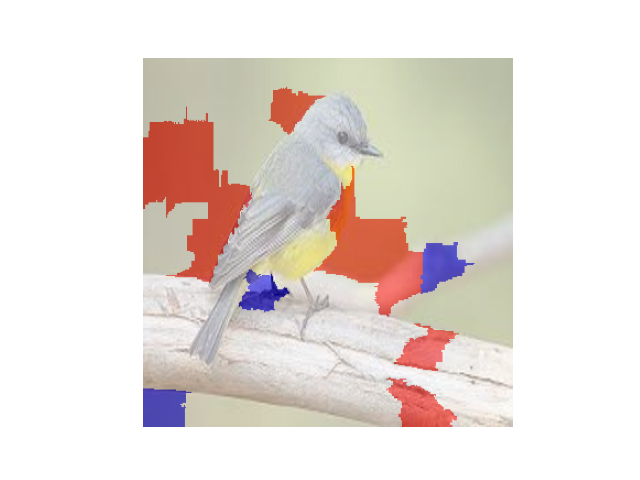}&
\includegraphics[width=0.24\columnwidth]{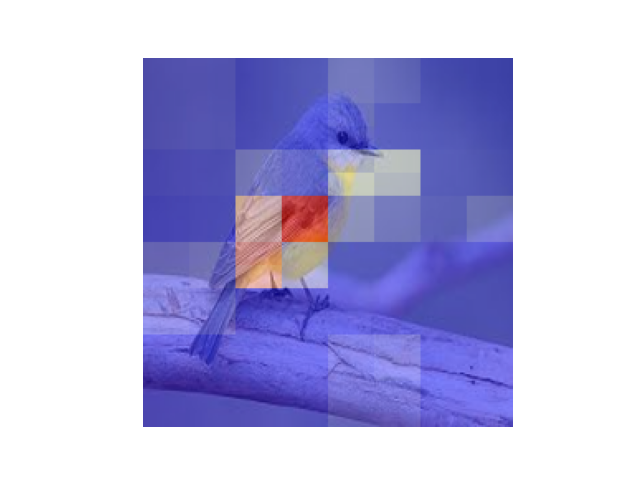}\\
\includegraphics[width=0.24\columnwidth]{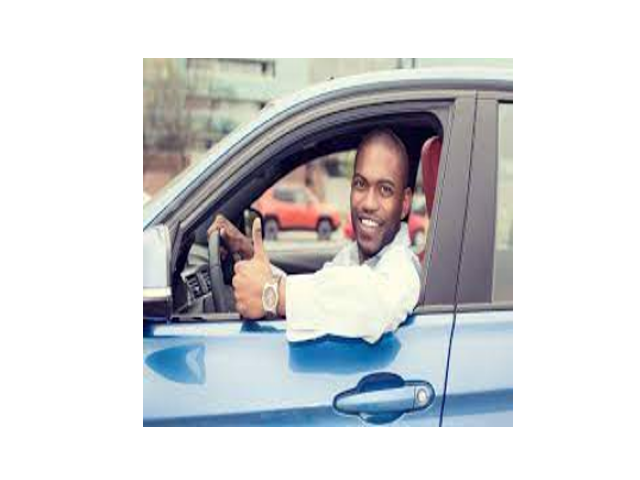}&
\includegraphics[width=0.24\columnwidth]{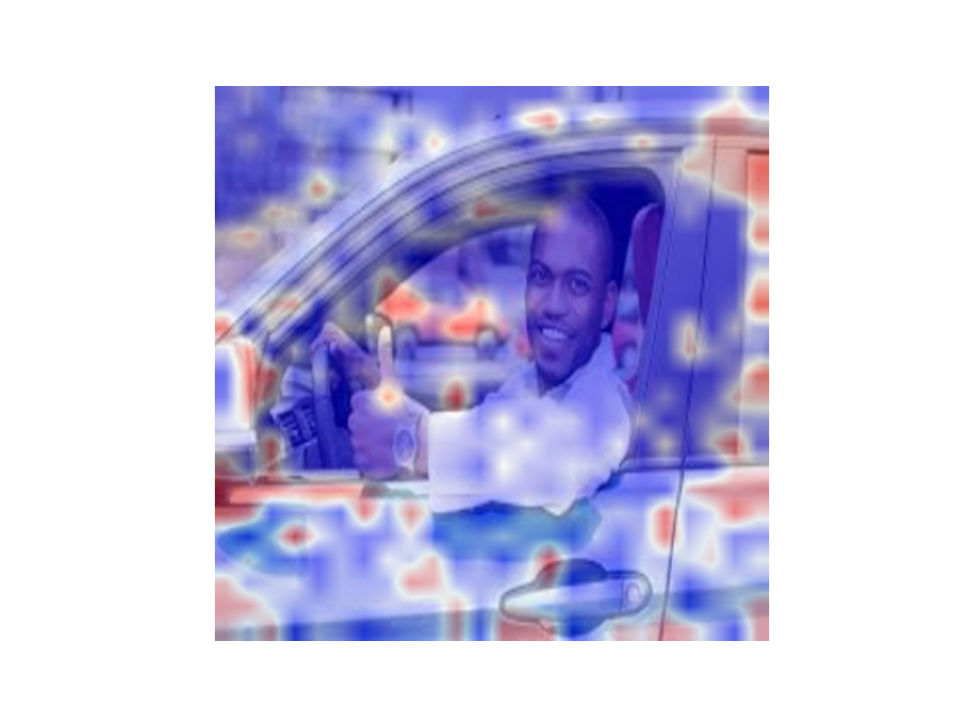}&
\includegraphics[width=0.24\columnwidth]{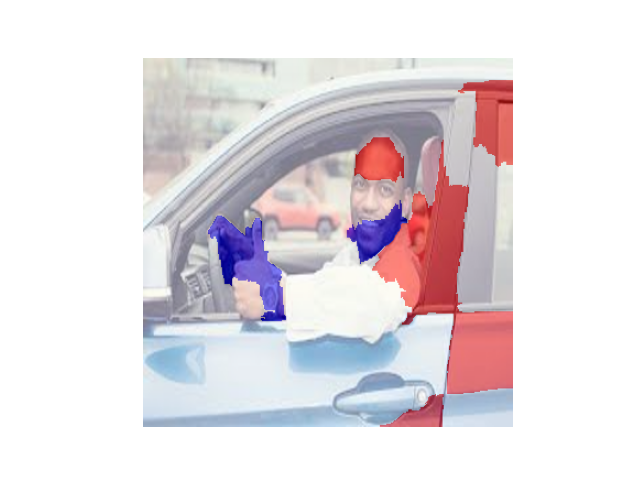}&
\includegraphics[width=0.24\columnwidth]{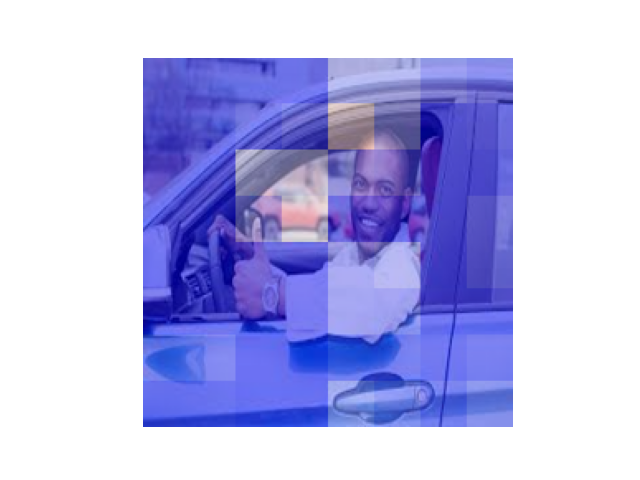}\\
\includegraphics[width=0.24\columnwidth]{Images/samples_cats_dogs2/Car/Car.png}&
\includegraphics[width=0.24\columnwidth]{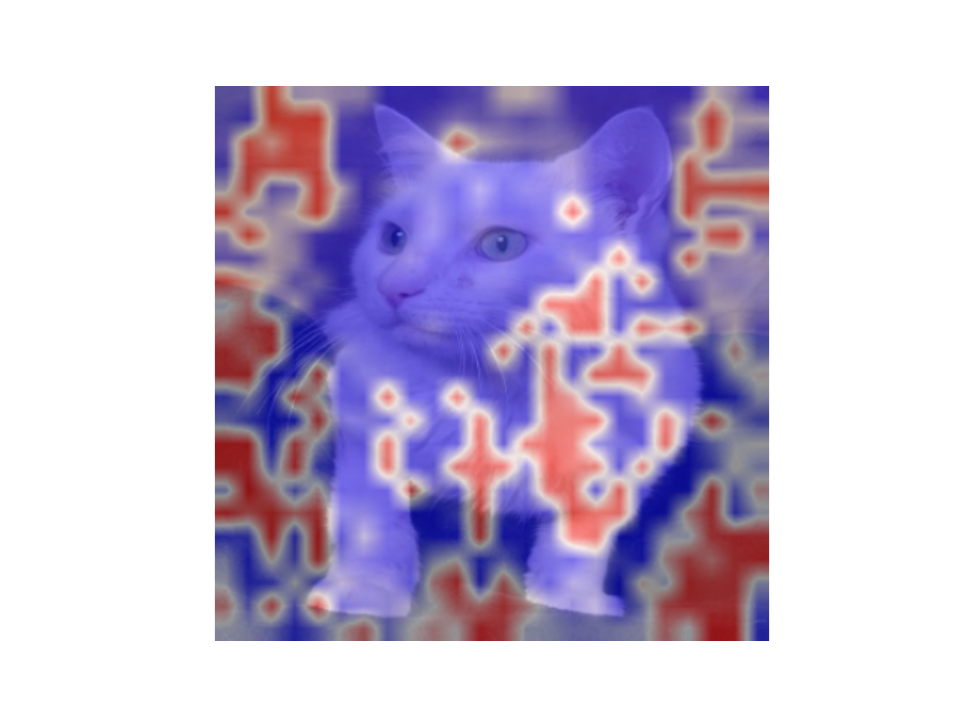}&
\includegraphics[width=0.24\columnwidth]{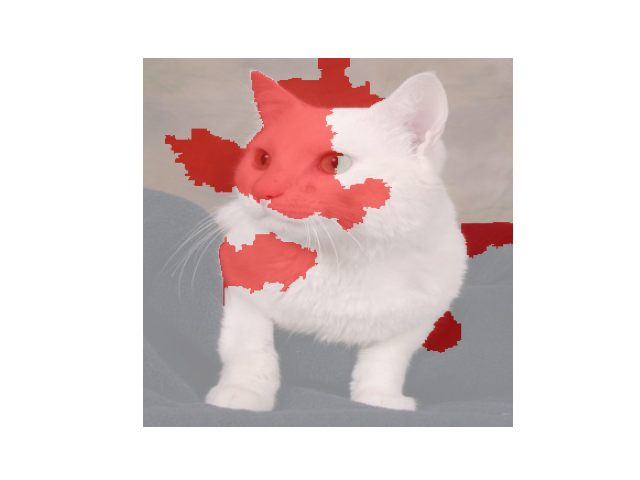}&
\includegraphics[width=0.24\columnwidth]{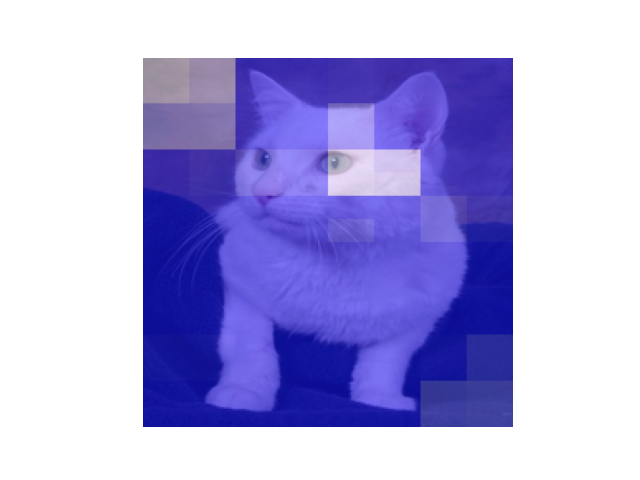}\\
\includegraphics[width=0.24\columnwidth]{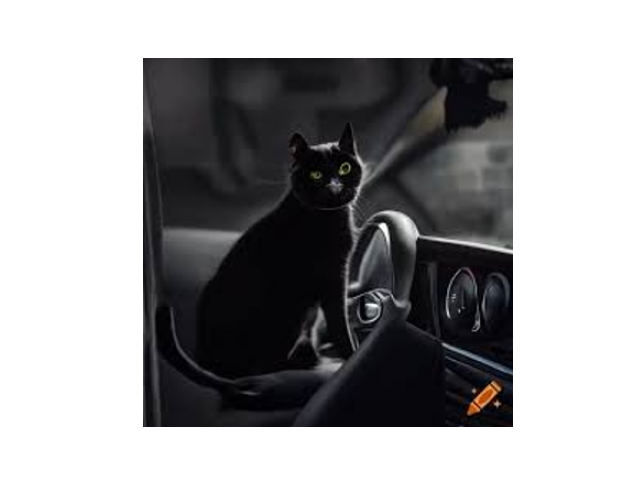}&
\includegraphics[width=0.24\columnwidth]{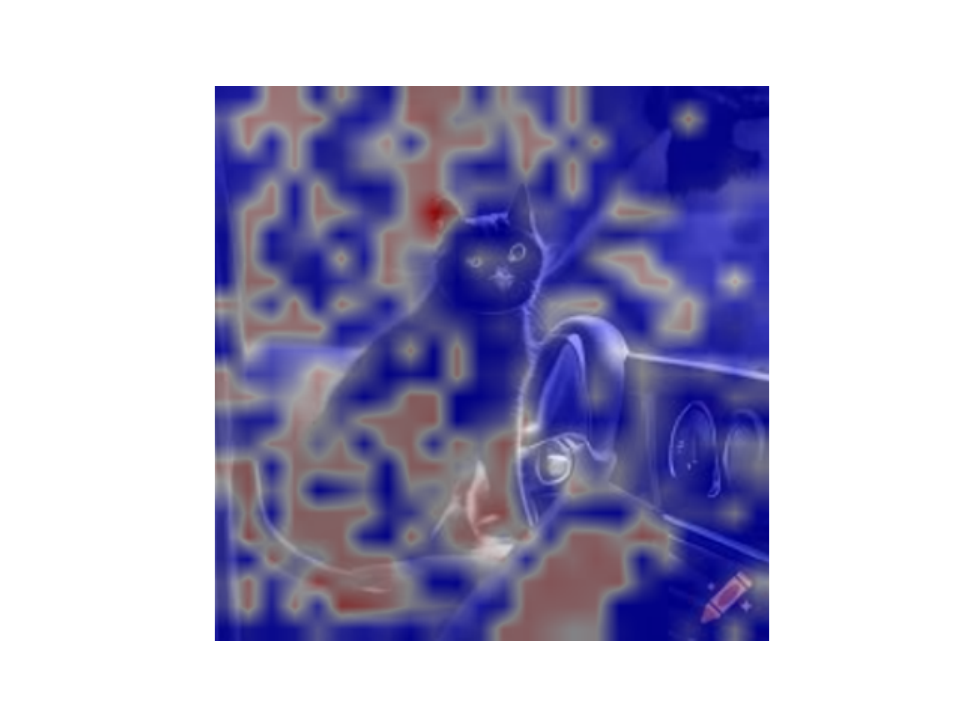}&
\includegraphics[width=0.24\columnwidth]{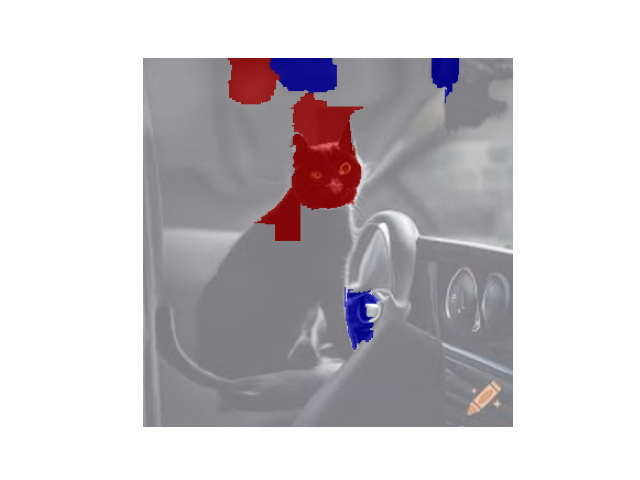}&
\includegraphics[width=0.24\columnwidth]{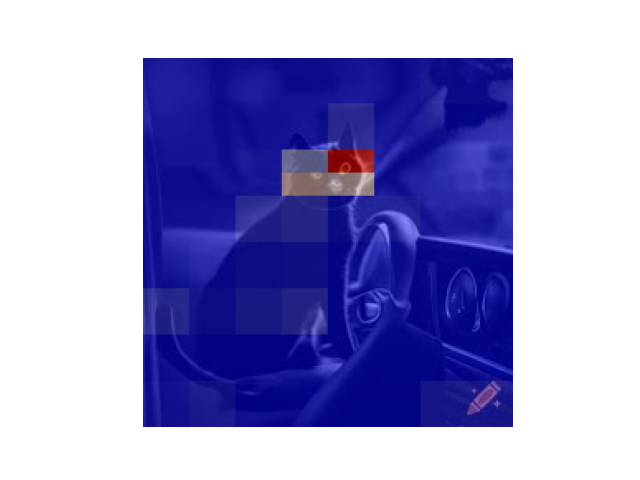}\\
\end{tabular}
\end{center}
\caption{\updateremi{\textbf{Sample-wise explanations (image level).} The first two examples come from the PascalPART dataset, and the last two samples come from the Cats/Dogs/Cars dataset in the biased setup. Note that the classifier mislabeled the $2^{nd}$ example as ``car'' and the $4^{th}$ example as ``car''.}}
\label{fig:extra_image_level}
\end{figure}

\begin{figure}[htb]
\begin{center}
\begin{adjustwidth}{-2cm}{2cm}
\begin{tabular}{l c c c }
(a) Input image & (b) \methodsample Explanation & (c) Yan et al. (sample) explanation  & (d) \CBMLIME explanation \\
\includegraphics[width=0.15\columnwidth]{Images/samples_pascalpart3/bird/bird.png}&
\includegraphics[width=0.32\columnwidth]{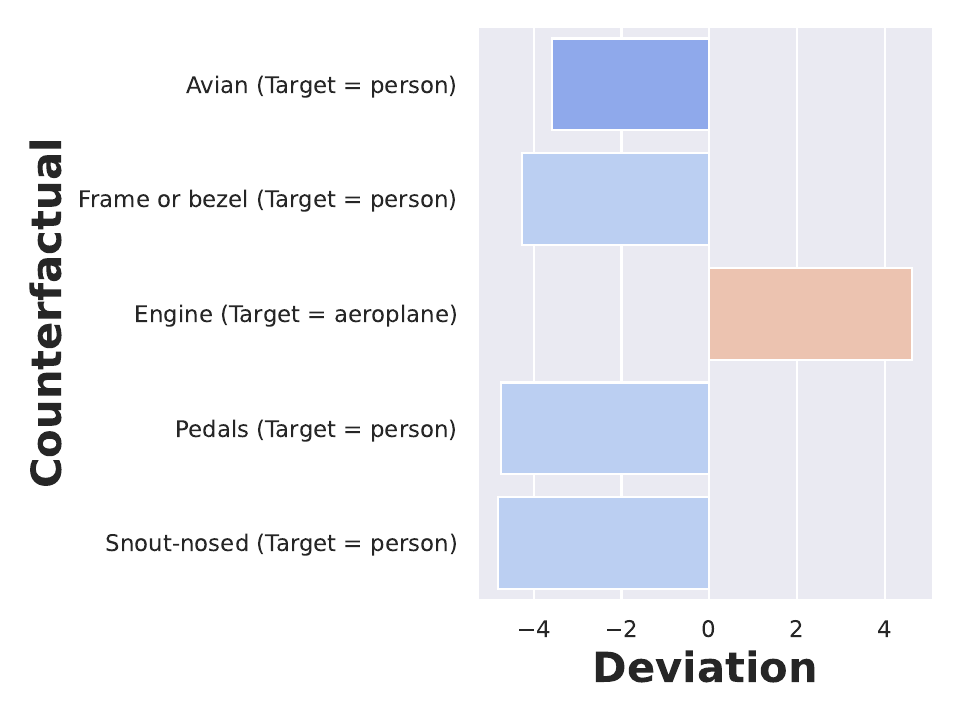}&
\includegraphics[width=0.32\columnwidth]{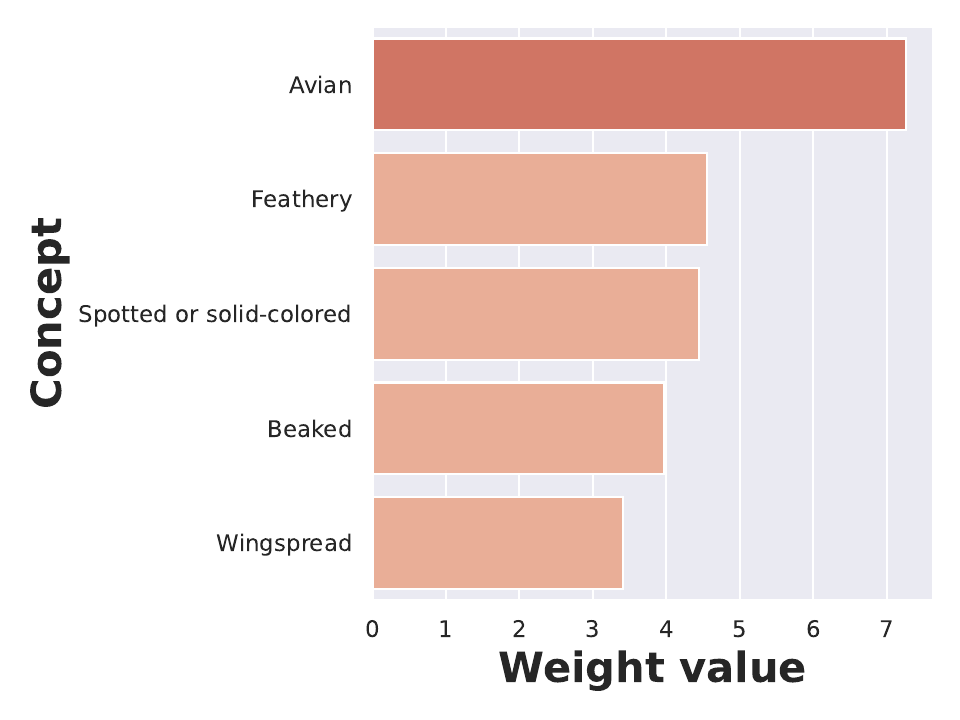}&
\includegraphics[width=0.32\columnwidth]{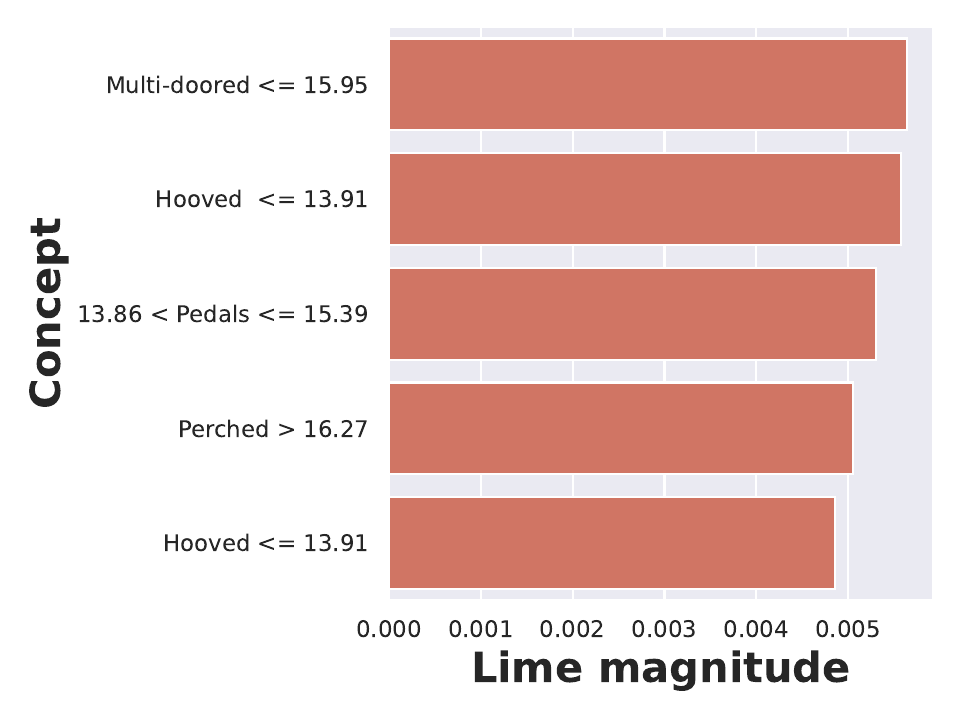}\\
\includegraphics[width=0.15\columnwidth]{Images/samples_pascalpart3/car/car.png}&
\includegraphics[width=0.32\columnwidth]{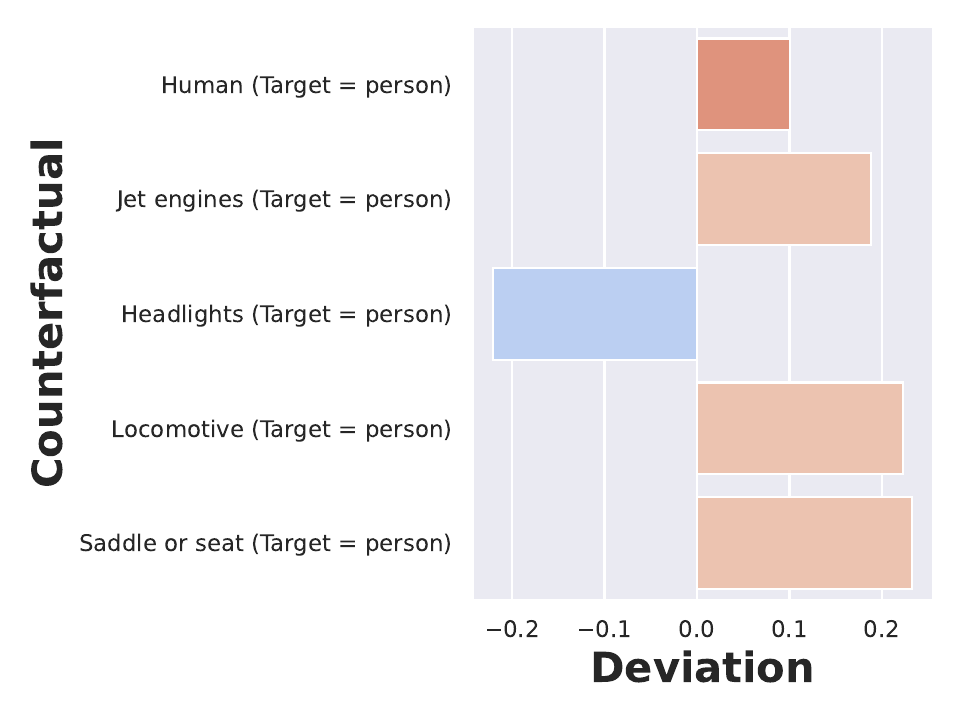}&
\includegraphics[width=0.32\columnwidth]{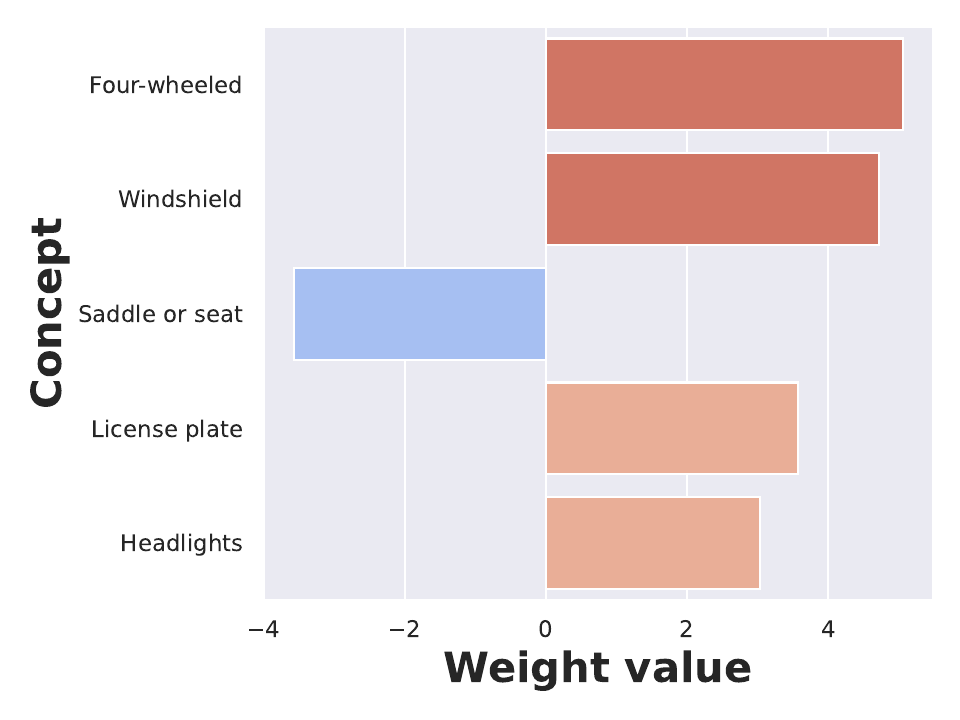}&
\includegraphics[width=0.32\columnwidth]{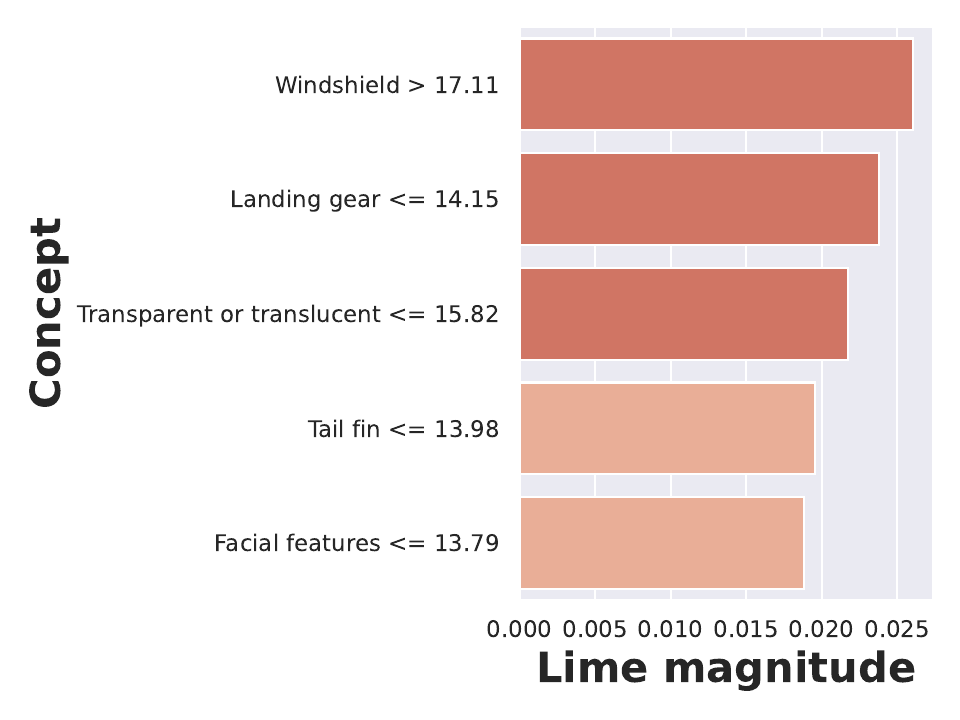}\\
\includegraphics[width=0.15\columnwidth]{Images/samples_cats_dogs2/Car/Car.png}&
\includegraphics[width=0.32\columnwidth]{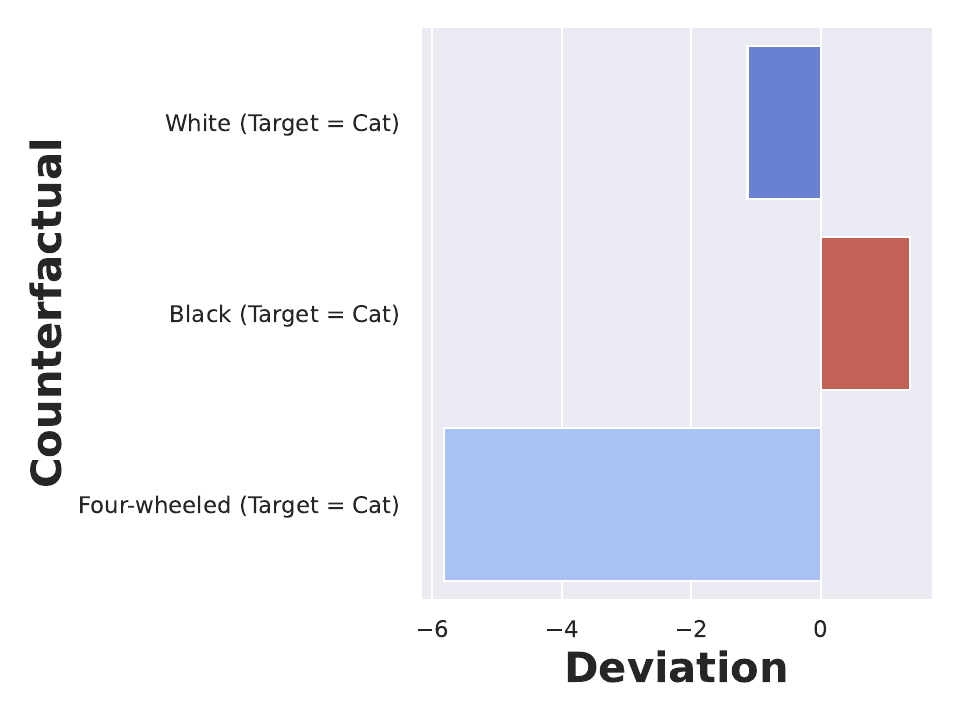}&
\includegraphics[width=0.32\columnwidth]{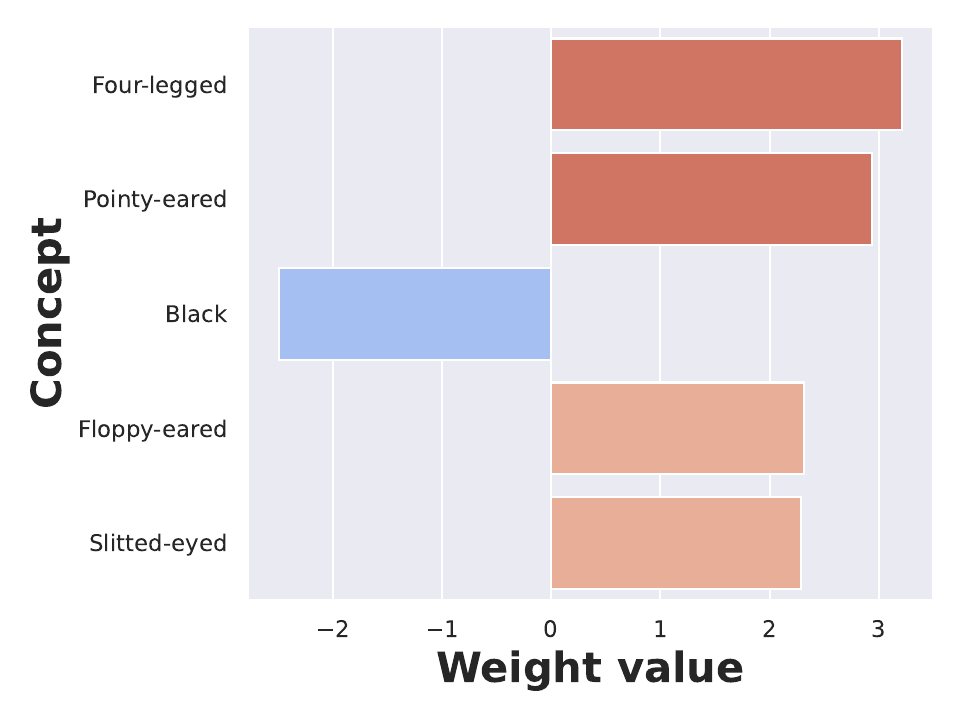}&
\includegraphics[width=0.32\columnwidth]{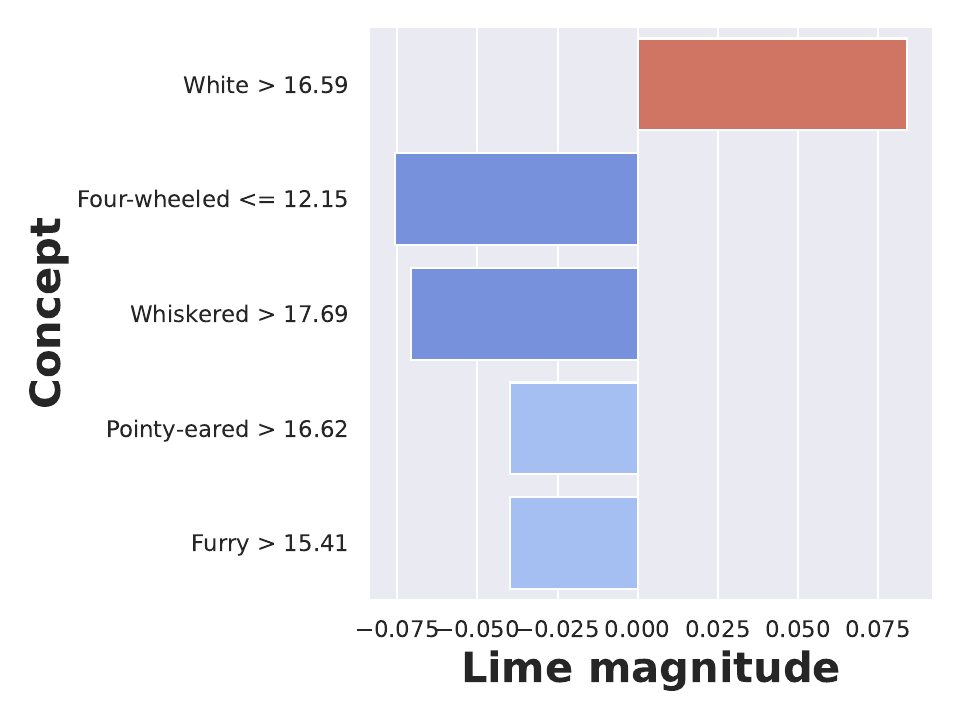}\\
\includegraphics[width=0.15\columnwidth]{Images/samples_cats_dogs2/Cat/Cat.png}&
\includegraphics[width=0.32\columnwidth]{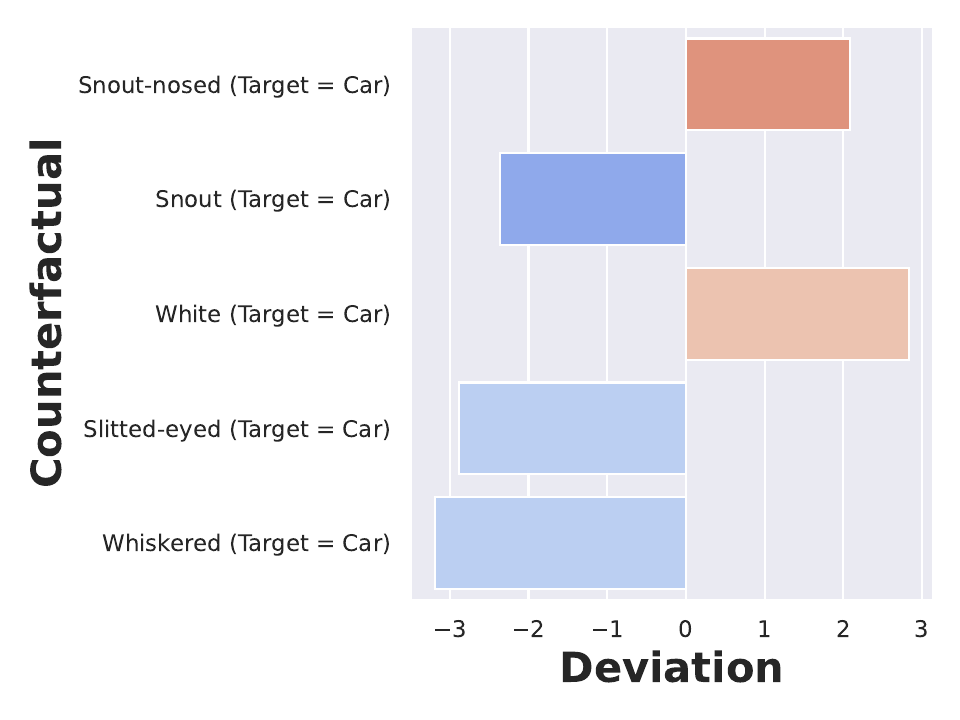}&
\includegraphics[width=0.32\columnwidth]{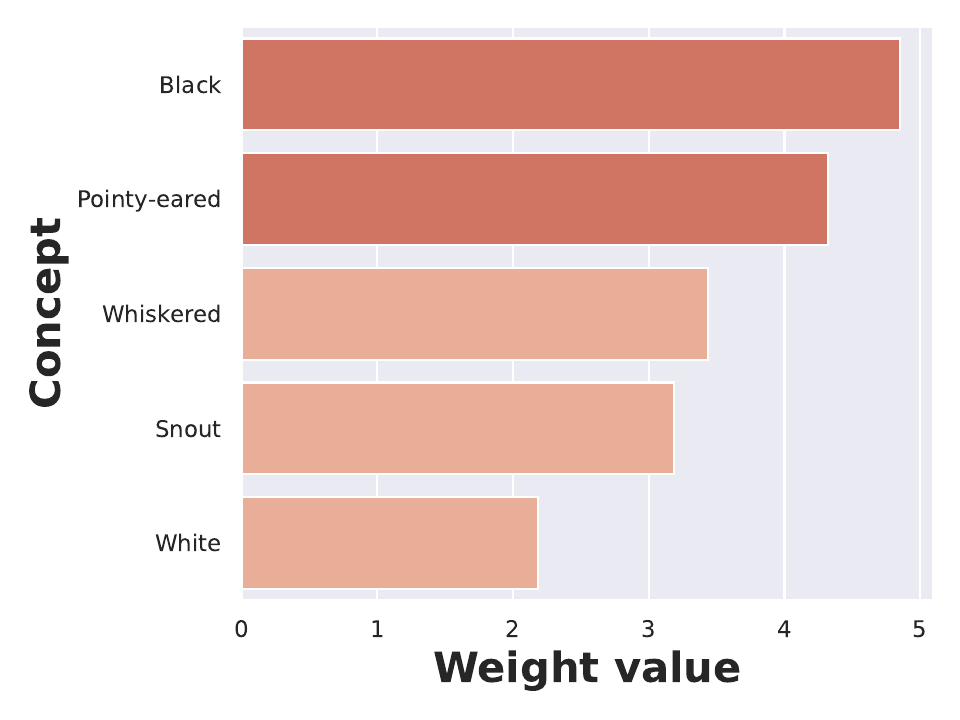}&
\includegraphics[width=0.32\columnwidth]{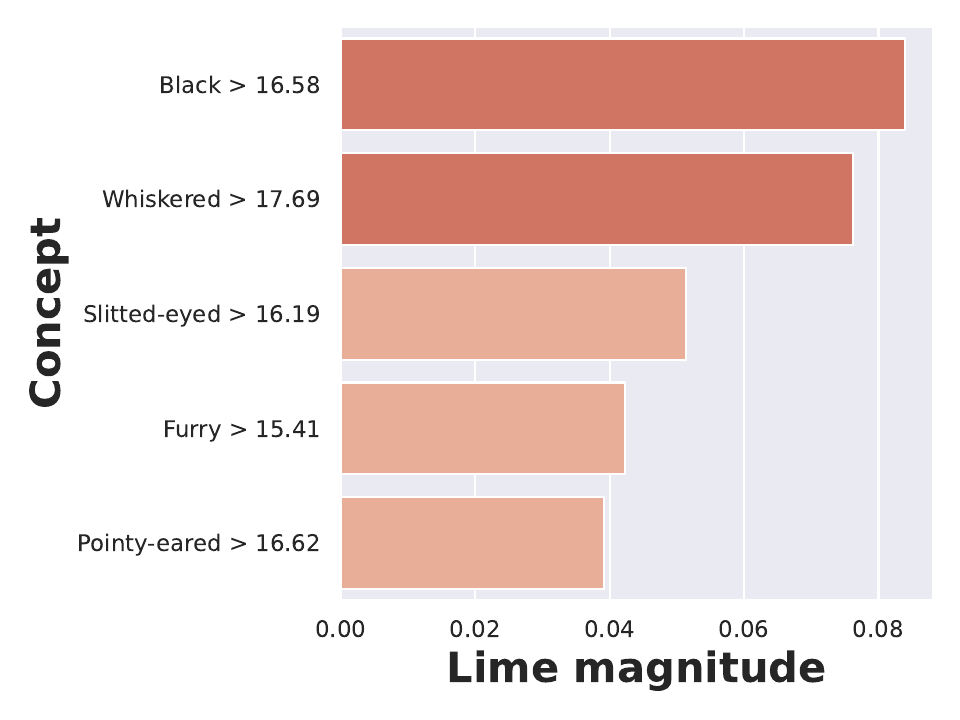}\\
\end{tabular}
\end{adjustwidth}
\end{center}
\caption{\updateremi{\textbf{Sample-wise explanations (concept level).} The first two examples come from the PascalPART dataset, and the last two samples come from the Cats/Dogs/Cars dataset in the biased setup. Note that the classifier mislabeled the $2^{nd}$ example as ``car'' and the $4^{th}$ example as ``car''.}}
\label{fig:extra_concept_level}
\end{figure}

\begin{figure}[htb]
\begin{center}
\begin{adjustwidth}{-2cm}{2cm}
\begin{tabular}{l c}
(a) Input image & (b) \CBMshap Explanation \\
\includegraphics[width=0.15\columnwidth]{Images/samples_pascalpart3/bird/bird.png}&
\includegraphics[width=1\columnwidth]{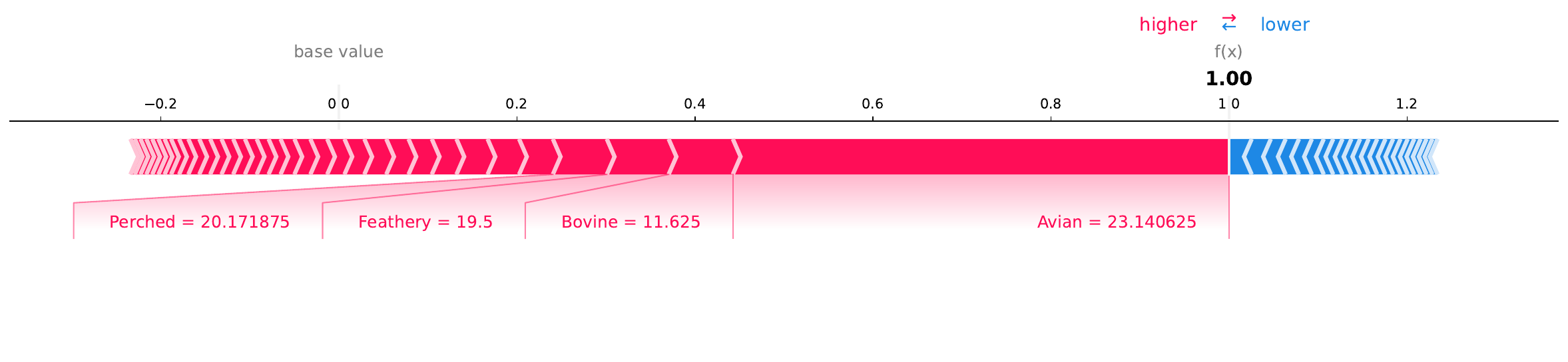}\\
\includegraphics[width=0.15\columnwidth]{Images/samples_pascalpart3/car/car.png}&
\includegraphics[width=1\columnwidth]{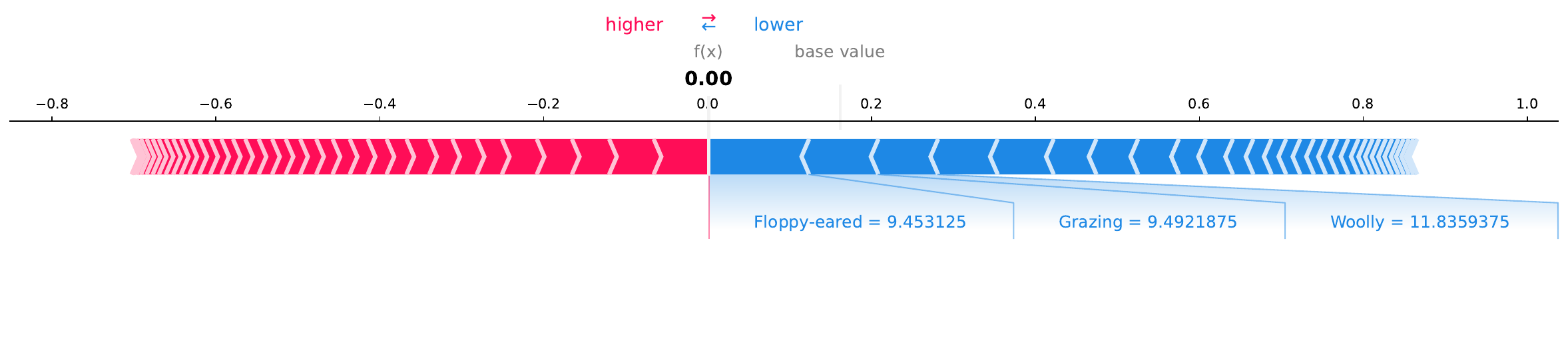}\\
\includegraphics[width=0.15\columnwidth]{Images/samples_cats_dogs2/Car/Car.png}&
\includegraphics[width=1\columnwidth]{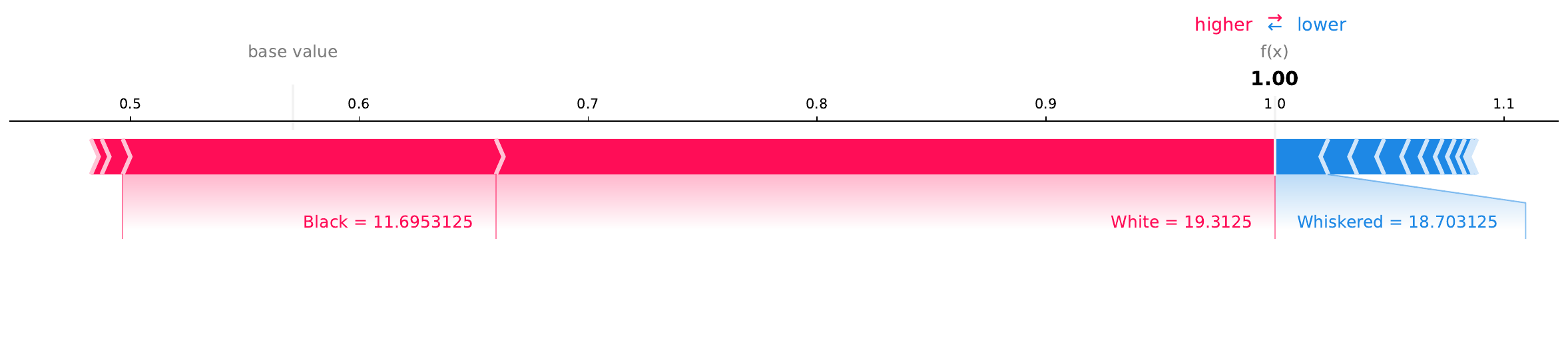}\\
\includegraphics[width=0.15\columnwidth]{Images/samples_cats_dogs2/Cat/Cat.png}&
\includegraphics[width=1\columnwidth]{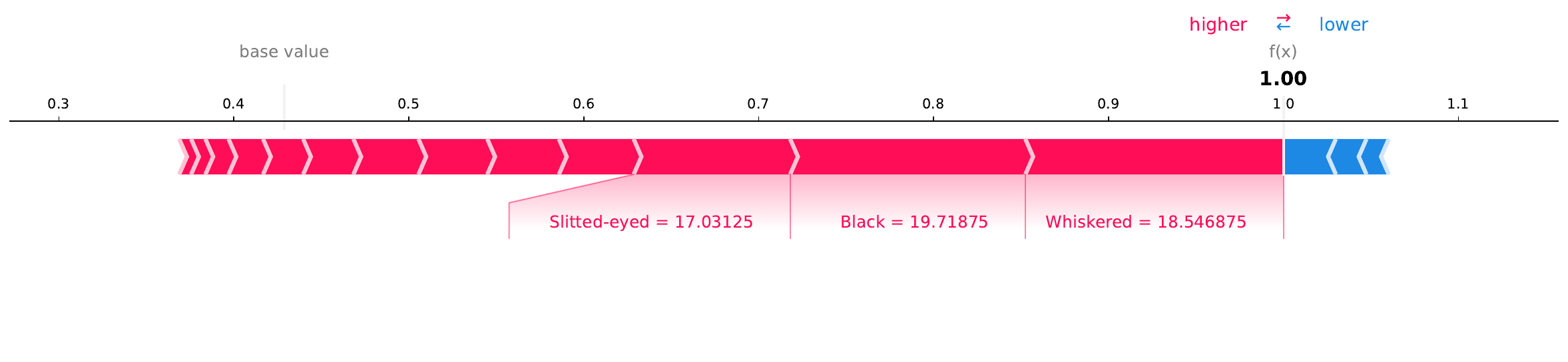}\\
\end{tabular}
\end{adjustwidth}
\end{center}
\caption{\updateremi{\textbf{Sample-wise explanations (concept level).} The first two examples come from the PascalPART dataset, and the last two samples come from the Cats/Dogs/Cars dataset in the biased setup. Note that the classifier mislabeled the $2^{nd}$ example as ``car'' and the $4^{th}$ example as ``car''.}}
\label{fig:extra_concept_level2}
\end{figure}

\begin{figure}[htb]
\begin{center}
\begin{adjustwidth}{-2cm}{2cm}
\begin{tabular}{l c c c }
(a) Input image & (b) \methoddata Explanation & (c) LaBo explanation & (d) Yan et al. (dataset) explanation\\
\includegraphics[width=0.15\columnwidth]{Images/samples_pascalpart3/bird/bird.png}&
\includegraphics[width=0.32\columnwidth]{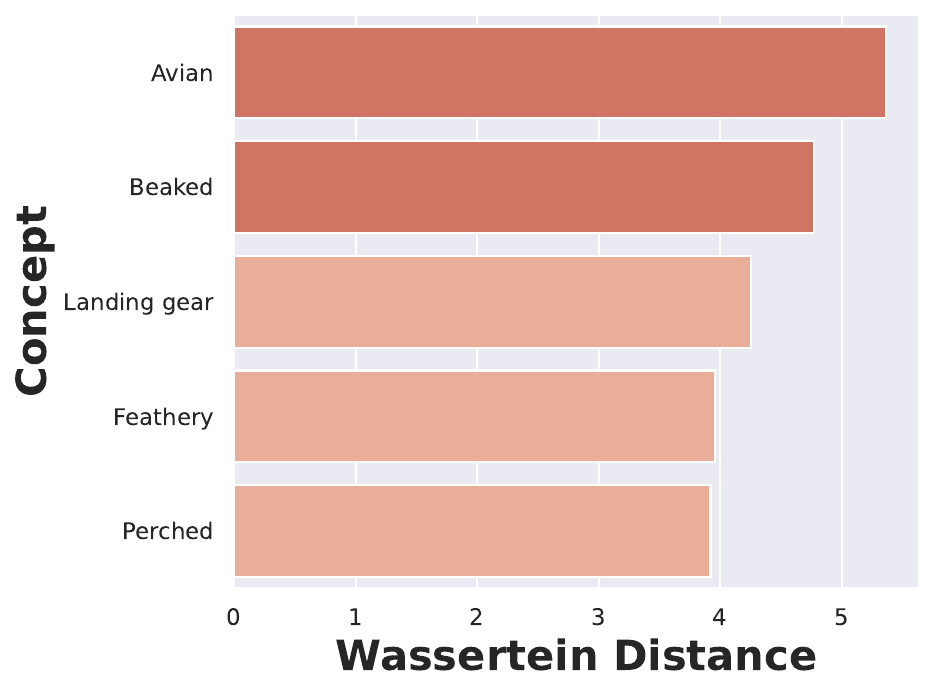}&
\includegraphics[width=0.32\columnwidth]{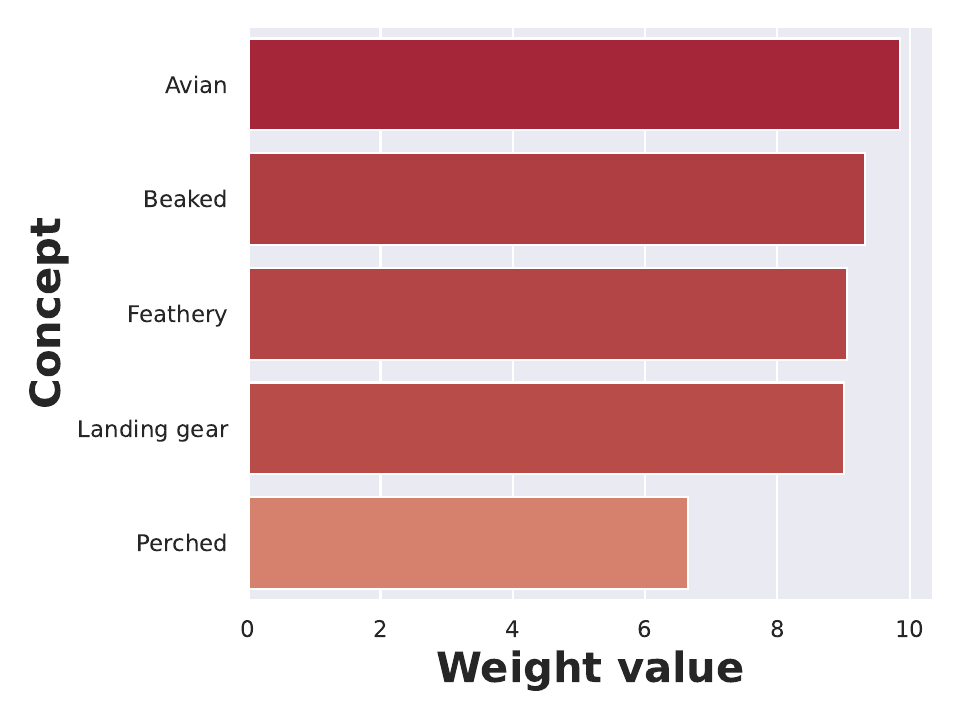}&
\includegraphics[width=0.32\columnwidth]{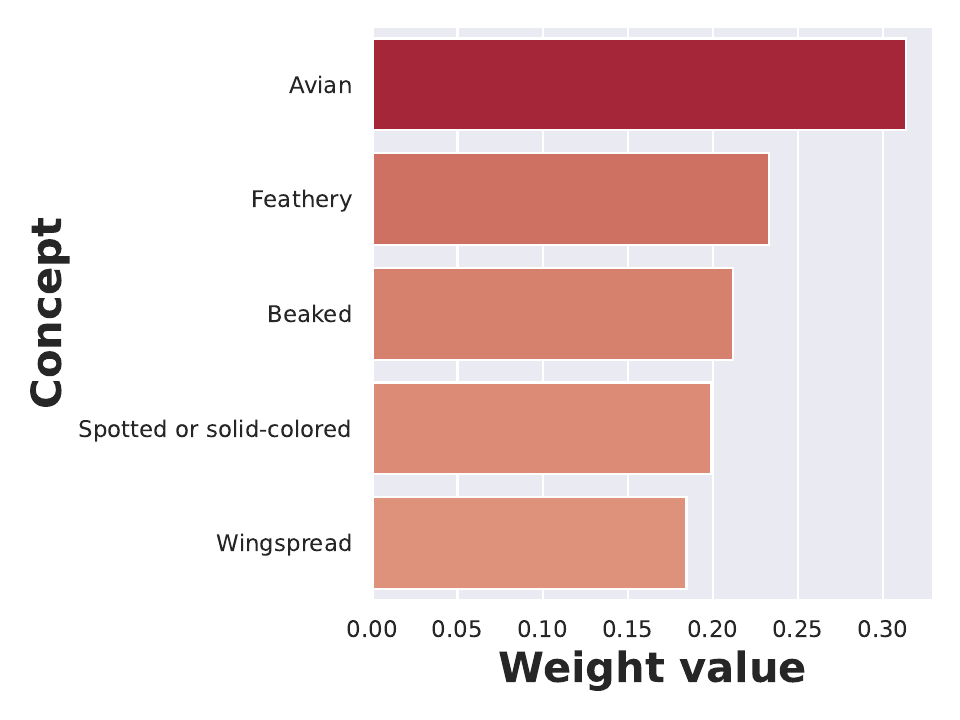}\\
\includegraphics[width=0.15\columnwidth]{Images/samples_pascalpart3/car/car.png}&
\includegraphics[width=0.32\columnwidth]{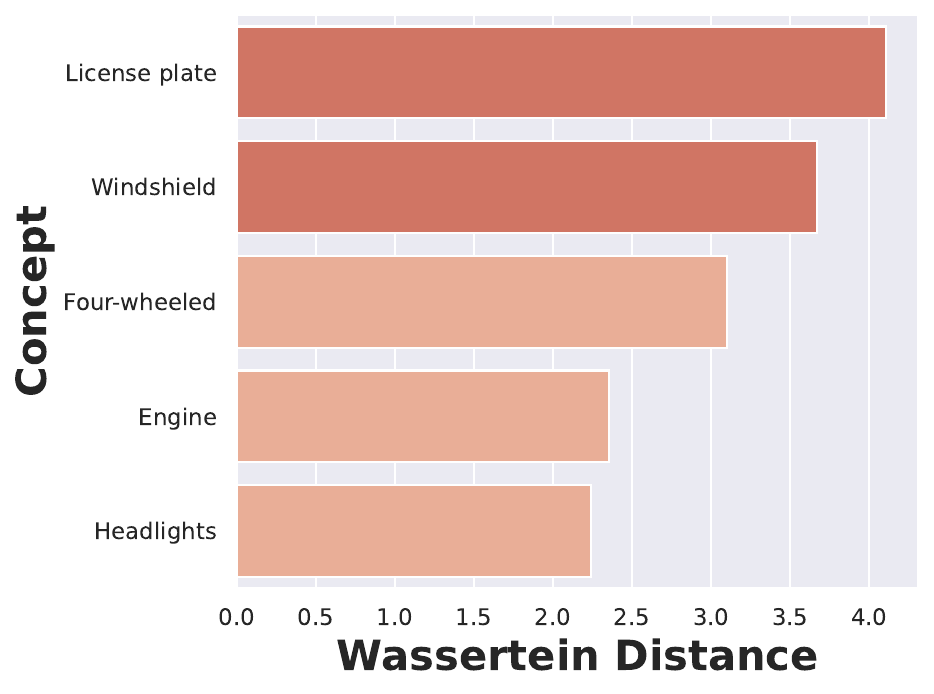}&
\includegraphics[width=0.32\columnwidth]{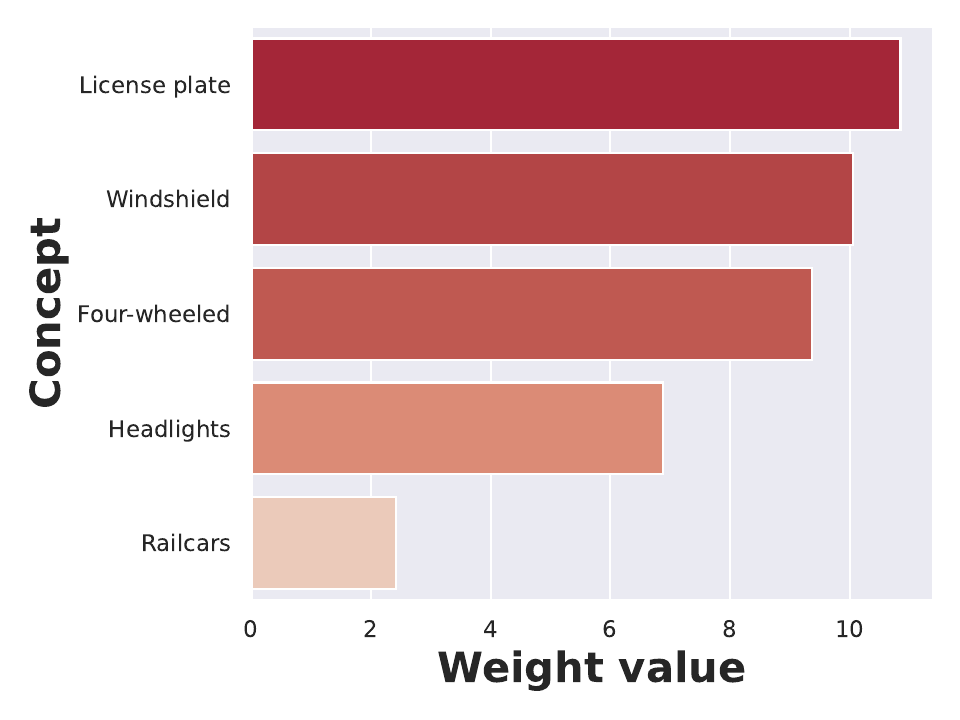}&
\includegraphics[width=0.32\columnwidth]{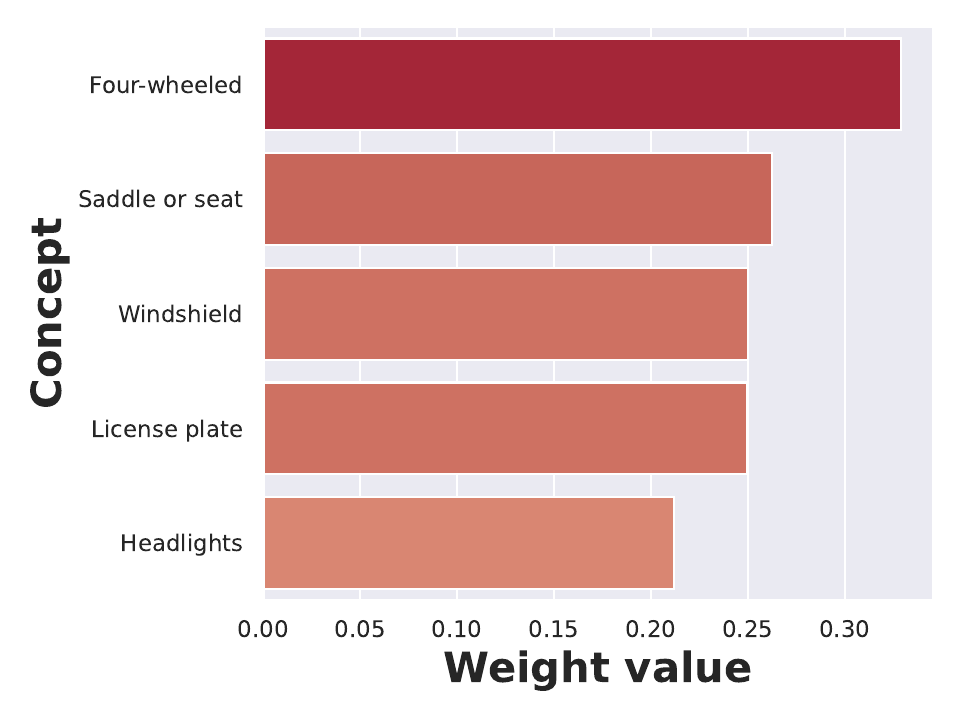}\\
\includegraphics[width=0.15\columnwidth]{Images/samples_cats_dogs2/Car/Car.png}&
\includegraphics[width=0.32\columnwidth]{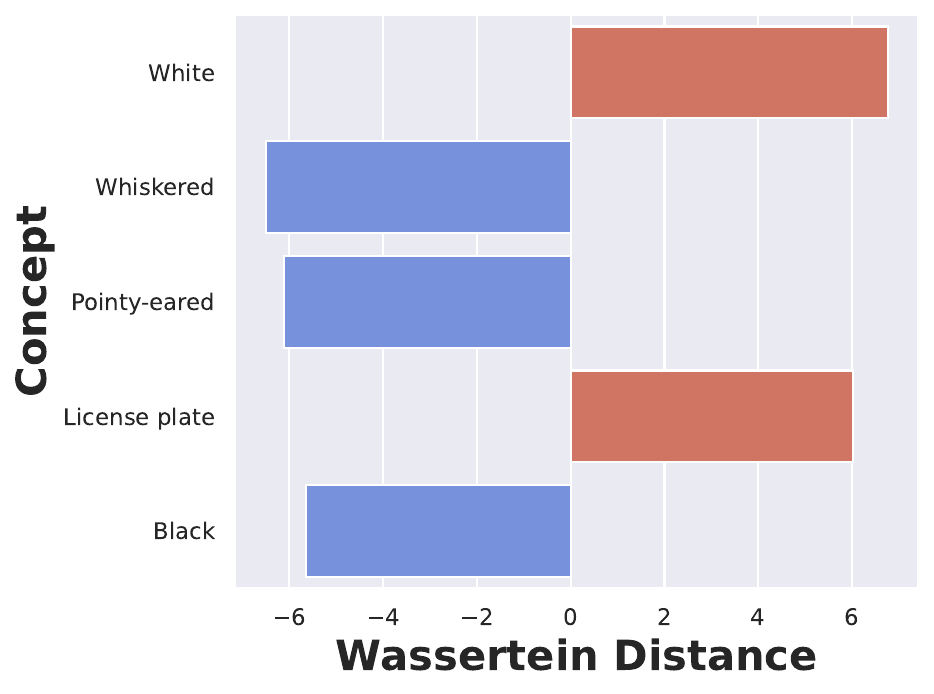}&
\includegraphics[width=0.32\columnwidth]{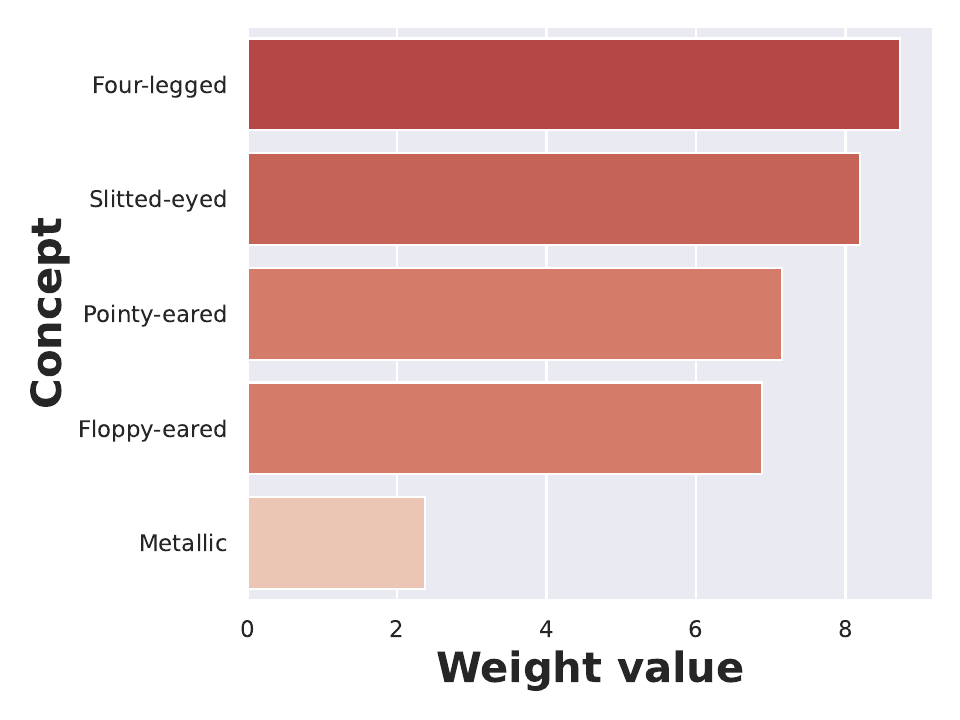}&
\includegraphics[width=0.32\columnwidth]{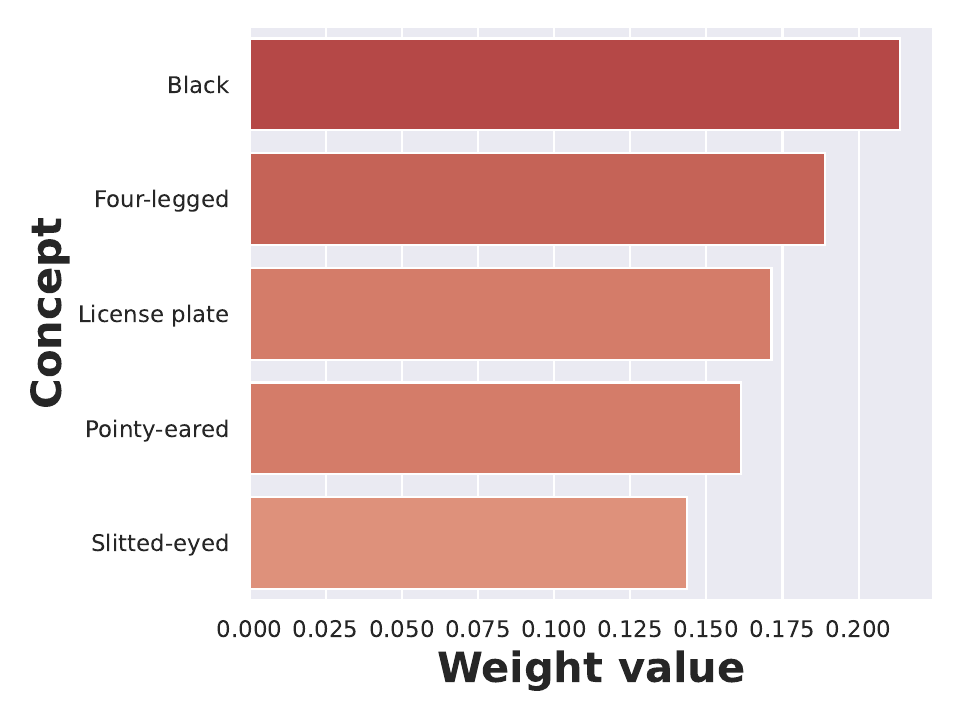}\\
\includegraphics[width=0.15\columnwidth]{Images/samples_cats_dogs2/Cat/Cat.png}&
\includegraphics[width=0.32\columnwidth]{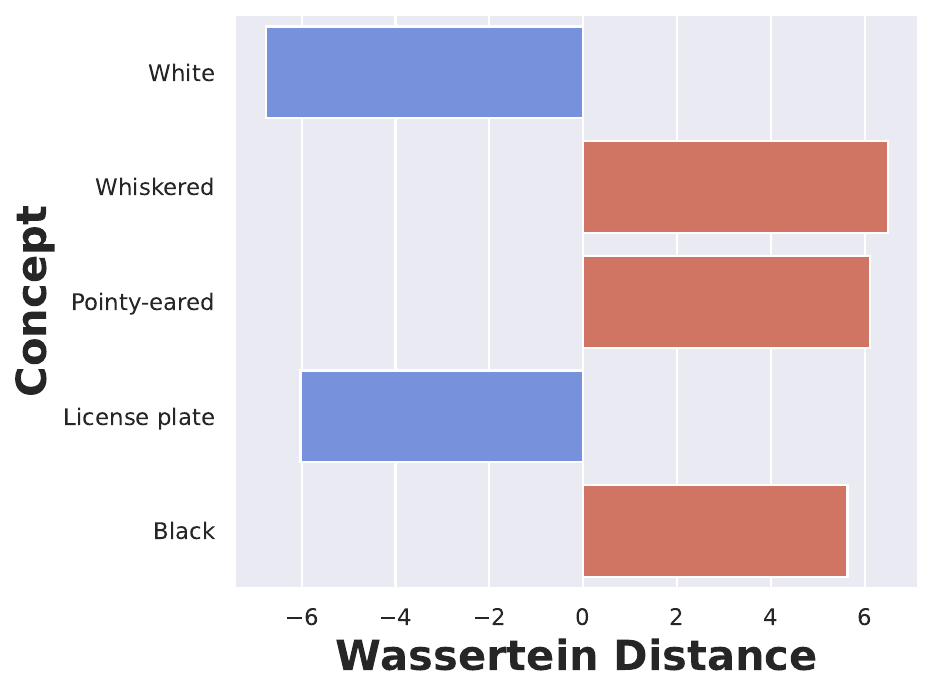}&
\includegraphics[width=0.32\columnwidth]{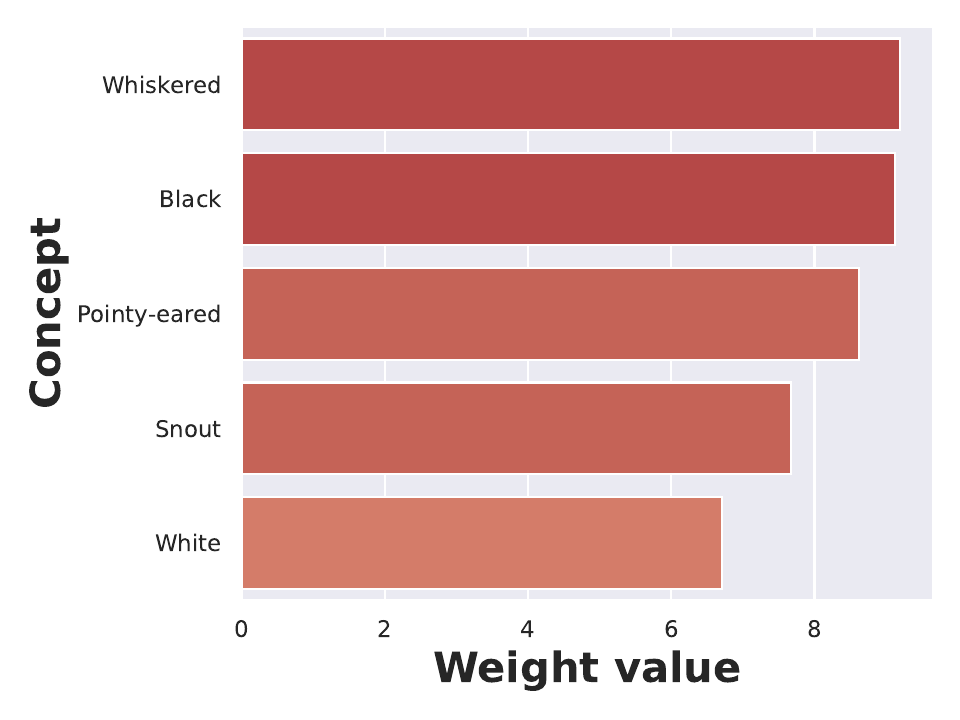}&
\includegraphics[width=0.32\columnwidth]{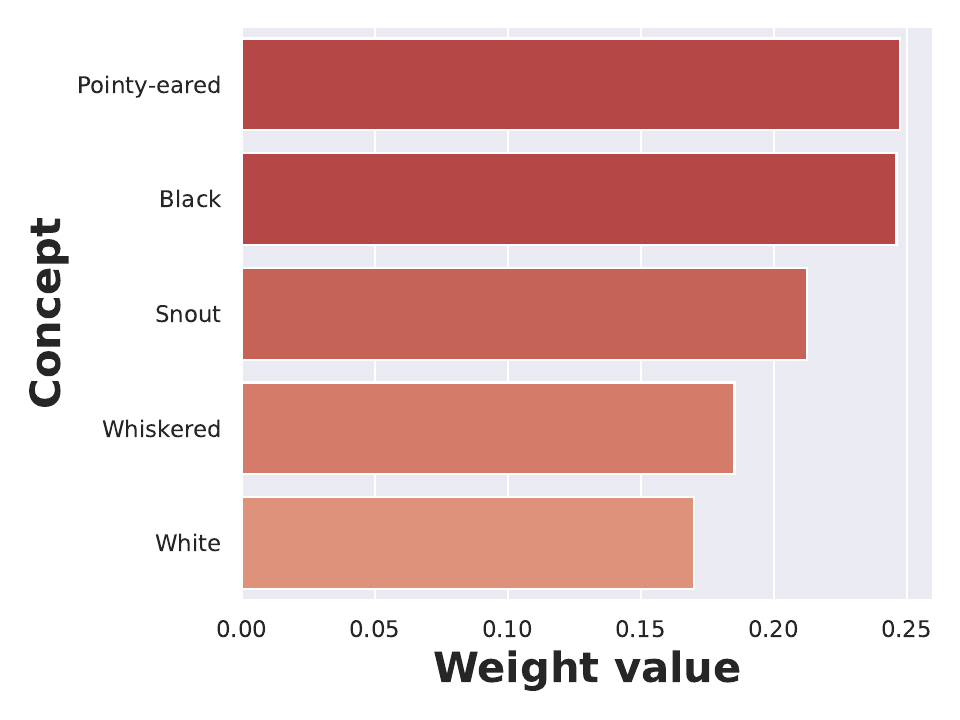}\\
\end{tabular}
\end{adjustwidth}
\end{center}
\caption{\updateremi{\textbf{Dataset-wise explanations.} The first two examples come from the PascalPART dataset, and the last two samples come from the Cats/Dogs/Cars dataset in the biased setup. Note that the classifier mislabeled the $2^{nd}$ example as ``car'' and the $4^{th}$ example as ``car''.}}
\label{fig:extra_detaset_wise}
\end{figure}

\end{document}

%% file: math_commands.tex
\usepackage{amsmath,amsfonts,bm}

\def\eqref#1{equation~\ref{#1}}

\def\1{\bm{1}}

\def\vmu{{\bm{\mu}}}

\def\vk{{\bm{k}}}

\def\vx{{\bm{x}}}

\def\vz{{\bm{z}}}

\def\vomega{{\bm{\omega}}}

\def\vSigma{{\bm{\Sigma}}}

\def\vZ{{\bm{Z}}}
\def\vepsilon{{\bm{\epsilon}}}

\DeclareMathAlphabet{\mathsfit}{\encodingdefault}{\sfdefault}{m}{sl}
\SetMathAlphabet{\mathsfit}{bold}{\encodingdefault}{\sfdefault}{bx}{n}

%% file: main.bbl
\begin{thebibliography}{57}
\providecommand{\natexlab}[1]{#1}
\providecommand{\url}[1]{\texttt{#1}}
\expandafter\ifx\csname urlstyle\endcsname\relax
  \providecommand{\doi}[1]{doi: #1}\else
  \providecommand{\doi}{doi: \begingroup \urlstyle{rm}\Url}\fi

\bibitem[Akkus et~al.(2023)Akkus, Chu, Djakovic, Jauch-Walser, Koch, Loss, Marquardt, Moldovan, Sauter, Schneider, et~al.]{akkus2023multimodal}
Cem Akkus, Luyang Chu, Vladana Djakovic, Steffen Jauch-Walser, Philipp Koch, Giacomo Loss, Christopher Marquardt, Marco Moldovan, Nadja Sauter, Maximilian Schneider, et~al.
\newblock Multimodal deep learning.
\newblock \emph{arXiv preprint arXiv:2301.04856}, 2023.

\bibitem[Arrieta et~al.(2020)Arrieta, D{\'\i}az-Rodr{\'\i}guez, Del~Ser, Bennetot, Tabik, Barbado, Garc{\'\i}a, Gil-L{\'o}pez, Molina, Benjamins, et~al.]{arrieta2020explainable}
Alejandro~Barredo Arrieta, Natalia D{\'\i}az-Rodr{\'\i}guez, Javier Del~Ser, Adrien Bennetot, Siham Tabik, Alberto Barbado, Salvador Garc{\'\i}a, Sergio Gil-L{\'o}pez, Daniel Molina, Richard Benjamins, et~al.
\newblock Explainable artificial intelligence ({XAI}): Concepts, taxonomies, opportunities and challenges toward responsible {AI}.
\newblock \emph{Information Fusion}, 58:\penalty0 82--115, 2020.

\bibitem[Bennetot et~al.(2022)Bennetot, Franchi, Del~Ser, Chatila, and Diaz-Rodriguez]{bennetot2022greybox}
Adrien Bennetot, Gianni Franchi, Javier Del~Ser, Raja Chatila, and Natalia Diaz-Rodriguez.
\newblock Greybox {XAI}: A neural-symbolic learning framework to produce interpretable predictions for image classification.
\newblock \emph{Knowledge-Based Systems}, 258:\penalty0 109947, 2022.

\bibitem[Brown et~al.(2020)Brown, Mann, Ryder, Subbiah, Kaplan, Dhariwal, Neelakantan, Shyam, Sastry, Askell, et~al.]{brown2020language}
Tom Brown, Benjamin Mann, Nick Ryder, Melanie Subbiah, Jared~D. Kaplan, Prafulla Dhariwal, Arvind Neelakantan, Pranav Shyam, Girish Sastry, Amanda Askell, et~al.
\newblock Language models are few-shot learners.
\newblock \emph{Advances in Neural Information Processing Systems}, 33:\penalty0 1877--1901, 2020.

\bibitem[Chambers(2018)]{chambers2018graphical}
John~M. Chambers.
\newblock \emph{Graphical Methods for Data Analysis}.
\newblock CRC Press, 2018.

\bibitem[Cortez \& Embrechts(2011)Cortez and Embrechts]{cortez2011opening}
Paulo Cortez and Mark~J. Embrechts.
\newblock Opening black box data mining models using sensitivity analysis.
\newblock In \emph{2011 IEEE Symposium on Computational Intelligence and Data Mining (CIDM)}, pp.\  341--348. IEEE, 2011.

\bibitem[Cukierski(2013)]{dogs-vs-cats}
Will Cukierski.
\newblock Dogs vs. cats, 2013.
\newblock URL \url{https://kaggle.com/competitions/dogs-vs-cats}.

\bibitem[Deng et~al.(2009)Deng, Dong, Socher, Li, Li, and Fei-Fei]{deng2009imagenet}
Jia Deng, Wei Dong, Richard Socher, Li-Jia Li, Kai Li, and Li~Fei-Fei.
\newblock Imagenet: A large-scale hierarchical image database.
\newblock In \emph{2009 IEEE Conference on Computer Vision and Pattern Recognition}, pp.\  248--255, 2009.

\bibitem[D{\'\i}az-Rodr{\'\i}guez et~al.(2022)D{\'\i}az-Rodr{\'\i}guez, Lamas, Sanchez, Franchi, Donadello, Tabik, Filliat, Cruz, Montes, and Herrera]{diaz2022explainable}
Natalia D{\'\i}az-Rodr{\'\i}guez, Alberto Lamas, Jules Sanchez, Gianni Franchi, Ivan Donadello, Siham Tabik, David Filliat, Policarpo Cruz, Rosana Montes, and Francisco Herrera.
\newblock Explainable neural-symbolic learning (x-nesyl) methodology to fuse deep learning representations with expert knowledge graphs: The monumai cultural heritage use case.
\newblock \emph{Information Fusion}, 79:\penalty0 58--83, 2022.

\bibitem[Donadello \& Serafini(2016)Donadello and Serafini]{DBLP:journals/ia/DonadelloS16}
Ivan Donadello and Luciano Serafini.
\newblock Integration of numeric and symbolic information for semantic image interpretation.
\newblock \emph{Intelligenza Artificiale}, 10\penalty0 (1):\penalty0 33--47, 2016.

\bibitem[Dosovitskiy et~al.(2021)Dosovitskiy, Beyer, Kolesnikov, Weissenborn, Zhai, Unterthiner, Dehghani, Minderer, Heigold, Gelly, et~al.]{dosovitskiy2020image}
Alexey Dosovitskiy, Lucas Beyer, Alexander Kolesnikov, Dirk Weissenborn, Xiaohua Zhai, Thomas Unterthiner, Mostafa Dehghani, Matthias Minderer, Georg Heigold, Sylvain Gelly, et~al.
\newblock An image is worth 16x16 words: Transformers for image recognition at scale.
\newblock In \emph{International Conference on Learning Representations}, 2021.

\bibitem[Feng et~al.(2023)Feng, Bair, and Kolter]{feng2023leveraging}
Zhili Feng, Anna Bair, and J.~Zico Kolter.
\newblock Leveraging multiple descriptive features for robust few-shot image learning.
\newblock \emph{arXiv preprint arXiv:2307.04317}, 2023.

\bibitem[Gabeff et~al.(2023)Gabeff, Russwurm, Tuia, and Mathis]{gabeff2023wildclip}
Valentin Gabeff, Marc Russwurm, Devis Tuia, and Alexander Mathis.
\newblock Wildclip: Scene and animal attribute retrieval from camera trap data with domain-adapted vision-language models.
\newblock \emph{bioRxiv}, pp.\  2023--12, 2023.

\bibitem[Galton(1886)]{galton1886regression}
Francis Galton.
\newblock Regression towards mediocrity in hereditary stature.
\newblock \emph{The Journal of the Anthropological Institute of Great Britain and Ireland}, 15:\penalty0 246--263, 1886.

\bibitem[Hastie et~al.(2009)Hastie, Tibshirani, and Friedman]{hastie2009elements}
Trevor Hastie, Robert Tibshirani, and Jerome~H. Friedman.
\newblock \emph{The Elements of Statistical Learning: Data Mining, Inference, and Prediction}, volume~2.
\newblock Springer, 2009.

\bibitem[He et~al.(2016)He, Zhang, Ren, and Sun]{he2016deep}
Kaiming He, Xiangyu Zhang, Shaoqing Ren, and Jian Sun.
\newblock Deep residual learning for image recognition.
\newblock In \emph{Proceedings of the IEEE Conference on Computer Vision and Pattern Recognition}, pp.\  770--778, 2016.

\bibitem[Hedstr{\"{o}}m et~al.(2023)Hedstr{\"{o}}m, Weber, Krakowczyk, Bareeva, Motzkus, Samek, Lapuschkin, and H{\"{o}}hne]{hedstrom2023quantus}
Anna Hedstr{\"{o}}m, Leander Weber, Daniel Krakowczyk, Dilyara Bareeva, Franz Motzkus, Wojciech Samek, Sebastian Lapuschkin, and Marina Marina~M.{-}C. H{\"{o}}hne.
\newblock Quantus: An explainable ai toolkit for responsible evaluation of neural network explanations and beyond.
\newblock \emph{Journal of Machine Learning Research}, 24\penalty0 (34):\penalty0 1--11, 2023.
\newblock URL \url{http://jmlr.org/papers/v24/22-0142.html}.

\bibitem[Johnson et~al.(2016)Johnson, Alahi, and Fei-Fei]{johnson2016perceptual}
Justin Johnson, Alexandre Alahi, and Li~Fei-Fei.
\newblock Perceptual losses for real-time style transfer and super-resolution.
\newblock In \emph{Computer Vision--ECCV 2016: 14th European Conference, Amsterdam, The Netherlands, October 11-14, 2016, Proceedings, Part II 14}, pp.\  694--711. Springer, 2016.

\bibitem[Kim et~al.(2023)Kim, Oh, Lee, Yu, Do, and Taghavi]{kim2023grounding}
Siwon Kim, Jinoh Oh, Sungjin Lee, Seunghak Yu, Jaeyoung Do, and Tara Taghavi.
\newblock Grounding counterfactual explanation of image classifiers to textual concept space.
\newblock In \emph{Proceedings of the IEEE/CVF Conference on Computer Vision and Pattern Recognition}, pp.\  10942--10950, 2023.

\bibitem[Koh et~al.(2020)Koh, Nguyen, Tang, Mussmann, Pierson, Kim, and Liang]{koh2020concept}
Pang~Wei Koh, Thao Nguyen, Yew~Siang Tang, Stephen Mussmann, Emma Pierson, Been Kim, and Percy Liang.
\newblock Concept bottleneck models.
\newblock In \emph{International Conference on Machine Learning}, pp.\  5338--5348, 2020.

\bibitem[Krause et~al.(2013)Krause, Stark, Deng, and Fei-Fei]{KrauseStarkDengFei-Fei_3DRR2013}
Jonathan Krause, Michael Stark, Jia Deng, and Li~Fei-Fei.
\newblock 3d object representations for fine-grained categorization.
\newblock In \emph{4th International IEEE Workshop on 3D Representation and Recognition (3dRR-13)}, Sydney, Australia, 2013.

\bibitem[Kumar et~al.(2009)Kumar, Berg, Belhumeur, and Nayar]{kumar2009attribute}
Neeraj Kumar, Alexander~C. Berg, Peter~N. Belhumeur, and Shree~K. Nayar.
\newblock Attribute and simile classifiers for face verification.
\newblock In \emph{2009 IEEE 12th International Conference on Computer Vision}, pp.\  365--372. IEEE, 2009.

\bibitem[Lamas et~al.(2021)Lamas, Tabik, Cruz, Montes, Martinez-Sevilla, Cruz, and Herrera]{lamas2021monumai}
Alberto Lamas, Siham Tabik, Policarpo Cruz, Rosana Montes, {\'A}lvaro Martinez-Sevilla, Teresa Cruz, and Francisco Herrera.
\newblock Monumai: Dataset, deep learning pipeline and citizen science based app for monumental heritage taxonomy and classification.
\newblock \emph{Neurocomputing}, 420:\penalty0 266--280, 2021.

\bibitem[Lampert et~al.(2009)Lampert, Nickisch, and Harmeling]{lampert2009learning}
Christoph~H. Lampert, Hannes Nickisch, and Stefan Harmeling.
\newblock Learning to detect unseen object classes by between-class attribute transfer.
\newblock In \emph{2009 IEEE Conference on Computer Vision and Pattern Recognition}, pp.\  951--958. IEEE, 2009.

\bibitem[LeCun et~al.(2015)LeCun, Bengio, and Hinton]{lecun2015deep}
Yann LeCun, Yoshua Bengio, and Geoffrey Hinton.
\newblock Deep learning.
\newblock \emph{Nature}, 521\penalty0 (7553):\penalty0 436--444, 2015.

\bibitem[Liu et~al.(2020)Liu, Duh, Liu, and Gao]{liu2020very}
Xiaodong Liu, Kevin Duh, Liyuan Liu, and Jianfeng Gao.
\newblock Very deep transformers for neural machine translation.
\newblock \emph{arXiv preprint arXiv:2008.07772}, 2020.

\bibitem[Losch et~al.(2019)Losch, Fritz, and Schiele]{losch2019interpretability}
Max Losch, Mario Fritz, and Bernt Schiele.
\newblock Interpretability beyond classification output: Semantic bottleneck networks.
\newblock \emph{arXiv preprint arXiv:1907.10882}, 2019.

\bibitem[Lundberg \& Lee(2017)Lundberg and Lee]{lundberg2017unified}
Scott~M. Lundberg and Su-In Lee.
\newblock A unified approach to interpreting model predictions.
\newblock \emph{Advances in Neural Information Processing Systems}, 30, 2017.

\bibitem[Luo et~al.(2022)Luo, Ji, Zhong, Chen, Lei, Duan, and Li]{luo2022clip4clip}
Huaishao Luo, Lei Ji, Ming Zhong, Yang Chen, Wen Lei, Nan Duan, and Tianrui Li.
\newblock Clip4clip: An empirical study of clip for end to end video clip retrieval and captioning.
\newblock \emph{Neurocomputing}, 508:\penalty0 293--304, 2022.

\bibitem[Luo et~al.(2023)Luo, Wang, Wu, Huang, and De~la Torre]{luo2023zero}
Jinqi Luo, Zhaoning Wang, Chen~Henry Wu, Dong Huang, and Fernando De~la Torre.
\newblock Zero-shot model diagnosis.
\newblock In \emph{Proceedings of the IEEE/CVF Conference on Computer Vision and Pattern Recognition}, pp.\  11631--11640, 2023.

\bibitem[Mahalanobis(2018)]{mahalanobis2018generalized}
Prasanta~Chandra Mahalanobis.
\newblock On the generalized distance in statistics.
\newblock \emph{Sankhy{\=a}: The Indian Journal of Statistics, Series A (2008-)}, 80:\penalty0 S1--S7, 2018.

\bibitem[McCullagh(2019)]{mccullagh2019generalized}
Peter McCullagh.
\newblock \emph{Generalized Linear Models}.
\newblock Routledge, 2019.

\bibitem[Menon \& Vondrick(2022)Menon and Vondrick]{menon2022visual}
Sachit Menon and Carl Vondrick.
\newblock Visual classification via description from large language models.
\newblock In \emph{The Eleventh International Conference on Learning Representations}, 2022.

\bibitem[Morales~Rodr{\'\i}guez et~al.()Morales~Rodr{\'\i}guez, Pegalajar~Cuellar, and Morales]{morales4402768fusion}
David Morales~Rodr{\'\i}guez, Manuel Pegalajar~Cuellar, and Diego~P Morales.
\newblock On the fusion of soft-decision-trees and concept-based models.
\newblock \emph{Available at SSRN 4402768}.

\bibitem[Nilsback \& Zisserman(2008)Nilsback and Zisserman]{Nilsback08}
Maria-Elena Nilsback and Andrew Zisserman.
\newblock Automated flower classification over a large number of classes.
\newblock In \emph{Indian Conference on Computer Vision, Graphics and Image Processing}, 12 2008.

\bibitem[Oikarinen et~al.(2023)Oikarinen, Das, Nguyen, and Weng]{oikarinen2023label}
Tuomas Oikarinen, Subhro Das, Lam Nguyen, and Lily Weng.
\newblock Label-free concept bottleneck models.
\newblock In \emph{International Conference on Learning Representations}, 2023.

\bibitem[Ouyang et~al.(2022)Ouyang, Wu, Jiang, Almeida, Wainwright, Mishkin, Zhang, Agarwal, Slama, Ray, et~al.]{ouyang2022training}
Long Ouyang, Jeffrey Wu, Xu~Jiang, Diogo Almeida, Carroll Wainwright, Pamela Mishkin, Chong Zhang, Sandhini Agarwal, Katarina Slama, Alex Ray, et~al.
\newblock Training language models to follow instructions with human feedback.
\newblock \emph{Advances in Neural Information Processing Systems}, 35:\penalty0 27730--27744, 2022.

\bibitem[Panousis et~al.(2023)Panousis, Ienco, and Marcos]{panousis2023sparse}
Konstantinos~P. Panousis, Dino Ienco, and Diego Marcos.
\newblock Sparse linear concept discovery models.
\newblock In \emph{Proceedings of the IEEE/CVF International Conference on Computer Vision}, pp.\  2767--2771, 2023.

\bibitem[Peters et~al.(2017)Peters, Janzing, and Sch{\"o}lkopf]{peters2017elements}
Jonas Peters, Dominik Janzing, and Bernhard Sch{\"o}lkopf.
\newblock \emph{Elements of Causal Inference: Foundations and Learning Algorithms}.
\newblock The MIT Press, 2017.

\bibitem[Petsiuk et~al.(2018)Petsiuk, Das, and Saenko]{petsiuk2018rise}
Vitali Petsiuk, Abir Das, and Kate Saenko.
\newblock {RISE}: Randomized input sampling for explanation of black-box models.
\newblock \emph{arXiv preprint arXiv:1806.07421}, 2018.

\bibitem[Plumb et~al.(2018)Plumb, Molitor, and Talwalkar]{plumb2018model}
Gregory Plumb, Denali Molitor, and Ameet~S Talwalkar.
\newblock Model agnostic supervised local explanations.
\newblock \emph{Advances in neural information processing systems}, 31, 2018.

\bibitem[Plumb et~al.(2022)Plumb, Ribeiro, and Talwalkar]{plumb2022finding}
Gregory Plumb, Marco~Tulio Ribeiro, and Ameet Talwalkar.
\newblock Finding and fixing spurious patterns with explanations.
\newblock \emph{Transactions on Machine Learning Research}, 2022.
\newblock ISSN 2835-8856.
\newblock URL \url{https://openreview.net/forum?id=whJPugmP5I}.
\newblock Expert Certification.

\bibitem[Quattoni \& Torralba(2009)Quattoni and Torralba]{quattoni2009recognizing}
Ariadna Quattoni and Antonio Torralba.
\newblock Recognizing indoor scenes.
\newblock In \emph{2009 IEEE Conference on Computer Vision and Pattern Recognition}, pp.\  413--420. IEEE, 2009.

\bibitem[Quinlan(1986)]{quinlan1986induction}
J.~Ross Quinlan.
\newblock Induction of decision trees.
\newblock \emph{Machine Learning}, 1:\penalty0 81--106, 1986.

\bibitem[Radford et~al.(2021)Radford, Kim, Hallacy, Ramesh, Goh, Agarwal, Sastry, Askell, Mishkin, Clark, et~al.]{radford2021learning}
Alec Radford, Jong~Wook Kim, Chris Hallacy, Aditya Ramesh, Gabriel Goh, Sandhini Agarwal, Girish Sastry, Amanda Askell, Pamela Mishkin, Jack Clark, et~al.
\newblock Learning transferable visual models from natural language supervision.
\newblock In \emph{International Conference on Machine Learning}, pp.\  8748--8763, 2021.

\bibitem[Ramdas et~al.(2017)Ramdas, Garc{\'\i}a~Trillos, and Cuturi]{ramdas2017wasserstein}
Aaditya Ramdas, Nicol{\'a}s Garc{\'\i}a~Trillos, and Marco Cuturi.
\newblock On wasserstein two-sample testing and related families of nonparametric tests.
\newblock \emph{Entropy}, 19\penalty0 (2):\penalty0 47, 2017.

\bibitem[Ramesh et~al.(2021)Ramesh, Pavlov, Goh, Gray, Voss, Radford, Chen, and Sutskever]{ramesh2021zero}
Aditya Ramesh, Mikhail Pavlov, Gabriel Goh, Scott Gray, Chelsea Voss, Alec Radford, Mark Chen, and Ilya Sutskever.
\newblock Zero-shot text-to-image generation.
\newblock In \emph{International Conference on Machine Learning}, pp.\  8821--8831, 2021.

\bibitem[Reed et~al.(2022)Reed, Zolna, Parisotto, Colmenarejo, Novikov, Barth-maron, Gim{\'e}nez, Sulsky, Kay, Springenberg, et~al.]{reed2022generalist}
Scott Reed, Konrad Zolna, Emilio Parisotto, Sergio~G{\'o}mez Colmenarejo, Alexander Novikov, Gabriel Barth-maron, Mai Gim{\'e}nez, Yury Sulsky, Jackie Kay, Jost~Tobias Springenberg, et~al.
\newblock A generalist agent.
\newblock \emph{Transactions on Machine Learning Research}, 2022.

\bibitem[Ribeiro et~al.(2016)Ribeiro, Singh, and Guestrin]{lime}
Marco~Tulio Ribeiro, Sameer Singh, and Carlos Guestrin.
\newblock ``{W}hy should {I} trust you?'': Explaining the predictions of any classifier.
\newblock In \emph{Proceedings of the 22nd {ACM} {SIGKDD} International Conference on Knowledge Discovery and Data Mining, San Francisco, CA, USA, August 13-17, 2016}, pp.\  1135--1144, 2016.

\bibitem[Rombach et~al.(2021)Rombach, Blattmann, Lorenz, Esser, and Ommer]{rombach2021high}
Robin Rombach, Andreas Blattmann, Dominik Lorenz, Patrick Esser, and Bj{\"o}rn Ommer.
\newblock High-resolution image synthesis with latent diffusion models.
\newblock In \emph{CVF Conference on Computer Vision and Pattern Recognition (CVPR)}, pp.\  10674--10685, 2021.

\bibitem[Schuhmann et~al.(2022)Schuhmann, Beaumont, Vencu, Gordon, Wightman, Cherti, Coombes, Katta, Mullis, Wortsman, et~al.]{schuhmann2022laion}
Christoph Schuhmann, Romain Beaumont, Richard Vencu, Cade Gordon, Ross Wightman, Mehdi Cherti, Theo Coombes, Aarush Katta, Clayton Mullis, Mitchell Wortsman, et~al.
\newblock Laion-5b: An open large-scale dataset for training next generation image-text models.
\newblock \emph{Advances in Neural Information Processing Systems}, 35:\penalty0 25278--25294, 2022.

\bibitem[Selvaraju et~al.(2017)Selvaraju, Cogswell, Das, Vedantam, Parikh, and Batra]{selvaraju2017grad}
Ramprasaath~R. Selvaraju, Michael Cogswell, Abhishek Das, Ramakrishna Vedantam, Devi Parikh, and Dhruv Batra.
\newblock Grad-cam: Visual explanations from deep networks via gradient-based localization.
\newblock In \emph{Proceedings of the IEEE International Conference on Computer Vision}, pp.\  618--626, 2017.

\bibitem[Tommasi et~al.(2017)Tommasi, Patricia, Caputo, and Tuytelaars]{tommasi2017deeper}
Tatiana Tommasi, Novi Patricia, Barbara Caputo, and Tinne Tuytelaars.
\newblock A deeper look at dataset bias.
\newblock \emph{Domain Adaptation in Computer Vision Applications}, pp.\  37--55, 2017.

\bibitem[Yan et~al.(2023{\natexlab{a}})Yan, Wang, Zhong, Dong, He, Lu, Wang, Shang, and McAuley]{yan2023learning}
An~Yan, Yu~Wang, Yiwu Zhong, Chengyu Dong, Zexue He, Yujie Lu, William~Yang Wang, Jingbo Shang, and Julian McAuley.
\newblock Learning concise and descriptive attributes for visual recognition.
\newblock In \emph{Proceedings of the IEEE/CVF International Conference on Computer Vision}, pp.\  3090--3100, 2023{\natexlab{a}}.

\bibitem[Yan et~al.(2023{\natexlab{b}})Yan, Wang, Zhong, He, Karypis, Wang, Dong, Gentili, Hsu, Shang, et~al.]{yan2023robust}
An~Yan, Yu~Wang, Yiwu Zhong, Zexue He, Petros Karypis, Zihan Wang, Chengyu Dong, Amilcare Gentili, Chun-Nan Hsu, Jingbo Shang, et~al.
\newblock Robust and interpretable medical image classifiers via concept bottleneck models.
\newblock \emph{arXiv preprint arXiv:2310.03182}, 2023{\natexlab{b}}.

\bibitem[Yang et~al.(2023)Yang, Panagopoulou, Zhou, Jin, Callison-Burch, and Yatskar]{yang2023language}
Yue Yang, Artemis Panagopoulou, Shenghao Zhou, Daniel Jin, Chris Callison-Burch, and Mark Yatskar.
\newblock Language in a bottle: Language model guided concept bottlenecks for interpretable image classification.
\newblock In \emph{Proceedings of the IEEE/CVF Conference on Computer Vision and Pattern Recognition}, pp.\  19187--19197, 2023.

\bibitem[Yuksekgonul et~al.(2022)Yuksekgonul, Wang, and Zou]{yuksekgonul2022post}
Mert Yuksekgonul, Maggie Wang, and James Zou.
\newblock Post-hoc concept bottleneck models.
\newblock In \emph{The Eleventh International Conference on Learning Representations}, 2022.

\end{thebibliography}
